\documentclass{article}

\PassOptionsToPackage{numbers, compress}{natbib}
\usepackage[final]{neurips_2023}

\usepackage[utf8]{inputenc}
\usepackage[ruled,linesnumbered]{algorithm2e}
\usepackage[final, babel]{microtype}
\usepackage[noend]{algpseudocode}
\usepackage{amsmath,amsthm,amsfonts,amssymb}
\usepackage{hyperref}
\usepackage{color}
\usepackage{mathrsfs}
\usepackage{enumitem}
\usepackage{multirow}
\usepackage{booktabs}
\usepackage{makecell}
\usepackage{graphicx}
\usepackage{subfigure}
\usepackage{caption}
\usepackage{thmtools}
\usepackage{thm-restate}
\usepackage{hhline}
\usepackage{cite}
\usepackage[table]{xcolor}
\usepackage{times}
\usepackage{rotating} 
\usepackage{bm}

\renewcommand{\epsilon}{\varepsilon}

\newcommand{\cA}{\mathcal{A}}
\newcommand{\cC}{\mathcal{C}}
\newcommand{\cS}{\mathcal{S}}
\newcommand{\cF}{\mathcal{F}}
\newcommand{\cE}{\mathcal{E}}
\newcommand{\cI}{\mathcal{I}}
\newcommand{\cM}{\mathcal{M}}

\newcommand{\cW}{\mathcal{W}}

\newcommand{\cY}{\mathcal{Y}}
\newcommand{\cZ}{\mathcal{Z}}
\newcommand{\EE}{\mathbb{E}}

\newcommand{\ph}{{h'}}
\newcommand{\tV}{\tilde{V}}
\newcommand{\tQ}{\tilde{Q}}

\newcommand{\mP}{{\mathbb{P}}}
\newcommand{\mR}{{\mathbb{R}}}

\newcommand{\hmP}{{\hat{\mathbb{P}}}}

\newcommand\numberthis{\addtocounter{equation}{1}\tag{\theequation}}

\let\hat\widehat
\let\tilde\widetilde

\newtheorem{theorem}{Theorem}[section]
\newtheorem{lemma}[theorem]{Lemma}
\newtheorem{corollary}[theorem]{Corollary}
\newtheorem{remark}[theorem]{Remark}

\theoremstyle{definition}
\newtheorem{definition}[theorem]{Definition}
\newtheorem{condition}[theorem]{Condition}
\newtheorem{example}[theorem]{Example}
\newtheorem{assumption}{Assumption}

\makeatletter
\newcounter{framework}

\makeatother

\def\##1\#{\begin{align}#1\end{align}}
\def\$#1\${\begin{align*}#1\end{align*}}

\newcount\Comments  %
\newcommand{\kibitz}[2]{\ifnum\Comments=1\textcolor{#1}{#2}\fi}

\newcommand{\blue}[1]{{\color{blue} #1}}

\newcommand{\MO}{{M_\tO}}
\newcommand{\ME}{{M_\tE}}

\newcommand{\DeltaO}{\Delta_{\tO}}
\newcommand{\DeltaE}{\Delta_{\tE}}
\newcommand{\DeltaOt}{\Delta_{\tO,\task}}

\newcommand{\rO}[1]{r_{\tO #1}}
\newcommand{\rE}[1]{r_{\tE #1}}

\newcommand{\hmPO}[1]{{\hat{\mathbb{P}}}_{\tO #1}}
\newcommand{\hmPE}[1]{{\hat{\mathbb{P}}}_{\tE #1}}
\newcommand{\mPO}[1]{{\mathbb{P}}_{\tO #1}}
\newcommand{\mPE}[1]{{\mathbb{P}}_{\tE #1}}
\newcommand{\hmPOt}[1]{{\hat{\mathbb{P}}}_{\tO,\task #1}}

\newcommand{\NO}[1]{N_{\tO #1}}
\newcommand{\NE}[1]{N_{\tE #1}}
\newcommand{\hmuO}[1]{\hat{\mu}_{\tO #1}}
\newcommand{\hmuE}[1]{\hat{\mu}_{\tE #1}}
\newcommand{\piO}[1]{{\pi}_{\tO #1}}
\newcommand{\piE}[1]{{\pi}_{\tE #1}}
\newcommand{\upiO}[1]{{\underline{\pi}}_{\tO #1}}

\newcommand{\muO}{\mu_\tO}
\newcommand{\muE}{\mu_\tE}
\newcommand{\algO}{\text{Alg}^\tO}
\newcommand{\algE}{\text{Alg}^\tE}

\newcommand{\tauO}{{\tau}_\tO}
\newcommand{\tauE}{{\tau}_\tE}
\newcommand{\alphaO}{{\alpha}}
\newcommand{\alphaE}{{\alpha}}

\newcommand{\Regret}{\text{Regret}}

\newcommand{\scalar}{\nu}

\newcommand{\uQO}[1]{\underline{Q}_{\tO #1}}

\newcommand{\uVO}[1]{\underline{V}_{\tO #1}}
\newcommand{\tVO}[1]{\tilde{V}_{\tO #1}}

\newcommand{\QO}[1]{Q_{\tO #1}}
\newcommand{\VO}[1]{V_{\tO #1}}

\newcommand{\uQOt}[1]{\underline{Q}_{\tO,\task #1}}
\newcommand{\tQOt}[1]{\tilde{Q}_{\tO,\task #1}}

\newcommand{\uVOt}[1]{\underline{V}_{\tO,\task #1}}
\newcommand{\tVOt}[1]{\tilde{V}_{\tO,\task #1}}

\newcommand{\uQE}[1]{\underline{Q}_{\tE #1}}
\newcommand{\tQE}[1]{\tilde{Q}_{\tE #1}}

\newcommand{\uVE}[1]{\underline{V}_{\tE #1}}
\newcommand{\tVE}[1]{\tilde{V}_{\tE #1}}

\newcommand{\QE}[1]{Q_{\tE #1}}
\newcommand{\VE}[1]{V_{\tE #1}}

\newcommand{\omuO}[1]{\overline{\mu}_{\tO #1}}
\newcommand{\umuO}[1]{\underline{\mu}_{\tO #1}}
\newcommand{\omuE}[1]{\overline{\mu}_{\tE #1}}

\newcommand{\umuOtk}[1]{\underline{\mu}_{\tO,\task^k #1}}

\newcommand{\Gl}{G^{\text{lower}}}
\newcommand{\Gu}{G^{\text{upper}}}

\newcommand{\tE}{\text{Hi}}
\newcommand{\tO}{\text{Lo}}

\newcommand{\Null}{\texttt{Null}}

\newcommand{\ratio}{2}

\newcommand{\tk}{{\tilde{k}}}

\newcommand{\surplusE}[1]{{\mathbf{E}}_{\tE #1}}

\newcommand{\dsurplusE}[1]{\ddot{\mathbf{E}}_{\tE #1}}

\newcommand{\epsClip}{\epsilon_\text{Clip}}

\newcommand{\Clip}[2]{\text{Clip}\left[#1 #2\right]}

\newcommand{\cECk}{\cE_{\textbf{Con},k}}
\newcommand{\cEBk}{\cE_{\textbf{Bonus},k}}
\newcommand{\cEOk}{\cE_{\algO,k}}
\newcommand{\cEEk}{\cE_{\algE,k}}

\newcommand{\cECTk}{\cE_{\textbf{Con},[\Task],k}}
\newcommand{\cEBTk}{\cE_{\textbf{Bonus},[\Task],k}}
\newcommand{\cEOTk}{\cE_{\algO,[\Task],k}}

\newcommand{\bonus}{b}

\newcommand{\threshold}{\xi}

\newcommand{\MOt}{M_{\tO,\task}}
\newcommand{\tOt}{{\tO,\task}}

\newcommand{\NOempty}{N_{\tO}}
\newcommand{\NEempty}{N_{\tE}}

\newcommand{\Poly}{\text{Poly}}

\newcommand{\task}{w}
\newcommand{\Task}{W}
\newcommand{\tcW}{\tilde{\mathcal{\Task}}}

\title{Robust Knowledge Transfer in Tiered RL}

\author{%
  Jiawei Huang\\
  Department of Computer Science\\
  ETH Zurich \\
  \texttt{jiawei.huang@inf.ethz.ch} \\
  \And
  Niao He\\
  Department of Computer Science\\
  ETH Zurich \\
  \texttt{niao.he@inf.ethz.ch} \\
}

\begin{document}
\maketitle

\begingroup
\allowdisplaybreaks

\begin{abstract}
    In this paper, we study the Tiered Reinforcement Learning setting, a parallel transfer learning framework, where the goal is to transfer knowledge from the low-tier (source) task to the high-tier (target) task to reduce the exploration risk of the latter while solving the two tasks in parallel. Unlike previous work, we do not assume the low-tier and high-tier tasks share the same dynamics or reward functions, and focus on robust knowledge transfer without prior knowledge on the task similarity. We identify a natural and necessary condition called the ``Optimal Value Dominance'' for our objective. Under this condition, we propose novel online learning algorithms such that, for the high-tier task, it can achieve constant regret on partial states depending on the task similarity and retain near-optimal regret when the two tasks are dissimilar, while for the low-tier task, it can keep near-optimal without making sacrifice. Moreover, we further study the setting with multiple low-tier tasks, and propose a novel transfer source selection mechanism, which can ensemble the information from all low-tier tasks and allow provable benefits on a much larger state-action space.
\end{abstract}

\section{Introduction}\label{sec:introduction}
Comparing with individual learning from scratch, transferring knowledge from other similar tasks or side information has been proven to be an effective way to reduce the exploration risk and improve sample efficiency in Reinforcement Learning (RL).
Multi-Task RL (MT-RL) \citep{vithayathil2020survey} and Transfer RL \citep{taylor2009transfer, lazaric2012transfer, zhu2020transfer} are two mainstream knowledge transfer frameworks; however, both are subject to limitations when dealing with real-world scenarios.
MT-RL studies the setting where a set of similar tasks are solved concurrently, and the main objective is to accelerate the learning by sharing information of all tasks together.
However, in practice, in many MT-RL scenarios, the tasks are not equally important and we are more interested in the performance of certain tasks.
For example, in robot learning, a few robots are more valuable and hard to fix, while the others are cheaper or just simulators.
Most existing works on MT-RL treat all tasks equally and focus primarily on the reduction of the total regret of all tasks as a whole \citep{brunskill2013sample, d2020sharing, zhang2021provably, hu2021near}, with no guarantee of improving a particular task.
In contrast, transfer RL distinguishes the priority of different tasks by categorizing them into source and target tasks and aims at transferring the knowledge from source tasks
(or some side information like value predictors) to facilitate the learning of target tasks \citep{mann2013directed, tkachuk2021effect, golowich2022can, gupta2022unpacking}.
However, a key assumption in transfer RL is that the source task is completely solved before the learning of the target task, and this is not always practical. For example, in some sim-to-real domain, the source task simulator may require a long time to solve \citep{choi2021use}, and in some user-interaction scenarios \citep{huang2022tiered}, the source and target tasks refer to different user groups and they have to be served simultaneously.
In these cases, it's more reasonable to solve the source and target tasks in parallel and 
transfer the information immediately once available.

Recently, \citep{huang2022tiered} proposed a new ``parallel knowledge transfer'' framework, called Tiered RL,  which is promising to fill the gap. Tiered RL considers the case when  a source task $\MO$ and a target task $\ME$ are learned in parallel, by two separate algorithms $\algO$ and $\algE$, and its goal is to reduce the exploration risk and regret in learning $\ME$ by leveraging knowledge transfer from $\MO$ to $\ME$.
\citep{huang2022tiered} showed that under the strong assumption that $\MO = \ME$, it's possible to achieve constant regret in learning $\ME$ while keeping regret in $\MO$ optimal in gap-dependent setting. Yet, their algorithm based on Pessimistic Value Iteration (PVI) \citep{jin2021pessimism} can be hardly applied  when the assumption $\MO = \ME$ breaks, given its pure exploitation nature,  making their result very restrictive.

In this paper, we study the general Tiered RL setting \emph{without prior knowledge} about how similar the tasks are
\footnote{We defer the detailed framework to Appx. \ref{appx:general_learning_framework} for completeness, which is the same as Fr. 1 in \citep{huang2022tiered}.}. 
The key question we would like to address is:
\textbf{Can we design algorithms s.t.: (1) $\Regret$ in $\MO$ keeps near-optimal; (2) $\Regret$ in $\ME$ achieves provable benefits when $\MO$ and $\ME$ are similar while retaining near-optimal otherwise?}
Note for $\algO$, we expect it to achieve near-optimal regret bounds, which is reasonable since the source task is often important and our results still hold if relaxing it.
As for $\algE$, we expect it to be \emph{robust}, i.e., it can adaptively exploit from $\MO$ if it is close to $\ME$, while avoiding negative transfer in other cases.
Notably, our setting strictly generalizes \citep{huang2022tiered} and is much more challenging, for balancing the exploitation from  $\MO$ and exploration from $\ME$ without prior knowledge of task similarity.  
We give positive answers to the above question in this paper and demonstrate provable benefits with robust knowledge transfer for Tiered RL framework. Below,  we summarize our main contributions in three aspects.

\textbf{Our first contribution} is to identify essential conditions and notions about when and how our objective is achievable.
We first provide a mild condition called \emph{Optimal Value Dominance} (OVD), and in Sec. \ref{sec:lower_bound}, we justify its necessity to our objective by a lower bound result.
Our lower bound holds even if $M_\tO$ is fully known to the learner, and therefore, it also justifies the necessity of similar assumptions in previous transfer RL literatures \citep{golowich2022can,gupta2022unpacking}.
Besides, we introduce the notion of \emph{transferable states} to characterize states on which $\ME$ is expected to achieve benefits by transferring knowledge from $\MO$.
We believe those findings also provide useful insights for further works.

As \textbf{our second contribution}, in Sec.~\ref{sec:single_task_setting}, we propose novel algorithms for Tiered Multi-Armed Bandit (MAB) and Tiered RL, which can achieve robust parallel transfer by balancing between pessimism-based exploitation from $\MO$ and optimism-based online learning in $\ME$.
Depending on the similarity between $\MO$ and $\ME$, our algorithms can enjoy constant regret on a proportion of state-action pairs or even on the entire $\ME$ by leveraging information from $\MO$, while timely interrupting negative transfer and retaining near-optimal regret on states dissimilar between two tasks.
Moreover, in the bandit setting, our result implies a strictly improved regret bound when $\ME=\MO$, compared with previous results under the same setting \citep{rouyer2020tsallis, huang2022tiered}.

Beyond the single low-tier task setting, in many real-world scenarios, it's reasonable to assume there are multiple different low-tier tasks $\cM_\tO = \{M_{\tO,w}\}_{w=1}^W$ available.
As \textbf{our third contribution}, in Sec.~\ref{sec:multiple_source_task}, we extend our algorithm to this setting with a new source task selection mechanism.
By novel techniques, we show that, even if each $M_{\tO,w}$ may only share similarity with $\ME$ on a small portion of states, we are able to ensemble the information from each source task together to achieve constant regret on a much larger state-action space ``stitched'' from the transferable state-action set in each individual $M_{\tO,w}$, at the expense of an additional $\log W$ factor in regret.
Besides, our algorithm is still robust to model difference and retains near-optimal regret in general.
Although we only study the Tiered RL setting in this paper, we believe our task selection strategy can be applied to standard transfer RL setting \citep{golowich2022can,gupta2022unpacking} when multiple (partially correct) side information or value predictors are provided, which is an interesting direction for future work.
Finally, we conduct experiments in toy examples to verify our theoretical results.

\subsection{Closely Related Work} For the lack of space, we only discuss closely related work here and defer the rest to Appx. \ref{appx:comparison_with_previous}.
The most related to us is the Tiered RL framework \citep{huang2022tiered}, which was originally motivated by Tiered structure in user-oriented applications.
However, they only studied the case when $\ME = \MO$, which limits the practicability of their algorithms. 
Although there is a sequence of work studying parallel transfer learning in multi-agent system \citep{taylor2019parallel, liang2020parallel}, but they mainly focused on heuristic empirical algorithms and did not have theoretical guarantees.

Transfer RL \citep{zhu2020transfer}, compared to learning from scratch, can reduces exploration risk of target task by leveraging information in similar source tasks or side information \citep{mann2013directed, tkachuk2021effect, golowich2022can, gupta2022unpacking}.
Comparing with transfer RL setting, our parallel transfer setting has some additional challenges.
Firstly, $\ME$ can only leverage estimated value/model/optimal policy from $\MO$ with uncertainty, which implies we need additional efforts to control failure events with non-zero probability comparing with normal transfer RL setting.
Secondly, the constraints on the optimality of $\Regret_K(\MO)$, although reasonable, restrict the transferable information because in $\MO$ estimation uncertainty can only be controlled on those states frequently visited.
Moreover, none of these previous work studies how to leverage multiple partially correct side information like what we did in Sec.~\ref{sec:multiple_source_task}.
In MT-RL setting, the benefits of leveraging information gathered from other tasks has been observed from both empirical and theoretical works \citep{wilson2007multi, brunskill2013sample, d2020sharing, sodhani2021multi, zhang2021provably, hu2021near}.
But MTRL treats each task equally, and the reduction of total regret over all tasks does not directly imply benefits achieved in a particular task.

\section{Preliminary and Problem Formulation}\label{sec:preliminary}
\textbf{Tiered Stochastic MAB and Tiered Episodic Tabular RL} 
In Tiered MAB setting, we consider a low-tier task $\MO$ and a high-tier task $\ME$ sharing the arm/action space $\cA = \{1,2,...,A\}$. By pulling the arm $i\in[A]$ in $\MO$ or $\ME$, the agent can observe a random variable $\rO{}(i)$ or $\rE{}(i)\in [0,1]$. We will use $\muO{}(i)=\EE[\rO{}(i)]$ and $\muE{}(i)=\EE[\rE{}(i)]$ to denote the expected return of the $i$-th arm in $\MO$ and $\ME$, respectively, and note that it's possible that $\muO{}(i) \neq \muE{}(i)$.

For Tiered RL, we assume that two tasks $\MO=\{\cS,\cA,H,\mP_\tO,r_\tO\}$ and $\ME=\{\cS,\cA,H,\mP_\tE,r_\tE\}$ share the finite state $\cS$ and action space $\cA$ across episode length $H$ (i.e. $\cS_h= \cS, \cA_h=\cA$ for any $h\in[H]$), 
but may have different time-dependent transition and reward functions $\mP_\tO=\{\mP_{\tO,h}\}_{h=1}^H, r_\tO=\{r_{\tO,h}\}_{h=1}^H$ and $\mP_\tE=\{\mP_{\tE,h}\}_{h=1}^H, r_\tE=\{r_{\tE,h}\}_{h=1}^H$.
W.l.o.g., we assume the initial state $s_1$ is fixed, and the reward functions $r_\tO$ and $r_\tE$ are deterministic and bounded by $[0,1]$.
In episodic MDPs, we study the time-dependent policy specified as $\pi:=\{\pi_1,...,\pi_H\}$ with $\pi_h:\cS_h\rightarrow \Delta(\cA_h)$ for all $h\in[H]$, where $\Delta(\cA_h)$ denotes the probability simplex over the action space. With a slight abuse of notation, when $\pi_h$ is a deterministic policy, we use $\pi_h: \cS_h\rightarrow \cA_h$ to denote the deterministic mapping.
Besides, we use $Q^\pi_h(s,a)=\EE[\sum_{h'=h}^H r_{h'}(s_{h'},a_{h'})|s_h=s,a_h=a,\pi],~V^\pi_h(s)=\EE_{a\sim\pi}[Q^\pi_h(s,a)]$ to denote the value function for $\pi$ at step $h\in[H]$, and denote $d^{\pi}_h(\cdot) := \Pr(s_h=\cdot|\pi)$ and $d^\pi_h(\cdot,\cdot) := \Pr(s_h=\cdot, a_h=\cdot|\pi)$ as the state and state-action occupancy w.r.t. policy $\pi$.
We use $\pi^*$ to denote the optimal policy, and $V^{*}_h,Q^{*}_h,d^*_h$ as a short note when $\pi = \pi^*$.
To avoid confusion, we will specify the policy and value functions in $\MO$ and $\ME$ by $\tO$ and $\tE$ in subscription, respectively. For example, in $\MO$ we have $\piO{}:=\{\piO{,1},...,\piO{,H}\}$, $\QO{,h}^{\piO{}}/\VO{,h}^{\piO{}}$ and $\QO{,h}^{*}/\VO{,h}^{*}$, $d^{\pi_\tO}_{\tO,h},d^*_{\tO,h}$, and similarly for $\ME$. 

\textbf{Gap-Dependent Setting}
Throughout, we focus on gap-dependent setting \citep{lattimore2020bandit,simchowitz2019non, xu2021fine, dann2021beyond}. 
Below we introduce the notion of gap for $\MO$ as an example, and those for $\ME$ follows similarly.
In MAB case, the gap in $\MO$ w.r.t.~arm $i$ is defined as $\DeltaO(i):=\max_{j\in[A]}\muO{}(j) - \muO{}(i),\forall i \in [A]$, and for tabular RL setting, we have $\DeltaO{}(s_h,a_h):=\VO{,h}^*(s_h)-\QO{,h}^*(s_h,a_h),\forall h\in[H],s_h\in\cS_h,a_h\in\cA_h$. 
We use $\Delta_{\tO,\min}$ to refer to the minimal gap such that 
$\Delta(s_h,a_h) \geq \Delta_{\tO,\min}$ for all non-optimal actions $a_h$, and use $\Delta_{\min}:=\min\{\Delta_{\tO,\min}, \Delta_{\tE,\min}\}$ to denote the minimal gap over two tasks. In the gap-dependent setting, we assume $\Delta_{\min}>0$.

\textbf{Knowledge Transfer from Multiple Low-Tier Tasks}
In this case, we assume there are $W > 1$ different source tasks $\cM_\tO = \{M_{\tO,w}\}_{w=1}^W$ and all the tasks share the same state and action spaces but may have different transition and reward function.
We defer the extended framework for this setting to Appx.~\ref{appx:supplementary_introduction}.
We specify the task index $w\in[W]$ in sub-scription to distinguish the notation for different source tasks (e.g. $\mP_{\tO,w,h}, Q^{(\cdot)}_{\tO,w,h}$). Moreover, we define $\Delta_{\min}:=\min\{\Delta_{\tO,1,\min},...,\Delta_{\tO,W,\min},\Delta_{\tE,\min}\}$.
For convenience, in the rest of the paper, we use TRL-MST (Tiered RL with Multiple Source Task) as a short abbreviation for this setting.

\textbf{Performance Measure}
We use Pseudo-Regret as performance measure:
$\Regret_K(\MO):=\EE\left[\sum_{k=1}^K V_1^*(s_1)-V_1^{\piO{}^k}(s_1)\right];~$ $\Regret_K(\ME):=\EE\left[\sum_{k=1}^K V_1^*(s_1)-V_1^{\piE{}^k}(s_1)\right],$
where $K$ is the number of iterations, $\{\piO{}^k\}_{k=1}^K$ and $\{\piE{}^k\}_{k=1}^K$ are generated by the algorithms.

\textbf{Frequently Used Notations}
We denote $[n]=\{1,2,...,n\}$.
Given a transition matrix $\mP:\cS\times\cA\rightarrow \Delta(\cS)$, and a function $V: \cS \rightarrow \mR$, we use $\mP V(s_h,a_h)$ as a short note of $\EE_{s'\sim\mP(\cdot|s,a)}[V(s')]$. 
We will use $i^*_\tO/i^*_\tE$ to denote the optimal arm in bandit setting, and $\piO{}^*/\piE{}^*$ denotes the optimal policy in RL setting. In TRL-MST setting, we use $i^*_{\tO,w}/\piO{,w}^*$ to distinguish different source tasks.

\subsection{Assumptions and Characterization of Transferable States}
Throughout the paper, we make several assumptions.  The first one is the uniqueness of the optimal policy, which is common in the literature \citep{rouyer2020tsallis,chen2022offline}.
\begin{assumption}\label{assump:unique_optimal_policy}
    Both $\MO$ (or $\{M_{\tO,w}\}_{w=1}^W$) and $\ME$ have unique optimal arms/policies.
\end{assumption}

 Next, we introduce a new concept called ``Optimal Value Dominance'' (OVD for short), which says that for each state (or at least those states reachable by optimal policy in $\MO$), the optimal value of $\MO$ is an approximate overestimation for the optimal value of $\ME$.
In Sec. \ref{sec:lower_bound}, we will use a lower bound to show such a condition is necessary to attain the robust transfer objective.
\begin{assumption}\label{assump:opt_value_dominance}
    In single source task setting, we assume $\MO$ has Optimal Value Dominance (OVD) over $\ME$, s.t., $\forall h\in[H]$, for all $s_h\in\cS_h$ (or only for those $s_h$ with $d^*_{\tO,h}(s_h) > 0$), we have:
    $
    \VO{,h}^*(s_h) \geq \VE{,h}^*(s_h) - \frac{\Delta_{\min}}{2(H+1)}
    $.
    In TRL-MST setting, we assume each $M_{\tO,w}$ has OVD over $\ME$.\footnote{For convenience, we include the bandit setting as a special case with $H=0$; see also Def.~\ref{def:close_state}, \ref{def:transferable_states} and \ref{def:transferable_states_MT}.}
\end{assumption}
We remark that Assump.~\ref{assump:opt_value_dominance}  is a rather mild condition that naturally holds with reward shaping.  
Note that since $V^*_{\tE,h}(\cdot)\leq H - h$, by shifting the reward function of $\MO$ to $r'_{\tO,h}(\cdot,\cdot)=r_{\tO,h}(\cdot,\cdot) + 1$, we immediately obtain the OVD property. 
Even though, such a reward shift may impair the set of transferable states as we will introduce in Def.~\ref{def:transferable_states}.
We provide several reasonable settings in Appx.~\ref{appx:examples_OVD} including identical model \citep{huang2022tiered}, small model difference, and known model difference, where Assump.~\ref{assump:opt_value_dominance} is satisfied and there exists a non-empty set of transferable states. We also point out that several existing work on transfer RL assumed something similar or even stronger \citep{golowich2022can, gupta2022unpacking}, which we defer a thorough discussion to Appx.~\ref{appx:assumption_comparison}.

\begin{assumption}\label{assump:lower_bound_Delta_min}
    The learner has access to a quantity $\tilde{\Delta}_{\min}$ satisfying $0 \leq \tilde{\Delta}_{\min} \leq \Delta_{\min}$.
\end{assumption}
The final one is about the knowledge of a lower bound of $\Delta_{\min}$, which can always be satisfied by choosing $\tilde{\Delta}_{\min}=0$. Nevertheless, it would be more beneficial if the learner has access to some quantity $\tilde{\Delta}_{\min}$ closer to $\Delta_{\min}$ than 0.
As we introduce below, the magnitude of $\tilde{\Delta}_{\min}$ is related to how we quantify the similarity between $\MO$ and $\ME$ and which states we expect to benefit from knowledge transfer. 
Below we focus on the single source task setting and defer the discussion for TRL-MST setting to Sec. \ref{sec:multiple_source_task}. 
\begin{definition}[$\epsilon$-Close]\label{def:close_state}
    Task $\ME$ is $\epsilon$-close to task $\MO$ on $s_h$ at step $h$ for some $\epsilon > 0$, if $\VO{,h}^*(s_h) - \VE{,h}^*(s_h) \leq \epsilon$ and $\pi^*_{\tE}(s_h)=\pi^*_{\tO}(s_h)$.
\end{definition}
\begin{definition}[$\lambda$-Transferable States]\label{def:transferable_states}
   State $s_h$ is $\lambda$-transferable for some $\lambda > 0$, if $d^*_\tO(s_h) > \lambda$ and $\ME$ is $\frac{\tilde{\Delta}_{\min}}{4(H+1)}$-close to $\MO$ on $s_h$. The set of $\lambda$-transferable states at $h \in [H]$ is denoted as $\cZ_{h}^\lambda$. 
\end{definition}

\noindent We regard $s_h$ in $\MO$ has transferable knowledge to $\ME$, if it can be reached by optimal policy in $\MO$ and the optimal value and action at $s_h$ for two tasks are similar. Here the condition
$d^*_\tO(s_h) > 0$ is necessary since in $\MO$, only the states reachable by $\pi^*_\tO$ can be explored sufficiently by $\algO$ 
due to its near optimal regret.
Combining with Assump.~\ref{assump:opt_value_dominance}, one can observe that the value difference on transferable states are controlled by $|V^*_{\tO,h}(s_h) - V^*_{\tE,h}(s_h)|=O(\frac{\tilde{\Delta}_{\min}}{H})\leq O(\frac{\Delta_{\min}}{H})$.
As we will show in Thm.~\ref{thm:necessity_Delta_min} in Sec. \ref{sec:lower_bound},  the term $O(\Delta_{\min})$ is indeed unimprovable if we expect robustness. 
\section{Lower Bound Results: Necessary Condition for Robust Transfer}\label{sec:lower_bound}
Now we establish lower bounds that show Assump.~\ref{assump:opt_value_dominance} is necessary and how the magnitude of $\Delta_{\min}$ restricts the robust transfer objective.
The results in this section are based on two-armed Bernoulli bandits for simplicity, and the proofs are deferred to Appx. \ref{appx:lower_bound}.
By extending these hard instances to RL case, there is a gap caused by the additional $\frac{1}{H}$ in Assump.~\ref{assump:opt_value_dominance} and Def.~\ref{def:transferable_states}, which comes from the requirement of our algorithm design, and we leave it to the future work.

\textbf{Justification for Assump. \ref{assump:opt_value_dominance}}
We show that if Assump.~\ref{assump:opt_value_dominance} is violated, it is impossible to have algorithms $(\algO,\algE)$ to simultaneously achieve constant regret when $\MO = \ME$, while retaining sub-linear regret for all the cases regardless of the similarity between $\MO$ and $\ME$.
Here we require constant regret on $\MO=\ME$ since we believe it is a minimal expectation to achieve benefits in transfer when the two tasks are identical.
Intuitively, without Assump.~\ref{assump:opt_value_dominance}, even if we know $\muO{}(i^*_\tO) = \muE{}(i^*_\tO)$, we cannot ensure $i^*_\tO$ is the optimal arm in $\ME$.
Then, if $(\algO,\algE)$ can achieve constant regret on $\MO = \ME$, the algorithm must stop exploration on the arm $i \neq i^*_\tO$ after finite steps, and thus, it suffers from linear regret in another instance of $\ME$ where $i^*_\tO \neq i^*_\tE$.

Moreover, Thm.~\ref{thm:necessity_of_OVD} holds even if the learner has full information of $M_\tO$, where the setting degenerates to normal transfer RL since there is no need to explore $M_\tO$. This explains why similar assumptions to Assump.~\ref{assump:opt_value_dominance} are considered in previous transfer RL works \citep{golowich2022can,gupta2022unpacking}.

\begin{restatable}{theorem}{ThmNeceOVD}\label{thm:necessity_of_OVD}
    Under the violation of Assump. \ref{assump:opt_value_dominance}, even regardless of the optimality of $\algO$, for each algorithm pair $(\algO, \algE)$, it cannot simultaneously (1) achieve constant regret for the case when $\MO = \ME$ and (2) ensure sub-linear regret in all the other cases.
\end{restatable}

\textbf{$\Delta_{\min}$ in Tolerance Error is Inevitable}\quad
Next, we show that, if $\ME$ and $\MO$ are $\Delta$-close for some $\Delta\geq \frac{\Delta_{\min}}{2}$, in general we cannot expect to achieve constant regret on $\ME$ by leveraging $\MO$ without other loss.
Similar to Thm.~\ref{thm:necessity_of_OVD}, the main idea is to construct different instances for $\ME$ with different optimal arms and cannot be distinguished within finite number of trials.

\begin{restatable}{theorem}{ThmDeltaminLB}[Transferable States are Restricted by $\Delta_{\min}$]\label{thm:necessity_Delta_min}
    Under Assump.~\ref{assump:opt_value_dominance}, regardless of the optimality of $\algO$, given arbitrary $\Delta_{\min}$ and arbitrary $\Delta \in [\frac{\Delta_{\min}}{2}, \Delta_{\min}]$, for each algorithm pair $(\algO, \algE)$, it cannot simultaneously (1) achieve constant regret for the case when $\MO$ and $\ME$ with minimal gap $\Delta_{\min}$ are $\Delta$-close, and (2) ensure sub-linear regret in all other cases.
\end{restatable}

\section{Robust Tiered MAB/RL with Single Source Task}\label{sec:single_task_setting}
In this section, we study Tiered MAB and Tiered RL when a single low-tier task $\MO$ is available.
The key challenge compared with \citep{huang2022tiered} is that we do not have knowledge about whether $\MO$ and $\ME$ are similar or not so the pure exploitation will not work.
Instead, the algorithm should be able to identify whether $\MO$ and $\ME$ are close enough to transfer by data collected so far, and balance between the exploration by itself and the exploitation from $\MO$ at the same time.

To overcome the challenge, we identify a state-wise checking event, such that, under Assump. \ref{assump:opt_value_dominance}, if $s_h$ is transferable, the event is true almost all the time, and otherwise, every mistake will reduce the uncertainty so the chance the event holds is limited.
By utilizing it, our algorithm can wisely switch between optimistic exploration and pessimistic exploitation and achieve robust transfer.
In Sec. \ref{sec:single_task_bandit}, we start with the MAB setting and illustrate the main idea, and in Sec. \ref{sec:single_task_RL}, we generalize our techniques to RL setting, and discuss how to overcome the challenges brought by state transition.

\subsection{Robust Transfer in Tiered Multi-Armed Bandits}\label{sec:single_task_bandit}

The algorithm is provided in Alg. \ref{alg:Bandit_Setting}.
We choose $\algO$ as UCB, and $\algE$ as an exploitation-or-UCB style algorithm branched by a checking event in line \ref{line:checking_cond_MAB}, which is the key step to avoid negative transfer.

\begin{algorithm}
    \textbf{Initilize}: $\alpha > 2;~ \NO{}^1(i),~\NE{}^1(i),~\hmuO{}^1(i),~\hmuE{}^1(i)\gets 0,~\forall i \in \cA;$ $f(k):=1+16A^2(k+1)^2$\\
    Pull each arm at the beginning $A$ iterations\\
    \For{$k=A+1,A+2,...,K$}{
        $\omuO{}^k(i) \gets \hmuO{}^k(i) + \sqrt{\frac{2\alphaO\log f(k)}{\NOempty^{k}(i)}}, \quad \overline{\pi}_{\tO}^k \gets \arg\max_i \omuO{}^k(i),\quad \piO{}^k \gets \overline{\pi}_{\tO}^k.$\\
        $\umuO{}^k(i) \gets \hmuO{}^k(i) - \sqrt{\frac{2\alphaO\log f(k)}{\NOempty^{k}(i)}}, \quad \underline{\pi}_{\tO}^k \gets \arg\max_i \umuO{}^k(i).$\\
        $\omuE{}^k(i) \gets \hmuE{}^k(i) + \sqrt{\frac{2\alphaE\log f(k)}{\NEempty^{k}(i)}}, \quad \overline{\pi}_{\tE}^k \gets \arg\max_i \omuE{}^k(i).$\\
        \textbf{if} $\umuO{}^k(\underline{\pi}_{\tO}^k) \leq \omuE{}^k(\underline{\pi}_{\tO}^k) + \epsilon $ and $\NOempty^k(\underline{\pi}_{\tO}^k) > k / \ratio$ \label{line:checking_cond_MAB}
        \textbf{then} $\pi_{\tE}^k \gets \underline{\pi}_{\tO}^k$
        \textbf{else} $\pi_{\tE}^k \gets \overline{\pi}_{\tE}^k$.\par
        Interact $\ME/\MO$ by $\pi_{\tE}^k/\piO{}^k$; Update $\NO{}^{k+1}/\NE{}^{k+1}$ and empirical mean $\hat{\mu}_{\tO}^{k+1}/\hat{\mu}_{\tE}^{k+1}$.
    }
    \caption{Robust Tiered MAB}\label{alg:Bandit_Setting}
\end{algorithm}

\textbf{Key Insights: Separation between Transferable and Non-Transferable Cases}\quad
To understand our checking event, we consider the following two cases: (1) $\MO$ and $\ME$ are $\epsilon$-close, and (2) $i^*_\tE \neq i^*_\tO$ (in the rest cases we have $i^*_\tE = i^*_\tO$ but $\muO(i^*_\tO) > \muE(i^*_\tO) + \epsilon$, so exploiting from $\MO$ is harmless).
Recall Assump. \ref{assump:opt_value_dominance}, in Case 1, we have $\muO(i^*_\tO) \leq \muE(i^*_\tO) + \epsilon$, while in Case 2, with an appropriate choice of $\epsilon$ (e.g. $\epsilon=\frac{\tilde{\Delta}_{\min}}{4}<\frac{\Delta_{\min}}{4}$), we have $\muO(i^*_\tO) \geq \muE(i^*_\tE) - \frac{\Delta_{\min}}{2} \geq \muE(i^*_\tO) + \epsilon + \frac{\DeltaE(i^*_\tO)}{2}$, which reveals the separation between two cases.
As a result, if we can construct an uncertainty-based upper bound $\omuE{}^k(i^*_\tO)$ for $\muE{}(i^*_\tO)$, we should expect the event $\cE:=\muO{}(i^*_\tO) < \omuE{}^k(i^*_\tO) + \epsilon$ almost always be true in Case 1, while in Case 2, everytime $\cE$ occurs and $\algE$ takes $i^*_\tO$, the ``self-correction'' is triggered: the uncertainty is reduced so the estimation $\omuE{}^k(i^*_\tO)$ gets closer to $\muO{}(i^*_\tO)$, and because of the separation between $\muO{}(i^*_\tO)$ and $\muE{}(i^*_\tO)$, the number of times that $\cE$ is true is limited.
The remaining issue is that we do not know $\muO{}(i^*_\tO)$ and $i^*_\tO$, and we approximate them with LCB value $\umuO{}^k(\cdot)$ and its greedy policy.
The additional checking event $\NO{}^k(\upiO{}^k) > k/2$ is used to increase the confidence that $\upiO{}^k = i^*_\tO$ once transfer, which also contributes to reducing the regret.
Finally, to achieve constant regret, we use $\alpha$ to control the total failure rate to be $\sum_{k=1}^{+\infty} k^{-\Theta(\alpha)}=C$ for some constant $C$.
We summarize the main result in Thm. \ref{thm:regret_bound_in_bandit} and defer the proof to Appx. \ref{appx:proof_for_bandits}.

\begin{restatable}{theorem}{ThmBanditRegret}[Tiered MAB with Single Source Tasks]\label{thm:regret_bound_in_bandit}
    Under Assump. \ref{assump:unique_optimal_policy}, \ref{assump:opt_value_dominance} and \ref{assump:lower_bound_Delta_min}, by running Alg. \ref{alg:Bandit_Setting} with $\epsilon=\frac{\tilde{\Delta}_{\min}}{4}$ and $\alpha > 2$, we always have
        $
        \Regret_{K}(\ME) = O(\sum_{\DeltaE{}(i)>0} \frac{1}{\DeltaE{}(i)}\log K)$.
    Moreover, if $\ME$ and $\MO$ are $\frac{\tilde{\Delta}_{\min}}{4}$-close, we have:
        $
        \Regret_{K}(\ME) = O(\sum_{\DeltaE{}(i)>0} \frac{1}{\DeltaE{}(i)}\log\frac{A}{\Delta_{\min}})
        $.
\end{restatable}

\textbf{Comparison with \citep{rouyer2020tsallis,huang2022tiered}} \quad
As we can see, our algorithm can automatically achieve constant regret if tasks are similar while retaining near-optimal otherwise.
Notably, even  when $\ME=\MO$, our regret bound $\tilde{O}(\sum_{\DeltaE{}(i)>0} \frac{1}{\DeltaE{}(i)})$  is strictly better than  $\tilde{O}(\sqrt{\frac{A}{\Delta_{\min}}}\sqrt{\sum_{\DeltaE(i) > 0}\frac{1}{\DeltaE(i)}})$ in \citep{rouyer2020tsallis} and $\tilde{O}(\sum_{\DeltaE(i) > 0}  (A-i)(\frac{1}{\DeltaE(i)}- \frac{\DeltaE(i)}{\DeltaE(i-1)^2}))$ in \citep{huang2022tiered} under the same setting.
\subsection{Robust Transfer in Tiered Tabular RL}\label{sec:single_task_RL}

\noindent In this section, we focus on RL setting. We provide the algorithm in Alg.~\ref{alg:RL_Setting}, where we defer the details of \textbf{ModelLearning} function and the requirements for \textbf{Bonus} function to Appx.~\ref{appx:RL_add_Alg_Cond_Nota}.
In the following, we first highlight our main result. 
\begin{restatable}{theorem}{ThmRegretBound}[Tiered RL with Single Source Tasks]\label{thm:regret_upper_bound}
    Under Assump. \ref{assump:unique_optimal_policy}, \ref{assump:opt_value_dominance} and \ref{assump:lower_bound_Delta_min}, Cond.~\ref{cond:requirement_on_algO} for $\algO$ and Cond.~\ref{cond:bonus_term} for \textbf{Bonus} function,
    by running Alg.~\ref{alg:RL_Setting} with $\epsilon=\frac{\tilde{\Delta}_{\min}}{4(H+1)}$, $\alpha > 2$, an arbitrary $\lambda > 0$, we have
        $$
        \Regret_{K}(\ME) = O\Big(
        SH \sum_{h=1}^H \sum_{(s_h,a_h)\in \cS_h\times\cA_h \setminus \cC^{\lambda}_h} (\frac{H}{\Delta_{\min}}\wedge \frac{1}{\DeltaE(s_h,a_h)})\log (SAHK)\Big).
        $$
\end{restatable}
Here the set $\cC^\lambda_h\subset \cS_h\times\cA_h$ captures the benefitable state-action pairs to be introduced later.
For simplicity, in Thm.~\ref{thm:regret_upper_bound} above, we omit all constant terms independent with $K$ that may include $\lambda^{-1}$, $\Delta_{\min}^{-1}$ or $\log 1/d^*_h(\cdot)$. The complete version of Thm.~\ref{thm:regret_upper_bound} can be found in Thm.~\ref{thm:regret_upper_bound_detailed}. As we can see, comparing with pure online learning algorithms \citep{simchowitz2019non, xu2021fine, dann2021beyond}, in our setting, $\ME$ only suffers non-constant regret on a subset of the state-action space. The $SH$ factor may be further improved by choosing better \textbf{Bonus} functions than our Example~\ref{example:choice_of_B} given in Appx. \ref{appx:RL_add_Alg_Cond_Nota}.

Different from the bandit setting, the state transition causes more challenges. In the following, we first explain the algorithm design to highlight how we overcome the difficulties, and then provide the analysis and proof sketch. Detailed proofs can be found in Appx.~\ref{appx:proof_for_RL}.

\begin{algorithm}[t]
    \textbf{Input}: 
    Ratio $\lambda \in (0,1)$; $\alpha > 2$;
    Auxiliary functions \textbf{Bonus} and \textbf{ModelLearning};
    Sequence of confidence level $(\delta_k)_{k\geq 1}$ with $\delta_k = 1/SAHk^\alpha$;
    $\epsilon := \tilde{\Delta}_{\min}/4(H+1)$ for some $\tilde{\Delta}_{\min} \leq \Delta_{\min}$\\
    \textbf{Initialize}: $D^0_{\tO}/D^0_{\tE} \gets \{\}$; $\forall k,~\underline{V}_{(\cdot),H+1}^k,\underline{Q}_{(\cdot),H+1}^k,\tVE{,H+1}^k,\tQE{,H+1}^k\gets 0$.\\
    \For{$k=1,2,...$}{
        $\piO{} \gets \algO(D^{k-1}_{\tO})$;\\
        $\{\hmPO{,h}^k\}_{h=1}^H \gets \textbf{ModelLearning}(D_{\tO}^{k-1}),~\{\bonus_{\tO,h}^k\}_{h=1}^H \gets \textbf{Bonus}(D_{\tO}^{k-1},~\delta_k)$.\\
        \For{$h=H,H-1...,1$}{
            $\uQO{,h}^k(\cdot,\cdot) \gets \max\{0, \rO{,h}(\cdot,\cdot) + \hmPO{,h}^k\uVO{,h+1}^k(\cdot,\cdot) - \bonus_{\tO,h}^k(\cdot,\cdot)\}.$ \\
            $\uVO{,h}^k(\cdot) = \max_a \uQO{,h}^k(\cdot,a),\quad \upiO{,h}^k(\cdot) \gets \arg\max_a \uQO{,h}^k(\cdot,a).$
        }
        $\{\hmPE{,h}^k\}_{h=1}^H \gets \textbf{ModelLearning}(D_{\tE}^{k-1})$;   $\quad \{\bonus_{\tE,h}^k\}_{h=1}^H \gets \textbf{Bonus}(D_{\tE}^{k-1},~\delta_k)$.\\
        \For{$h=H,H-1...,1$}{
            $\uQE{,h}^{\pi_{\tE}^k}(\cdot,\cdot) \gets \max\{0, \rE{,h}(\cdot,\cdot) + \hmPE{,h}^{k}\uVE{,h+1}^k(\cdot,\cdot) - \bonus_{\tE,h}^k(\cdot,\cdot)\}$\\
            $\tQ_{\tE,h}^{k}(\cdot,\cdot) \gets \min\{H, r_{\tE,h}(\cdot,\cdot) + \hmPE{,h}^k \tV^k_{\tE, h+1}(\cdot,\cdot) + \bonus_{\tE,h}^k(\cdot,\cdot)\}$.\\
            \For{$s_h\in\cS_h$}{
                \If{$\uVO{,h}^k(s_h) \leq \tQ_{\tE,h}^{k}(s_h,\upiO{,h}^k) + \epsilon$ and $\max_a \NO{,h}^k(s_h,a) > \frac{\lambda}{3} k$ \label{line:checking_cond_RL}}{
                    $\pi_{\tE}^k(s_h) \gets \arg\max_a \NO{,h}^k(s_h,a)$.     
                    // \blue{``Trust and Exploit'' Branch}
                }
                \lElse{
                    $\pi_{\tE}^k(\cdot) \gets \arg\max_a \tQ_{\tE,h}^{k}(\cdot,a)$.
                    // \blue{``Explore by itself'' Branch} 
                }  
                $\tV^k_{\tE,h}(s_h) \gets \min\{H, \tQ_{\tE,h}^{k}(s_h,\pi_{\tE}^k)+\frac{1}{H}(\tQ_{\tE,h}^{k}(s_h,\pi_{\tE}^k) - \uQE{,h}^{\pi_{\tE}^k}(s_h,\pi_{\tE}^k))\}$ \label{line:revision_to_overestimation}\\
                $\uVE{,h}^{\pi_{\tE}^k}(s_h) = \uQE{,h}^{\pi_{\tE}^k}(s_h,\pi_{\tE}^k)$.
            }
        }
        Deploy $\piO{}/\pi_{\tE}$ to $\MO/\ME$ and get $\tauO^k/\tauE^k$;~and update $D_{\tO}^{k},D_{\tE}^{k}$.
    }
    \caption{Robust Tiered RL}\label{alg:RL_Setting}
\end{algorithm}

\textbf{Technical Challenges and Algorithm Design}\label{sec:RL_alg_design}
Similar to MAB setting, for $\MO$ we choose an arbitrary near-optimal algorithm, and for $\ME$ we set up a state-wise checking event $\uVO{,h}^k(\cdot)\leq \tQ^k_{\tE,h}(\cdot,\pi_{\tO,h}^k) + \epsilon$ in line \ref{line:checking_cond_RL} to determine whether to exploit from $\MO$ or not. Here $\uVO{,h}^k$ and $\tQ^k_{\tE,h}$ serve as lower and upper bounds for $V^*_\tO$ and $Q^*_\tE$, and $\uVO{,h}^k$ and $\upiO{,h}^k$ are constructed by Pessimistic Value Iteration \citep{jin2021pessimism}, which can be shown to converge to $V^*_{\tO,h}$ and $\piO{,h}^*$, respectively.
Similar to \citep{huang2022tiered}, for the choice of $\algO$ and the bonus term used to construct lower/upper confidence estimation, we consider general algorithm framework under Cond.~\ref{cond:requirement_on_algO} and Cond.~\ref{cond:bonus_term} in Appx. \ref{appx:RL_add_Alg_Cond_Nota}.

To overcome challenges resulting from state transition, we make two major modifications when moving from MAB to RL setting.  First of all, because of the constraint on the optimality of $\algO$, we cannot expect $\MO$ to provide useful information on those $(s_h,a_h)$ with $d^*_\tO(s_h,a_h) = 0$ since they will not be explored sufficiently. Therefore, in the checking event in line \ref{line:checking_cond_RL}, we include $\max_a \NO{,h}^k(s_h,a) > \Theta(\lambda k)$ as a criterion, where $\lambda$ is a hyper-parameter chosen and input to the algorithm.
Intuitively, for all $s_h,a_h$, we should expect $\NO{,h}^k(s_h,a_h)\approx \tilde{O}(d^*_\tO(s_h,a_h)\cdot k)$ when $k$ is large enough.
Therefore, by comparing $\NO{,h}^k$ with $\lambda k$, we can filter out those $s_h$ with $d^*_\tO(s_h) < \lambda$ to avoid harm from inaccurate estimation.

Secondly and more importantly, different from MAB setting, besides the error occurred at a particular step, we also need to handle the error accumulated during the back-propagation process of value iteration.
In our case, this is reflected by the loss of overestimation when we incorporate selective exploitation into the optimism-based exploration framework.
To see this, suppose at some $s_h$, we have an overestimation on optimal value $Q^*_{\tE,h}$ denoted as $\tilde{Q}_{\tE,h}^k$.
When the checking criterion is satisfied, if we mimic the MAB setting, i.e., assign $\pi_{\tO,h}^k$ to $\pi_{\tE,h}^k$ and update value by $\tilde{V}_{\tE,h}^k(s_h) \gets \tilde{Q}_{\tE,h}^k(s_h,\pi_{\tE,h}^k)$, 
when $\pi^*_\tO(s_h) \neq \pi^*_{\tE,h}(s_h)$, $\tilde{V}_{\tE,h}^k(s_h)$ is no longer guaranteed to be an overestimation for $V^*_{\tE,h}(s_h)$.
As $\tilde{V}_{\tE,h}^k(s_h)$ involves in back-propagation, it will pull down the estimation value for its ancestor states, thus reducing the chance to visit $s_h$ and slowing down the ``self-correction process'' which works well in MAB setting.

The key insight to overcome such difficulty is that, if the checking event holds yet $\pi^*_\tO(s_h) \neq \pi^*_{\tE,h}(s_h)$, the gap between $\tQ^k_{\tE,h}(s_h,\piO{}^*)$ and $Q^*_{\tE,h}(s_h,\piO{}^*)$ should not be small, and we can show that $\tQ^k_{\tE,h}(s_h,\upiO{}^k) \approx \tQ^k_{\tE,h}(s_h,\piO{}^*) \geq Q^*_{\tE,h}(s_h,\piO{}^*) + \Theta(\frac{H}{H+1}\DeltaE(s_h,\piO{}^*))$ with the choice of $\epsilon = O(\tilde{\Delta}_{\min}/H)$.
Therefore, revising $\tQ^k_{\tE,h}(s_h,\upiO{}^k)$ by adding $1/H$ of the gap $\tQ^k_{\tE,h}(s_h,\upiO{}^k)-Q^*_{\tE,h}(s_h,\piO{}^*)$ (line \ref{line:revision_to_overestimation}) is enough to guarantee the overestimation.
Lastly, since $Q^*_{\tE,h}(s_h,\piO{}^*)$ in unknown,  we construct an underestimation $\uQE{,h}^{\piE{}^k}(s_h,\piO{}^*)$ and use it instead.
As a result, we have the following theorem, where the clip function is defined by $\text{Clip}[x|w] := x \cdot \mathbb{I}[x \geq w]$.
\begin{restatable}{theorem}{ThmOverEst}\label{thm:overestimation}
    There exists $k_{ost} = \Poly(S,A,H,\lambda^{-1},\Delta_{\min}^{-1})$, such that, for all $k \geq k_{ost}$,
    on some event $\cE_k$ with $\mP(\cE_k) \leq 3\delta_k$, we have $\QE{,h}^*(s_h,a_h) \leq \tQE{,h}^k(s_h,a_h),~\VE{,h}^*(s_h) \leq \tVE{,h}^k(s_h),\forall h \in [H],~s_h\in\cS_h,~a_h\in\cA_h$ and
    \begin{align}
        \VE{,1}^*(s_1) - \VE{,1}^{\pi_{\tE}^k}(s_1) \leq 2e\EE_{\pi_{\tE}^k}\left[\sum_{h=1}^H \Clip{\min\{H, 4\bonus^k_{\tE,h}(s_h,a_h)\}}{\big|\frac{\Delta_{\min}}{4eH}\vee \frac{\DeltaE(s_h,a_h)}{4e}}\right].\label{eq:regret_bound}
    \end{align}
\end{restatable}

\textbf{Benefits of Knowledge Transfer}\quad 
We first take a look at $k \geq k_{ost}$. As implied from Eq.~\eqref{eq:regret_bound}, we can upper bound the regret on each $s_h,a_h$ 
by summing over the RHS of Eq.~\eqref{eq:regret_bound}.
Note that by Cond.~\ref{cond:bonus_term}, $b^k_{\tE,h}(s_h,a_h) = {O}(\frac{\Poly(SAH)\log k}{\sqrt{\NE{,h}^k(s_h,a_h)}})$ and $\EE[\NE{,h}^k(s_h,a_h)] = \sum_{k'=1}^{k-1} d^{\piE{}^k}(s_h,a_h)$, we can establish the near-optimal 
regret bound with similar techniques in \citep{simchowitz2019non} regardless of the similarity between $\ME$ and $\MO$.
Moreover, because of the knowledge transfer from $\MO$, we can achieve better regret bounds for $\ME$. 
In the following, we characterize three subclasses of state-action pairs, on which $\algE$ only suffers constant regret.
\emph{First of all}, for those $s_h \in \cZ^{\lambda}_h$, we can expect the checking event almost always hold for arbitrary $k$. Hence, when $k$ is large enough, $\piE{,h}^k(s_h) = \upiO{,h}^k(s_h)\approx \pi^*_{\tE,h}(s_h)$, implying $\algE$ will almost never take sub-optimal actions at $s_h$ since then. 
We denote this first subclass as $\cC^{1,\lambda}_h := \{(s_h,a_h) | s_h \in \cZ^{\lambda}_h, a_h \neq \pi_{\tE,h}^*(s_h)\}$.
\emph{Secondly}, note that, given a state $s_h$, if all possible trajectories starting from $s_1$ to $s_h$ have overlap with $\cC^{1,\lambda}_{h'}$ for some $h'\in[h-1]$, when $k$ is large enough, $\pi_{\tE}^k$ will almost have no chance to reach $s_h$ and will not suffer the regret at $s_h$.
For convenience, we define function $\text{Block}(\{\cC^{1,\lambda}_{\ph}\}_{\ph=1}^{h-1}, s_h)$ which takes \texttt{True} for those states described above, and takes \texttt{False} for the others. Then, we define the second subclass by $\cC^{2,\lambda}_h := \{(s_h,a_h)|\text{Block}(\{\cC^{1,\lambda}_{\ph}\}_{\ph=1}^{h-1}, s_h) = \texttt{True},~s_h\not\in \cZ^{\lambda}_h,~a_h\in\cA_h\}$.
\emph{Finally}, for those $s_h,a_h$ with $d^*_{\tE}(s_h,a_h) > 0$, we can show $\NE{,h}^k(s_h,a_h)\approx d^*_{\tE}(s_h,a_h) k$. Therefore, $b^k_{\tE,h}(s_h,a_h) \propto \log k/\sqrt{\NE{,h}^k(s_h,a_h)}$ in Eq.~\eqref{eq:regret_bound} will decay and the clipping operator will take effect, which leads to constant regret.
This third subclass is denoted by $\cC^*_h := \{(s_h,a_h) | d^*_\tE(s_h,a_h) > 0\}$.
Based on the above discussion, we define $\cC_h^{\lambda}:=\cC_h^{\lambda,1} \cup \cC_h^{\lambda,2} \cup \cC^*_h$ to be the benefitable states set in Thm.~\ref{thm:regret_upper_bound}.

For $k \leq k_{ost}$, for the lack of overestimation, we simply use $H$ to upper bound the value gap $V^* - V^{\pi^k_{\tE}}$. This results in a $\Poly(S,A,H,\lambda^{-1},\Delta^{-1}_{\min})$ burn-in term, which was omitted in Thm.~\ref{thm:regret_upper_bound} since it is independent with $K$.
Besides, by the definition of $k_{ost}$ in Thm.~\ref{thm:overestimation}, we can see the trade-off of choosing $\lambda$: a smaller $\lambda$ can enlarge $\cC^\lambda_h$ so we have constant regret on more state-action pairs, while it also results in the delay of overestimation by the larger $k_{ost}$.

\textbf{Constant Regret in the Entire MDP}\quad
We may expect constant regret in the entire $\ME$ in some special cases.
Note that, if $\forall h\in[H], \forall s_h$ with $d^*_{\tE}(s_h) > 0$, $s_h \in \cZ^{\lambda}_h$, we have $\cC^{\lambda}_h = \cS_h\times\cA_h$, $\Regret_K(\ME)$ will be independent w.r.t. $K$.
From this perspective, if $\lambda$ is chosen appropriately, e.g. $\lambda \leq \min_{s_h} d^*_{\tE}(s_h)$, we can recover the constant regret under the setting $\MO = \ME$ in \citep{huang2022tiered}.

\textbf{Choice of $\lambda$}\quad In this paper, we do not treat $\lambda$ as a parameter to optimize. In practice, without prior knowledge about $\max_{s_h} d^*_\tO(s_h)$, one may choose $\lambda = O(1/S)$ to ensure some chance that transferable states exist, since there exists at least some states satisfying $d^*_\tO(s_h) \geq 1/S$.

\section{Robust Tiered MAB/RL with Multiple Low-Tier Tasks}\label{sec:multiple_source_task}

Now, we focus on the case when a source task set $\cM_\tO := \{M_{\tO,w}\}_{w=1}^W$ is available (see Frw.~\ref{framework:general_learning_framework_MT} in Appx.~\ref{appx:supplementary_introduction}).
Our objective is to achieve benefits on those states $s_h$ as long as there exists some task $w \in [W]$ such that $M_{\tO,w}$ and $\ME$ are close on $s_h$, while retaining near-optimal regret in other cases under Assump.~\ref{assump:opt_value_dominance}.
The key challenge comparing with single task case is that, $\algE$ has to identify for each state which task in $\cM_\tO$ is the appropriate one to leverage.
The main novelty and contribution in this section is a task selection mechanism we call \emph{``Trust till Failure''}, which can automatically adapt to the similar task if it exists.
We first highlight the main results for MAB and RL setting.
\begin{restatable}{theorem}{ThmMTMAB}[Tiered MAB with Multiple Source Tasks]\label{thm:regret_bound_MAB_MT}
    Under Assump. \ref{assump:unique_optimal_policy}, \ref{assump:opt_value_dominance}, and \ref{assump:lower_bound_Delta_min}, by running Alg.~\ref{alg:Bandit_Setting_MT} with $\cM_\tO =\{M_{\tO,w}\}_{w=1}^W$ and $\MO$, with $\epsilon=\frac{\tilde{\Delta}_{\min}}{4}$ and $\alpha > 2$, we always have:
        $
        \Regret_{K}(\ME) = O(\sum_{\DeltaE{}(i) > 0} \frac{1}{\DeltaE{}(i)}\log (\Task K)).$
    Moreover, if at least one task in $\cM_\tO$ is $\frac{\tilde{\Delta}_{\min}}{4}$-close to $\ME$, we further have:
        $
        \Regret_{K}(\ME) = O(\sum_{\DeltaE{}(i) > 0} \frac{1}{\DeltaE{}(i)}\log\frac{A\Task}{\Delta_{\min}}).
        $
\end{restatable}
\begin{restatable}{theorem}{ThmRegretBoundRLMT}[Tiered RL with Multiple Source Tasks]\label{thm:regret_bound_RL_MT}
    Under Assump. \ref{assump:unique_optimal_policy}, \ref{assump:opt_value_dominance}, \ref{assump:lower_bound_Delta_min}, and 
    Cond.~\ref{cond:bonus_term},~\ref{cond:requirement_on_algO_MT},
    by running Alg.~\ref{alg:RL_Setting_MT} in Appx.~\ref{appx:RL_add_Alg_Cond_Nota_MT} with $\epsilon=\frac{\tilde{\Delta}_{\min}}{4(H+1)}$, $\alpha > 2$ and any $\lambda > 0$, we have 
    $$
    \Regret_{K}(\ME) = O\Big(SH\sum_{h=1}^H \sum_{(s_h,a_h) \in \cS_h\times\cA_h \setminus \cC^{\lambda,[W]}_h} (\frac{H}{\Delta_{\min}}\wedge \frac{1}{\DeltaE(s_h,a_h)})\log (SAH\Task K)\Big).
    $$
\end{restatable}

\begin{algorithm}
    \textbf{Initilize}: $\alpha > 2;~ \NO{}^1(i),~\NE{}^1(i),~\hmuO{}^1(i),~\hmuE{}^1(i)\gets 0,~\forall i \in \cA;$ $f(k):=1+16TA^2(k+1)^2$\\
    Pull each arm at the beginning $A$ iterations. Set $w^A \gets \Null$.\\
    \For{$k=A+1,2,...,K$}{
        \For{$\task=1,2...,\Task$}{
            $\omuO{,\task}^k(i) \gets \hmuO{,\task}^k(i) + \sqrt{\frac{2\alphaO\log f(k)}{\NO{,\task}(i)}}, \quad \overline{\pi}_{\tO,\task}^k \gets \arg\max_i \omuO{,\task}^k(i),\quad\piO{,\task}^k \gets \overline{\pi}_{\tO,\task}^k.$\\
            $\umuO{,\task}^k(i) \gets \hmuO{,\task}^k(i) - \sqrt{\frac{2\alphaO\log f(k)}{\NO{,\task}(i)}}, \quad \underline{\pi}_{\tO,\task}^k \gets \arg\max_i \umuO{,\task}^k(i).$
        }
        $\omuE{}^k(i) \gets \hmuE{}^k(i) + \sqrt{\frac{2\alphaE\log f(k)}{\NEempty^{k}(i)}}, \quad \overline{\pi}_{\tE}^k \gets \arg\max_i \omuE{}^k(i).$\\
        $\cI^k \gets \{\task\in[\Task]| \umuO{,\task}^k(\underline{\pi}_{\tO,\task}^k) \leq \omuE{}^k(\underline{\pi}_{\tO,\task}^k) + \epsilon $ and $\NO{,\task}^k(\underline{\pi}_{\tO,\task}^k) > k / \ratio\}$ \label{line:bandit_MT_TS_start}\\
        \lIf{$\cI^k = \emptyset$}{
            $\task^k \gets \Null,\quad \piE{}^k \gets \overline{\pi}_{\tE}^k$
        }
        \Else{
            \lIf{$\task^{k-1} \neq \Null$ and $\task^{k-1} \in \cI^k $}{
                $\task^k \gets \task^{k-1},\quad \piE{}^k \gets \upiO{,\task^k}^k$
            }
            \ElseIf{$\task^{k-1} \neq \Null$ and $\exists \task \in \cI^k$ such that $\piE{}^{k-1} = \arg\max_i \NO{,\task}^k(i)$}{
                    $\task^k \gets w,\quad \piE{}^k \gets \upiO{,w}^k$\label{line:bandit_MT_TS_mid}
            }
            \lElse{
                $\task^k \sim \text{Unif}(\cI^k),\quad \piE{}^k \gets \upiO{,\task^k}^k$
            }
        }\label{line:bandit_MT_TS_end}
        Interact with $\ME/\{M_{\tO,}\}_{\task=1}^\Task$ by $\pi_{\tE}^k/\{\piO{,\task}^k\}_{\task=1}^\Task$; \\
        Update $\{\NO{,w}^{k+1}\}_{w=1}^W,\NE{}^{k+1}$ and empirical mean $\{\hat{\mu}_{\tO,w}^{k+1}\}_{w=1}^W,\hat{\mu}_{\tE}^{k+1}$ for each arm.
    }
    \caption{Robust Tiered MAB with Multiple Source Tasks}\label{alg:Bandit_Setting_MT}
\end{algorithm}
For the lack of space, in the following, we only analyze the bandit setting to explain the key idea of our task selection strategy.
For the RL setting, we defer to Appx.~\ref{appx:RL_multi_task_version} the algorithm Alg.~\ref{alg:RL_Setting_MT}, detailed version of Thm.~\ref{thm:regret_bound_RL_MT} (Thm.~\ref{thm:regret_upper_bound_detailed}), defintion of transferable set $\cC^{\lambda,[W]}_h$ (Def.~\ref{def:benefitable_states_MT}), and technical details.

\textbf{Algorithm Design and Proof Sketch for Bandit Setting}~The algorithm is provided in Alg. \ref{alg:Bandit_Setting_MT}.
Comparing with Alg. \ref{alg:Bandit_Setting} in single task setting, 
the main difference is the task selection strategy from line \ref{line:bandit_MT_TS_start} to line \ref{line:bandit_MT_TS_end}.
We first examine each source task with a checking event similar to single task setting, and collect those feasible tasks passing the test to $\cI^k$.
Intuitively, for those $M_{\tO,w^*}$ close to $\MO$, we expect $M_{\tO,w} \in \cI^k$ holds almostly for arbitrary $k > 0$, while for the other $M_{\tO,w'}$, if it takes the position of $w^k$, following $M_{\tO,w'}$ will reduce the uncertainty and it will be ruled out from $\cI^k$, eventually. So we expect $w^k$ can ``escape'' from dissimilar source tasks but be absorbed to the similar task if exists.
Therefore, if $\cI^k$ is empty, $\algE$ will do exploration by itself.
Otherwise, we choose one from $\cI^k$ to transfer the action until it fails on the checking event.
However, for any $\epsilon$ the algorithm chosen, those ``marginally similar'' source tasks (denoted as $M_{\tO,\tilde{w}}$), which are $\epsilon'$-close to $\ME$ for some $\epsilon'$ only slightly larger than $\epsilon$, may cause some trouble.
Because the checking event will finally eliminate $M_{\tO,\tilde{w}}$ since they are not $\epsilon$-close, but it may occupy the position $w^k$ for a long time before elimination, especially when $\epsilon'$ is extremely close to $\epsilon$.
After eliminating $M_{\tO,\tilde{w}}$, $\algE$ needs to re-select one from $\cI^k$. Now since other sub-optimal arms in $\ME$ haven't been chosen for a long time and the confidence level $\delta_k=O(1/k^\alpha)$ is decreasing, $\cI^k$ will include those dissimilar tasks again, which causes difficulty to identify the true similar task.
To solve this issue, once the previous trusted task fails, we give priority to the task recommending the same action as the previous one (line \ref{line:bandit_MT_TS_mid}).
As a result, since $M_{\tO,\tilde{w}}$ and $M_{\tO,w^*}$ share the  optimal action, after the elimination of $M_{\tO,\tilde{w}}$, we can expect $w^k$ to only switch among those tasks $M_{\tO,w}$ with $\pi^*_{\tO,w} = \pi^*_\tE$.
We highlight this technical novelty to Lem.~\ref{lem:absorbing_to_similar_task} below, and defer all the proofs to Appx.~\ref{appx:bandit_multi_task_version}.
\begin{restatable}{lemma}{LemAbsSimTask}[Absorbing to Similar Task]\label{lem:absorbing_to_similar_task}
    Under Assump. \ref{assump:unique_optimal_policy}, \ref{assump:opt_value_dominance} and \ref{assump:lower_bound_Delta_min},
    there exists a constant $c^*$, s.t., if there exists at least one $\task^* \in [\Task]$ such that $M_{\tO,\task^*}$ is $\frac{\tilde{\Delta}_{\min}}{4}$-close to $\ME$, by running Alg.~\ref{alg:Bandit_Setting_MT} with $\epsilon = \frac{\tilde{\Delta}_{\min}}{4}$ and $\alpha > 2$, for any $k \geq k^* := c^*\frac{\alpha A }{\Delta_{\min}^2}\log\frac{\alpha A \Task}{\Delta_{\min}}$, we have $
        \Pr(\piE{}^k \neq i^*_{\tE}) = O(\frac{A}{k^{2\alpha-2}})$.
\end{restatable}

\section{Experiments}
In this section, we evaluate our most representative algorithm, Alg. 7, in multiple source tasks setting.

\textbf{Experiments Setting}\footnote{Code is available at \url{https://github.com/jiaweihhuang/Robust-Tiered-RL}}
We set $S=A=3$ and $H=5$. The details for construction of source and target tasks are defered to Appx.~\ref{appx:experiments}.
We adapt StrongEuler in \citep{simchowitz2019non} as online learning algorithm to solve source tasks, and use the bonus function in \citep{simchowitz2019non} as the bonus function in our Alg. 7.
We evaluate our algorithm when $W = 0, 1, 2, 5$, where $W = 0$ means the high-tier task $M_{\text{Hi}}$ is simply solved by normal online learning method (StrongEuler) without any parallel knowledge transfer.
We choose $\lambda = 0.3 \approx 1/S$ in Alg. 7, and in the MDP instance we test, across all $S\cdot H=15$ states, for $W=1,2,5$, the number of transferable states would be 6, 9 and 13, respectively.

We choose iteration number $K = 1e7$, where we start the transfer since $k=5e5$ to avoid large "burn-in" terms. As we can see, after the transfer starts, the regret in target task will suddenly increase for a while, because the target task has to make some mistakes and learn from it as a result of the model uncertainty. However, because of our algorithm design, the negative transfer will terminate after a very short period. As predicted by our theory, by adding more and more source tasks which can introduce new transferable states, the target task will suffer less and less regret.

\begin{figure*}[h!]
\begin{center}
    \includegraphics[width=0.55\textwidth]{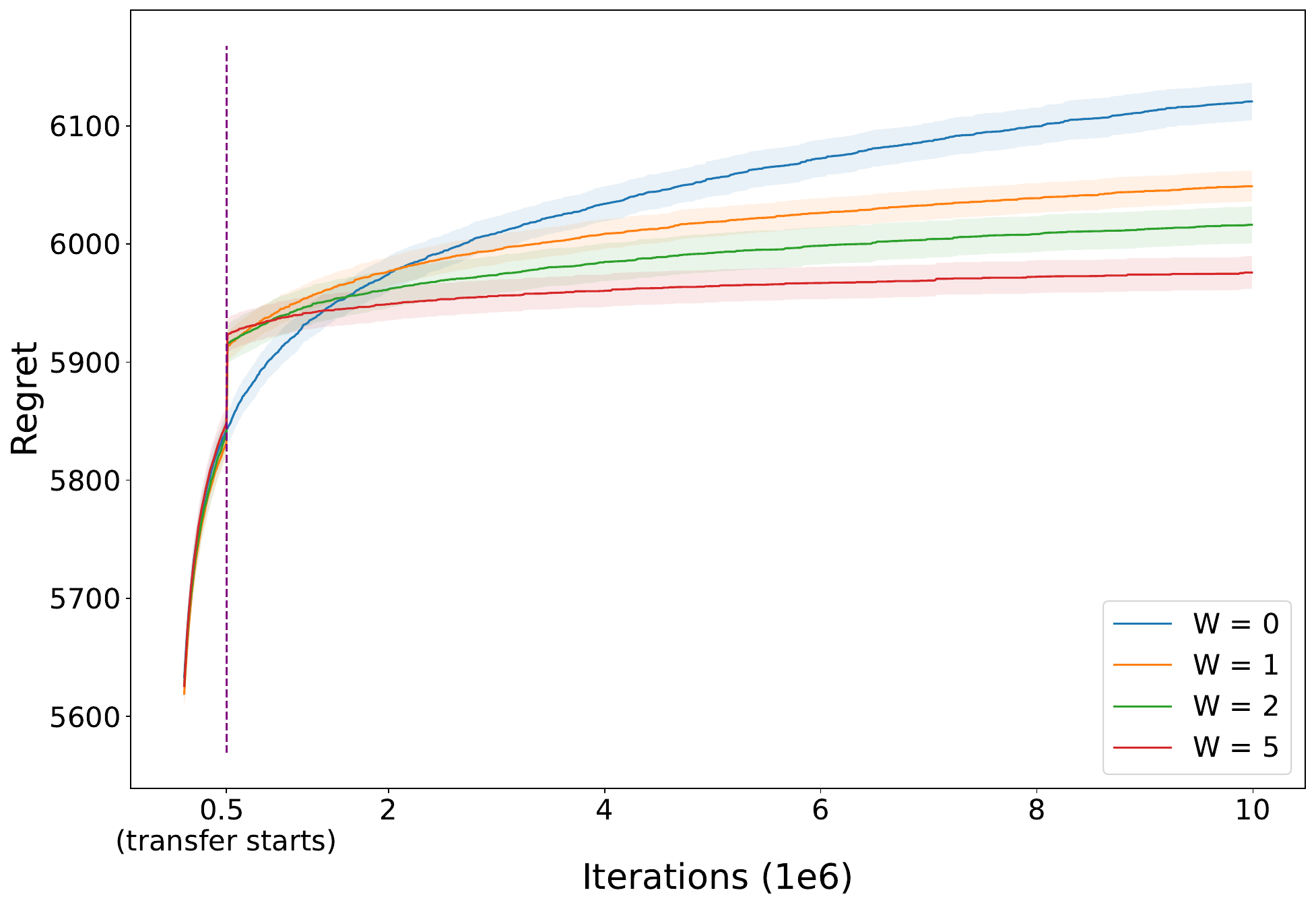}
    \caption{
    \textbf{Regret in the Target Task given Multiple Source Tasks} 
    We report the result when $W$ source tasks are available with $W=0,1,2,5$.
    The shadows indicate 96\% confidence interval.
    }
\end{center}
\end{figure*}

\section{Conclusion and Future Work}
In this paper, we study how to do robust parallel transfer RL when single or multiple source tasks are avilable, without knowledge on models' similarity.
The possible future directions include relaxing assumptions, better strategies to leveraging multiple source tasks, and identifying mild structural assumptions allowing for more aggressive transfer, and we defer to Appx.~\ref{appx:open_problems} for more details.

\newpage
\begin{ack}
The authors would like to thank Andreas Krause for valuable discussion. The work is supported by ETH research grant and Swiss National Science Foundation (SNSF) Project Funding No. 200021-207343.
\end{ack}
\bibliographystyle{plainnat}
\bibliography{references}

\newpage
\tableofcontents
\newpage
\appendix
\section{Extended Introduction}\label{appx:supplementary_introduction}
\subsection{Tiered-RL Framework}\label{appx:general_learning_framework}
\begin{algorithm}[H]
    Initialize $D^1_{\tO},D^1_{\tE} \gets \{\}.$ \\
    \For{$k = 1,2,...,K$}{
        {$\piO{}^k \gets \algO(D^k_\tO)$}; \quad
        {$\piO{}^k$} interacts with $\MO$, and collect data {$\tau_{\tO}^k$}; $\quad D^{k+1}_\tO = D^k_\tO \cup \{\tau_{\tO}^k\}$\\
        {$\piE{}^k \gets \algE(D^k_\tE)$}; \quad ~
        {$\piE{}^k$} interacts with $\ME$, and collect data {$\tau_{\tE}^k$}; \quad  
        $D^{k+1}_\tE = D^k_\tE \cup \{\tau_{\tE}^k\}$. 
    }
    \caption{The Tiered RL Framework with Single Low-Tier Task}\label{framework:general_learning_framework}
\end{algorithm}
\begin{algorithm}[H]
    Initialize $D^1_{\tO} \gets \{\};~D^1_{\tE,w} \gets \{\},~\forall w\in[W].$ \\
    \For{$k = 1,2,...,K$}{
        \For{$w \in [W]$}{
            $\piO{,w}^k \gets \algO(D^k_{\tE,w})$;\quad $\piO{,w}^k$ interacts with $M_{\tO,w}$, and collect data $\tau_{\tO,w}^k$.\\
            $D_{\tO,w}^{k+1} = D^k_\tE \cup \{\tau_{\tO,w}^k\}$.
        }
        $\piE{}^k \gets \algE(D^k_\tE)$; \quad $\piE{}^k$ interacts with $\ME$, and collect data $\tau_{\tE}^k$;\quad $D_\tE^{k+1} = D^k_\tE \cup \{\tau_{\tE}^k\}$. 
    }
    \caption{The Tiered RL Framework with Multiple Low-Tier Tasks}\label{framework:general_learning_framework_MT}
\end{algorithm}

\subsection{Other Related Works}\label{appx:comparison_with_previous}

\paragraph{Online and Offline RL} 
In normal online RL/MAB setting, the learner targets at actively explore the environment while balancing the trade-off between the exploration and exploitation \citep{lattimore2020bandit,agarwal2019reinforcement, simchowitz2019non, dann2021beyond,azar2017minimax,jin2018q}
Differently, motivated by many real world scenarios where historical data are available, offline RL considers how to do pure exploitation given the pre-collected dataset without additional exploration and new information collection, and theoretical works mainly focus on methods for sufficiently exploitation \citep{zhan2022offline, uehara2021pessimistic, xie2021bellman, jin2021pessimism, liu2020provably,jiang2020minimax, chen2022offline}.

Recently, there is also a line of work studying the settings lying between pure online and offline RL, 
such as hybrid setting where offline data is available for online exploration \citep{xie2021policy, wagenmaker2022leveraging}, and efficient batched exploration with limited policy deployments \citep{huang2022towards}.
Tiered RL framework can be regarded as another approach to bridging the online and offline setting, where we do online learning in the high-tier task with a gradually updated dataset from low-tier task for reference.

\paragraph{Detailed Comparison with Previous Transfer RL Paper about Assumptions}\label{appx:assumption_comparison}
The most recent works in transfer RL are \citep{golowich2022can} and \citep{gupta2022unpacking}. In general we are not comparable because of our different settings, but we can observe some similarity of our assumptions and the way to capture the transferable states.

\citep{golowich2022can} considered the case where a predicted Q-function $\{\tilde{Q}_h(\cdot,\cdot)\}_{h=1}^H$ is provided for each state action pair, which can be regarded as $\{Q^*_{\tO,h}(\cdot,\cdot)\}_{h=1}^H$ in our setting.
The key assumption in their paper is the ``approximate distillation'' condition (Def. 3.1 in \citep{golowich2022can}), which assumed that for each $s_h$, there exists $a_h \in \cA_h$ such that $\Delta(s_h,a_h) + \max\{0,Q^*_h(s_h,a_h) - \tQ_h(s_h,a_h)\} \leq \epsilon$. 
However, according to Eq.~(2) in their Thm. 3.1, there is an $\epsilon'TH = 4\epsilon(H+1)TH$ term in the regret of their algorithm (where $T$ is the episode number). 
Therefore, in order to achieve regret sub-linear to $T$, they need $\epsilon = T^{-\alpha}$ for some $\alpha > 0$. 
As $T \rightarrow +\infty$, we have $\epsilon \rightarrow 0$, then their ``approximate distillation'' condition will reduce to $V^*_{\tO,h}(s_h) \geq V^*_{\tE,h}(s_h)$, which is a stronger version of our OVD condition in Assump.~\ref{assump:opt_value_dominance}.

As for \citep{gupta2022unpacking}, the authors assumed that there is a value function and parameter $\beta$ such that $\beta \tilde{V}_h(s_h)$ (i.e. $V^*_{\tO}(s_h)$ in our setting) forms an overestimation for $V^*_h(s_h)$ in target task, which is also similar to our Assump.~\ref{assump:opt_value_dominance}.
Besides, although they didn't make it explicitly, to achieve provable benefits, they also require such a overestimation $\beta \tilde{V}_h(s_h)$ should not deviate too far away from the true value $V^*_h(s_h)$.
To see this, in Sec. 5.1 of \citep{gupta2022unpacking}, they use $V^*_h(s) \geq \Delta + \tQ^u_h(s_h,a_h)$ to characterize state-action pairs with regret reduction, where $\tQ^u_h(s_h,a_h) := \EE_{s'}[r(s_h,a_h) + \beta \tV_{h+1}(s')]$, and $\beta \tV_{h+1}$ has to stay close to $V^*_h(s)$ for such condition to be realizable.

Finally, both \citep{golowich2022can,gupta2022unpacking} assumed the condition holds for each state action pair, while our Assump.~\ref{assump:opt_value_dominance} can only require the overestimation on states reachable by optimal policy.

\subsection{Examples for Assump. \ref{assump:opt_value_dominance}}\label{appx:examples_OVD}
\begin{example}[Identical Model \citep{huang2022tiered}]\label{example:OVD_1} For arbitrary $h\in[H],s_h\in\cS_h,a_h\in\cA_h$, $r_{\tO,h}(s_h,a_h) = r_{\tE,h}(s_h,a_h)$, $\mPE{,h}(\cdot|s_h,a_h) = \mPO{,h}(\cdot|s_h,a_h)$.
\end{example}
\begin{example}[Small Model Error]\label{example:OVD_2} For arbitrary $h\in[H],s_h\in\cS_h,a_h\in\cA_h$, $|r_{\tO,h}(s_h,a_h) - r_{\tE,h}(s_h,a_h)| \leq \frac{\Delta_{\min}}{4H(H+1)}$, $\|\mPE{,h}(\cdot|s_h,a_h) - \mPO{,h}(\cdot|s_h,a_h)\|_1 \leq \frac{\Delta_{\min}}{4H^2(H+1)}$.
\end{example}
\begin{example}[Known Model Difference]\label{example:OVD_3} Suppose there exists known quantities $\xi_r$ and $\xi_{\mP}$ such that, for arbitrary $h\in[H],s_h\in\cS_h,a_h\in\cA_h$:
    \begin{align*}
        |r_{\tO,h}(s_h,a_h) - r_{\tE,h}(s_h,a_h)| \leq \xi_r,\quad \|\mPE{,h}(\cdot|s_h,a_h) - \mPO{,h}(\cdot|s_h,a_h)\|_1 \leq \xi_{\mP}.
    \end{align*}
    Then, one can revise the reward function of $\MO$ to $r_{\tO}'$ defined by $r_{\tO,h}'(s_h,a_h) = r_{\tO,h}(s_h,a_h) + \xi_r + (H-h)\xi_{\mP}$, and the new MDP $M_{\tO'} = \{\cS,\cA,\mP_{\tO}, r_{\tO}',H\}$ has optimal value dominance on $\ME$.
\end{example}
\paragraph{Proofs for Examples Above}
Ex.~\ref{example:OVD_1} is obvious, we just prove the rest two. First of all, for arbitrary $h\in[H]$ and $s_h \in \cS_h$, we should have:
\begin{align*}
    & \VE{,h}^*(s_h) - \VO{,h}^*(s_h)\leq \QE{,h}^*(s_h,\piE{}^*) - \QO{,h}^*(s_h,\piE{}^*) \\
    = & r_{\tE,h}(s_h,\piE{}^*) - r_{\tO,h}(s_h,\piE{}^*) + (\mPE{,h} - \mPO{,h})\VE{,h+1}^*(s_h,\piE{}^*) + \mPO{,h}(\VE{,h+1}^*-\VO{,h+1}^*)(s_h,\piE{}^*)\\
    \leq & r_{\tE,h}(s_h,\piE{}^*) - r_{\tO,h}(s_h,\piE{}^*) + (\mPE{,h} - \mPO{,h})\VE{,h+1}^*(s_h,\piE{}^*)+ \mPO{,h}(\VE{,h+1}^*-\QO{,h+1}^*(\cdot,\piE{}^*))(s_h,\piE{}^*) \\
    \leq & ...\\
    \leq & \EE_{\MO, \piE{}^*}[\sum_{\ph = h}^H r_{\tE,h}(s_\ph,a_\ph) - r_{\tO,\ph}(s_\ph,a_\ph) + (\mPE{,\ph} - \mPO{,\ph})\VE{,\ph+1}^*(s_\ph,a_\ph) |s_h]\\
    \leq & \EE_{\MO, \piE{}^*}[\sum_{\ph = h}^H |r_{\tE,h}(s_\ph,a_\ph) - r_{\tO,\ph}(s_\ph,a_\ph)| + (H-h)\cdot \|\mPE{,\ph}(\cdot|s_\ph,a_\ph) - \mPO{,\ph}(\cdot|s_\ph,a_\ph)\|_1 |s_h],
\end{align*}
Therefore, in Example \ref{example:OVD_2}, we should expect:
\begin{align*}
    \VE{,h}^*(s_h) - \VO{,h}^*(s_h) \leq (H - h)\cdot \frac{\Delta_{\min}}{4H(H+1)} + (H-h) \cdot (H-h) \cdot \frac{\Delta_{\min}}{4H^2(H+1)}  \leq \frac{\Delta_{\min}}{2(H+1)}.
\end{align*}
Besides, for Example \ref{example:OVD_3}, we have:
\begin{align*}
    \VE{,h}^*(s_h) - V_{\tO',h}^*(s_h) \leq & \EE_{\MO,\piE{}^*}[\sum_{\ph = h}^H r_{\tE,h}(s_\ph,a_\ph) - r_{\tO,\ph}(s_\ph,a_\ph) + \xi_r \\
    & \quad \quad \quad \quad  + (H-h)\xi_{\mP} + (H-h) \cdot \|\mPE{,\ph}(\cdot|s_\ph,a_\ph) - \mPO{,\ph}(\cdot|s_\ph,a_\ph)\|_1 |s_h]\\
    \leq & 0
\end{align*}
Therefore, both Example \ref{example:OVD_2} and \ref{example:OVD_3} satisfy Assump. \ref{assump:opt_value_dominance}.

\subsection{Detailed Discussion on Open Problems}\label{appx:open_problems}
We believe there are many interesting directions to follow in the future and highlight in three aspects: 

First of all, we conjecture our unique optimal policy assumption can be relaxed and the $O(\frac{1}{H})$ factor in Def.~\ref{def:transferable_states} and Def.~\ref{def:transferable_states_MT} can be removed by advanced techniques.
It's also important to study how to get rid of lower bound knowledge in $\Delta_{\min}$.

Secondly, in TRL-MST setting, for those $s_h$ such that there are multiple source tasks $M_{\tO,\task_1},M_{\tO,\task_2},...,M_{\tO,\task_j} \in \cM_{\tO}$ close to $\ME$, beyond the constant regret, one may consider to integrate the information from those source tasks together to further accelerate the learning; 
Moreover, although we do not make it explicitly, it's possible to combine our techniques in Sec.~\ref{sec:multiple_source_task} with existing MT-RL algorithms to develop algorithms with guarantees about the reduction on not only the total regret but also some specific tasks.

Finally, although we show in Sec.~\ref{sec:lower_bound} that robust transfer objective requires OVD assumption, and the model difference tolerance is at most $O(\Delta_{\min})$, we conjecture that, there might exists milder assumptions about the structure of source and target tasks and the prior knowledge about it, which may eliminate out our hard instance.
Additionaly, in some cases, it may be reasonable to relax the objective by allowing some chance of negative transfer in part of target tasks. Then, we can do more aggressive transfer without too much concern on the algorithm's overall performance. These potential directions are left for future work.

\section{Proofs for Lower Bound}\label{appx:lower_bound}
\ThmNeceOVD*
\begin{proof}
    Consider the two-armed bandit setting.
    Given arbitrary $\Delta,\mu \in (0,1)$ satisfying $0 < \mu - \Delta < \mu < \mu + \Delta$, we can construct two two-armed Bernoulli bandit problem $M$ and $M'$ such that:
    \begin{align*}
        \mu_{M}(1) = \mu_{M'}(1) = \mu;\quad \mu_{M}(2) = \mu - \Delta;\quad \mu_{M'}(2) = \mu + \Delta.
    \end{align*}
    We choose $M$ to be the low-tier task, i.e. $\MO = M$, and choose $M$ and $M'$ to be the high-tier task. 
    Note that the minimal gap in $M$ and $M'$ is $\Delta$, and $\mu_M(1) = \mu_{M'}(2) - \Delta < \mu_{M'}(2) - \frac{\Delta}{2}$, which implies $\MO$ does not have optimal value dominance on $\ME$ when $\MO = M$ and $\ME = M'$.

    Now, we consider the following learning process: the learner will get access to the low-tier task $\MO=M$, and the high-tier task $\ME$ will be uniformly randomly selected between $M$ and $M'$, while the learner does not know which it is.
    Without loss of generality, we consider deterministic algorithms $\algO$ and $\algE$ (since one can first generate the randomness before the learning process), i.e. for arbitrary step $k$, given the interaction history $\tau_k := (a_{\tE}^1,r_{\tE}^1,a_{\tO}^1,r_{\tO}^1,...,a_{\tE}^{k-1},r_{\tE}^{k-1},a_{\tO}^{k-1},r_{\tO}^{k-1})$, the policy $(\pi^k_{\tO},\pi^k_{\tE})$ produced by $\algO(\tau_k)$ and $\algE(\tau_k)$ is fixed, where $a^k_\tO,r^k_\tO$ (or $a^k_\tE,r^k_\tE$) denotes the arm pulled and the reward observed in task $\MO$ (or $\ME$) at iteration $k$.

    In the following, we will use $\Pr_{\algO,\algE}^{\MO,\ME}(\cdot)$ to denote the probability if the learner use algorithm pair $(\algO,\algE)$ and solve task pair $(\MO,\ME)$.
    Note that the pseudo-regret of $\algE$ when $\ME=M$ can be written as:
    \begin{align}
        \Regret_K(\ME;\ME=M) =& \sum_{\tau_K:\Pr{}_{\algO,\algE}^{M,M}(\tau_K)>0} \Pr{}_{\algO,\algE}^{M,M}(\tau_K) N_{\tE}(2;\tau_K)\Delta.\label{eq:decompose_pseudo_regret}
    \end{align}
    where $N_{\tE}(i;\tau_K)$ denotes the number of times the $i$-th arm is pulled in task $\ME$ in trajectory $\tau_K$.

    Because both $\MO$ and $\ME$ are two-armed Bernoulli bandits, and each arm in those MDPs has non-zero probability mass on both value 0 and 1 and the algorithms are deterministic, for arbitrary $k \geq 1$ and $\tau_k$, $\Pr{}_{\algO,\algE}^{M,M}(\tau_k) > 0$ if and only if $\Pr{}_{\algO,\algE}^{M,M}(\tau_k) > 0$.

    Now, we consider the following probability ratio, for arbitrary $\tau_k$ with $\Pr{}_{\algO,\algE}^{M,M}(\tau_k) > 0$:
    \begin{align*}
        \frac{\Pr{}_{\algO,\algE}^{M,M'}(\tau_k)}{\Pr{}_{\algO,\algE}^{M,M}(\tau_k)} =& \frac{\Pr{}_{M'}(r_{\tE}=r_{\tE}^{k-1}|a_{\tE}^{k-1})}{\Pr{}_{M}(r_{\tE}=r_{\tE}^{k-1}|a_{\tE}^{k-1})} \cdot \frac{\Pr{}_{\algO,\algE}^{M,M'}(\tau_{k-1})}{\Pr{}_{\algO,\algE}^{M,M}(\tau_{k-1})} \tag{$\Pr{}_{M}(\cdot)$ denotes the probability of event on model $M$}\\
        =& \prod_{k'=1}^k \frac{\Pr{}_{M'}(r_{\tE}=r_{\tE}^{k'}|a_{\tE}^{k'})}{\Pr{}_{M}(r_{\tE}=r_{\tE}^{k'}|a_{\tE}^{k'})} \geq (\frac{1-\mu-\Delta}{1-\mu+\Delta})^{N_{\tE}(2;\tau_k)}.\numberthis\label{eq:lower_bound_ratio}
    \end{align*}
    where for the first equality, we use the fact that the algorithms are deterministic, and the randomness of $r_{\tO}^{k-1}$ only depends on $a_{\tO}^{k-1}$ so it cancels out; the last inequality is because that the ratio is 1 if $a_{\tE}^{k'}=1$ and the ratio can be lower bounded by $(1-\mu-\Delta)/(1-\mu+\Delta)$ otherwise,
    Therefore, combining with Eq.~\eqref{eq:decompose_pseudo_regret}, we have:
    \begin{align*}
        \Regret_K(\ME;\ME=M') =&  \sum_{\tau_K} \Pr{}_{\algO,\algE}^{M,M'}(\tau_K) N_{\tE}(1;\tau_K)\Delta\\
        =& \sum_{\tau_K: \Pr{}_{\algO,\algE}^{M,M}(\tau_K) > 0} \Pr{}_{\algO,\algE}^{M,M}(\tau_K)
        \frac{\Pr{}_{\algO,\algE}^{M,M'}(\tau_K)}{\Pr{}_{\algO,\algE}^{M,M}(\tau_K)} N_{\tE}(1;\tau_K) \Delta\\
        \geq & \sum_{\tau_K: \Pr{}_{\algO,\algE}^{M,M}(\tau_K) > 0} \Pr{}_{\algO,\algE}^{M,M}(\tau_K)
        (\frac{1-\mu-\Delta}{1-\mu+\Delta})^{N_{\tE}(2;\tau_K)} N_{\tE}(1;\tau_K) \Delta \numberthis\label{eq:regret_ratio}
    \end{align*}
    Suppose the algorithm pair $(\algO,\algE)$ can achieve constant regret $C$ when $(\MO,\ME)=(M,M)$, i.e.
    \begin{align*}
        \EE_{\algO,\algE,M,M}[N_{\tE}(2;\tau_K)\Delta]\leq C,\quad \forall K \geq 1.
    \end{align*}
    then, according to Markov inequality, for arbitrary constant $\delta \in (0,1)$ we have:
    \begin{align}
        \Pr{}_{\algO,\algE}^{M,M}(N_{\tE}(2;\tau_K) \leq \frac{C}{\Delta \delta}) \geq 1-\delta,\quad \forall K \geq 1.
    \end{align}
    which is equivalent to (note that $\tau_K$ is the random variable) 
    \begin{align*}
        \sum_{\tau_K:N_{\tE}(2;\tau_K)\leq \frac{C}{\Delta \delta}}\Pr{}_{\algO,\algE}^{M,M}(\tau_K) \geq 1-\delta.
    \end{align*}
    Combining with Eq.~\eqref{eq:regret_ratio}, by choosing an arbitrary constant $\delta \in (0,1)$, for arbitrary $K \geq 1$, we have:
    \begin{align*}
        \Regret_K(\ME;\ME=M') \geq & \sum_{\tau_K:N_{\tE}(2;\tau_K)\leq \frac{C}{\Delta \delta}}\Pr{}_{\algO,\algE}^{M,M}(\tau_K) (\frac{1-\mu-\Delta}{1-\mu+\Delta})^{N_{\tE}(2;\tau_K)}N_{\tE}(1;\tau_K)\Delta\\
        \geq & \sum_{\tau_K:N_{\tE}(2;\tau_K)\leq \frac{C}{\Delta \delta}}\Pr{}_{\algO,\algE}^{M,M}(\tau_K) (\frac{1-\mu-\Delta}{1-\mu+\Delta})^{\frac{C}{\Delta\delta}}(K-\frac{C}{\Delta\delta})\Delta\\
        \geq & (1-\delta)\cdot (\frac{1-\mu-\Delta}{1-\mu+\Delta})^{\frac{C}{\Delta\delta}}(K-\frac{C}{\Delta\delta})\Delta\\
        = & O(K).
    \end{align*}
    which finishes the proof.
\end{proof}

\ThmDeltaminLB*
\begin{proof}
    We can construct three two-armed Bernoullis bandit problem $M,M'$ and $M''$ such that:
    \begin{align*}
        &\mu_M(1) = \mu,\quad \mu_M(2) = \mu - \Delta;\\
        &\mu_{M'}(1) = \mu - \Delta',\quad \mu_{M'}(2) = \mu - \Delta - \Delta';\\
        &\mu_{M''}(1) = \mu - \Delta',\quad \mu_{M''}(2) = \mu + \Delta - \Delta';
    \end{align*}
    where $\Delta$ and $\mu$ are chosen to satisfy $0 < \mu - 2\Delta < \mu < \mu + \Delta$, and $\Delta' \in [\frac{\Delta}{2}, \Delta]$. 
    Note that by construction, $\Delta$ is $\Delta_{\min}$.
    Now, consider the following learning process, the learner will be provided $M$ as the low-tier task $\MO$, and the high-tier task $\ME$ will be uniformly sampling from $\{M',M''\}$. 
    Easy to check that, $\MO=M$, has optimal value dominance on $\ME$ when $\ME = M'$ or $\ME = M''$. 
    Next, we want to show that, for arbitrary algorithm pair $(\algO,\algE)$, if the learner can achieve constant regret when $\ME=M'$, it must achieve linear regret when $\ME = M''$.

    The remaining proof is similar to the proof for Thm. \ref{thm:necessity_of_OVD}. First of all, we have:
    \begin{align}
        \Regret_K(\ME;\ME=M') =& \sum_{\tau_K:\Pr{}_{\algO,\algE}^{M,M'}(\tau_K)>0} \Pr{}_{\algO,\algE}^{M,M'}(\tau_K) N_{\tE}(2;\tau_K)\Delta.\label{eq:decompose_pseudo_regret_2}
    \end{align}
    As an analogue of Eq.~\eqref{eq:lower_bound_ratio}, we have:\begin{align*}
        \frac{\Pr{}_{\algO,\algE}^{M,M'}(\tau_k)}{\Pr{}_{\algO,\algE}^{M,M''}(\tau_k)}\geq (\frac{\mu-\Delta-\Delta'}{\mu+\Delta-\Delta'})^{N_{\tE}(2;\tau_k)}.
    \end{align*}
    Combining with Eq.~\eqref{eq:decompose_pseudo_regret_2}, we have:
    \begin{align*}
        \Regret_K(\ME;\ME=M'') \geq & \sum_{\tau_K: \Pr{}_{\algO,\algE}^{M,M'}(\tau_K) > 0} \Pr{}_{\algO,\algE}^{M,M'}(\tau_K)
        (\frac{\mu-\Delta-\Delta'}{\mu+\Delta-\Delta'})^{N_{\tE}(2;\tau_K)} N_{\tE}(1;\tau_K) \Delta \numberthis\label{eq:regret_ratio_2}
    \end{align*}
    Suppose the algorithm pair $(\algO,\algE)$ can achieve constant regret $C$ when $(\MO,\ME)=(M,M')$, we must have:
    \begin{align*}
        \sum_{\tau_K:N_{\tE}(2;\tau_K)\leq \frac{C}{\Delta \delta}}\Pr{}_{\algO,\algE}^{M,M'}(\tau_K) \geq 1-\delta.
    \end{align*}
    By choosing an arbitrary fixed constant $\delta \in (0,1)$, for arbitrary $K$, we have:
    \begin{align*}
        \Regret_K(\ME;\ME=M'') \geq & \sum_{\tau_K:N_{\tE}(2;\tau_K)\leq \frac{C}{\Delta \delta}}\Pr{}_{\algO,\algE}^{M,M'}(\tau_K) (\frac{1-\mu-\Delta}{1-\mu+\Delta})^{N_{\tE}(2;\tau_K)}N_{\tE}(1;\tau_K)\Delta\\
        \geq & \sum_{\tau_K:N_{\tE}(2;\tau_K)\leq \frac{C}{\Delta \delta}}\Pr{}_{\algO,\algE}^{M,M'}(\tau_K) (\frac{1-\mu-\Delta}{1-\mu+\Delta})^{\frac{C}{\Delta\delta}}(K-\frac{C}{\Delta\delta})\Delta\\
        \geq & (1-\delta)\cdot (\frac{\mu-\Delta-\Delta'}{\mu+\Delta-\Delta'})^{\frac{C}{\Delta\delta}}(K-\frac{C}{\Delta\delta})\Delta\\
        = & O(K).
    \end{align*}
    which finishes the proof.
\end{proof}

\section{Proofs for Tiered MAB with Single Source/Low-Tier Task}\label{appx:proof_for_bandits}
\begin{lemma}[Concentration Inequality]\label{lem:bandit_concentration}
    In Alg. \ref{alg:Bandit_Setting}, at each iteration $k$, we have:
    \begin{align*}
        \Pr(|\muO(i) - \hmuO{}^k(i)| \geq \sqrt{\frac{2 \alphaO \log f(k)}{\NO{}^{k}(i)}}) \leq \frac{2}{f(k)^\alphaO} \leq \frac{1}{8Ak^{2\alpha}}, \quad \forall i \in [A]\\
        \Pr(|\muE(i) - \hmuE{}^k(i)| \geq \sqrt{\frac{2 \alphaE \log f(k)}{\NE{}^{k}(i)}}) \leq \frac{2}{f(k)^\alphaO} \leq \frac{1}{8Ak^{2\alpha}}, \quad \forall i \in [A]
    \end{align*}
\end{lemma}
As a direct result, we have the following lemma:
\begin{restatable}{lemma}{LemUnderEst}[Valid Under Estimation]\label{lem:valid_under_est}
    For arbitrary $i\in[A]$, if $\muO(i) \leq \muE(i) + \epsilon$, for arbitrary iteration $k$ in Alg. \ref{alg:Bandit_Setting}, we have:
    \begin{align*}
        \Pr(\umuO{}(i) \leq \omuE{}(i) + \epsilon) \geq 1 - \frac{4}{f(k)^\alpha} \geq 1 - \frac{1}{4Ak^{2\alpha}}
    \end{align*}
\end{restatable}
\begin{proof}
    According to Lem. \ref{lem:bandit_concentration}, w.p. at least $1-\frac{4}{f(k)^\alpha}$ we have:
    \begin{align*}
        \Pr(\umuO{}(i) \leq \omuE{}(i) + \epsilon)\geq &\Pr(\{\umuO{}(i) \leq \muO(i) \}\cap \{\muE(i) + \epsilon \leq \omuE{}(i) + \epsilon\}) \geq 1-\frac{4}{f(k)^\alpha} \geq 1 - \frac{1}{4Ak^{2\alpha}}.
    \end{align*}
\end{proof}
Next, we recall two useful lemma: Lemma 4.2 and Lemma D.1 from \citep{huang2022tiered}.
\begin{lemma}[Property of UCB; Lem 4.2 in \citep{huang2022tiered}]\label{lem:upper_bound_of_Nk_geq_k_div_scalar}
    With the choice that $f(k)=1+16A^2(k+1)^2$, there exists a constant $c$, for arbitrary $i$ with $\DeltaO(i) > 0$ and arbitrary $\scalar\in[1, 4A]$, in UCB algorithm, we have:
    \begin{align*}
        \Pr(\NO{}^k(i) \geq \frac{k}{\scalar}) \leq \frac{2}{k^{2\alpha-1}},\quad\forall k \geq \scalar + c\cdot\frac{\alpha \scalar}{\DeltaO^2(i)} \log(1+\frac{\alpha A}{\Delta_{\min}}).
    \end{align*}
\end{lemma}
\begin{lemma}[Lemma D.1 in \citep{huang2022tiered}]\label{lem:combining_UCB_with_LCB_new}
    Given an arm $i$, we separate all the arms into two parts depending on whether its gap is larger than $\DeltaO(i)$ and define $\Gl_i:=\{\iota|\DeltaO(\iota) > \DeltaO(i) / 2\}$ and $\Gu_i:=\{\iota|\DeltaO(\iota) \leq \DeltaO(i) / 2\}$.
    With the choice that $f(k)=1+16A^2(k+1)^2$, there is a constant $c$, such that for arbitrary $i$ with $\DeltaO(i) > 0$, for $\underline{\pi}_{\tO}^k$ in Alg \ref{alg:Bandit_Setting}, there exists a constant $c$, such that:
    \begin{align}
        \Pr(i = \underline{\pi}_{\tO}^k) \leq 2/k^{2\alpha}+2A/k^{2\alpha-1},\quad\forall k \geq k_i := 8\alpha 
        c\Big(\sum_{\iota \in \Gl_i} \frac{1}{\DeltaO^2(\iota)} + 
        \frac{4|\Gu_i|}{\DeltaO^2(i)}\Big)\log(1+\frac{\alpha A}{\Delta_{\min}})\label{eq:def_ki}
    \end{align}
\end{lemma}

\begin{lemma}\label{lem:behavior_of_UCB}
    We denote $k_i' := 3A + c\cdot \frac{3 \alpha A}{\DeltaO^2(i)}\log(1+\frac{\alpha A}{\Delta_{\min}})$ and $\tilde{k}_i := 3 + c\cdot \frac{3 \alpha }{\DeltaO^2(i)}\log(1+\frac{\alpha A}{\Delta_{\min}})$, where $c$ is specified in Lem. \ref{lem:upper_bound_of_Nk_geq_k_div_scalar}, and denote $k_{\max} := \max_{i \neq i^*} \max\{k_i,~k_i'\}$, where $k_i$ is defined in Lemma \ref{lem:combining_UCB_with_LCB_new}, we have:
    \begin{align}
        &\Pr(i^* = \upiO{}^k) = 1 - \sum_{i \neq i^*} \Pr(i = \upiO{}^k) \geq 1 - \frac{2A}{k^{2\alpha}} - \frac{2A^2}{k^{2\alpha-1}},\quad \forall k \geq k_{\max}. \label{eq:converge_of_upiOk}\\
        &\Pr(\NO{}^k(i^*) > \frac{k}{2}) \geq \Pr(\NO{}^k(i^*) \geq \frac{2k}{3}) \geq 1 - \sum_{i\neq i^*}\Pr(\NO{}^k(i) \leq \frac{k}{3A}) \geq 1 - \frac{2A}{k^{2\alpha - 1}},\quad \forall k \geq k_{\max}.\label{eq:over_half}\\
        &\Pr(\NO{}^k(i) > \frac{k}{2}) \leq \Pr(\NO{}^k(i) \geq \frac{k}{3}) \leq \frac{2}{k^{2\alpha - 1}},\quad \forall k \geq \tilde{k}_i.
    \end{align}
\end{lemma}
\begin{proof}
    By applying Lem. \ref{lem:combining_UCB_with_LCB_new} and Lem. \ref{lem:upper_bound_of_Nk_geq_k_div_scalar} we can obtain the results.
\end{proof}

\begin{lemma}\label{lem:upper_bound_NEK}
    For arbitrary $K \geq A + 1$ and arbitrary $k_0 \leq K$, and $i\neq i^*_{\tE}$, we have:
    \begin{align*}
        \NE{}^{K}&(i) \leq  k_0 + \sum_{k=k_0 + 1}^{K} \mathbb{I}[\{\umuO{}^k(\underline{\pi}_\tO^k) \leq \omuE{}^k(\underline{\pi}_\tO^k) + \epsilon\} \cap \{\NO{}^k(\underline{\pi}_\tO^k) > k/\ratio\}\cap\{i = \upiO{}^k\}\cap\{\pi_{\tE}^k = i\}]\\
        &+\sum_{k=k_0 + 1}^{K} \mathbb{I}[0 \geq \hmuE{}^k(i^*_{\tE}) + \sqrt{\frac{2\alphaE \log f(k)}{\NE{}^k(i^*_{\tE})}} - \muE(i^*_{\tE})]\\
        &+\sum_{k=k_0 + 1}^{K} \mathbb{I}[\tilde{\mu}_{\tE}^k(i) + \sqrt{\frac{2\alphaE \log f(K)}{k}} - \muE(i) - \DeltaE(i)\geq 0]\numberthis\label{eq:decomposition_of_N}
    \end{align*}
\end{lemma}
\begin{proof}
    \begin{align*}
        &\NE{}^{K}(i) = \sum_{k=1}^{K} \mathbb{I}[\pi_{\tE}^k = i]\leq k_0 + \sum_{k_0 + 1}^K \mathbb{I}[\pi_{\tE}^k = i] \\
        \leq & k_0 + \sum_{k=k_0 + 1}^{K} \mathbb{I}[\underbrace{\{\umuO{}^k(\underline{\pi}_\tO^k) \leq \omuE{}^k(\underline{\pi}_\tO^k) + \epsilon\} \cap \{\NO{}^k(\underline{\pi}_\tO^k) > k/\ratio\}\cap\{i = \upiO{}^k\}}_{e_1}\cap\{\pi_{\tE}^k = i\}]\\
        &+\sum_{k=k_0 + 1}^{K} \mathbb{I}[\underbrace{\{\hmuE{}^k(i) + \sqrt{\frac{2\alphaE \log f(k)}{\NE{}^k(i)}} - \muE(i) - \DeltaE(i)\geq \hmuE{}^k(i^*_{\tE}) + \sqrt{\frac{2\alphaE \log f(k)}{\NE{}^k(i^*_{\tE})}} - \muE(i^*_{\tE})\}}_{e_2} \cap \{\pi_{\tE}^k = i\}]. \tag{If $\pi_{\tE}^k = i$ happens, one of $e_1$ and $e_2$ must hold}
    \end{align*}
    For the second term, we have:
    \begin{align*}
        &\sum_{k=k_0 + 1}^{K} \mathbb{I}[\{\hmuE{}^k(i) + \sqrt{\frac{2\alphaE \log f(k)}{\NE{}^k(i)}} - \muE(i) - \DeltaE(i)\geq \hmuE{}^k(i^*_{\tE}) + \sqrt{\frac{2\alphaE \log f(k)}{\NE{}^k(i^*_{\tE})}} - \muE(i^*_{\tE})\} \cap \{\pi_{\tE}^k = i\}]\\
        \leq &\sum_{k=k_0 + 1}^{K} \mathbb{I}[0 \geq \hmuE{}^k(i^*_{\tE}) + \sqrt{\frac{2\alphaE \log f(k)}{\NE{}^k(i^*_{\tE})}} - \muE(i^*_{\tE})]\\
        &+\sum_{k=k_0 + 1}^{K} \mathbb{I}[\{\hmuE{}^k(i) + \sqrt{\frac{2\alphaE \log f(k)}{\NE{}^k(i)}} - \muE(i) - \DeltaE(i)\geq 0\}\cap\{\pi_{\tE}^k = i\}] \tag{$\mathbb{I}[a \geq b] \leq \mathbb{I}[a \geq c] + \mathbb{I}[c \geq b]$; $\mathbb{I}[a\cap b] \leq \mathbb{I}[a]$}\\
        \leq & \sum_{k=k_0 + 1}^{K} \mathbb{I}[0 \geq \hmuE{}^k(i^*_{\tE}) + \sqrt{\frac{2\alphaE \log f(k)}{\NE{}^k(i^*_{\tE})}} - \muE(i^*_{\tE})]+\sum_{k=k_0 + 1}^{K} \mathbb{I}[\tilde{\mu}_{\tE}^k(i) + \sqrt{\frac{2\alphaE \log f(K)}{k}} - \muE(i) - \DeltaE(i)\geq 0]
    \end{align*}

    where in the last step, $\tilde{\mu}_{\tE}^k(i)$ is defined to be the average of $k$ random samples from reward distribution of arm $i$ in $\ME$, and we replace $\NE{}^k(i)$ in the denominator with increasing $k$ since the indicator function equals 1 only when $\{\pi_{\tE}^k = i\}$, which implies that $\NE{}^k(i)$ should increase by 1. 
\end{proof}

\ThmBanditRegret*
\begin{proof}
    We first study the case when $\ME$ and $\MO$ satisfy Def. \ref{def:close_state}.
    \paragraph{Case 1: $\ME$ and $\MO$ are $\epsilon$-close}
    In this case, since $i^*_{\tO}=i^*_{\tE}$, we use $i^*$ to denote the common optimal arm. As a result of Lem. \ref{lem:valid_under_est}, we have:
    \begin{align*}
        \Pr(\umuO{}^k(i^*) \leq \omuE{}^k(i^*) + \epsilon) \geq 1 - \frac{1}{4Ak^{2\alpha}}.
    \end{align*}
    We consider the same $k_{\max}$ defined in Lem. \ref{lem:behavior_of_UCB}, as a result of Lem. \ref{lem:behavior_of_UCB}, for arbitrary $K \geq k_{\max} + 1$,
    \begin{align*}
        \sum_{k = k_{\max} + 1}^K \Pr(\pi_{\tE}^k \neq i^*_{\tE}) \leq& \sum_{k = k_{\max} + 1}^K \Pr(\umuO{}^k(i^*) > \omuE{}^k(i^*) + \epsilon) + \Pr(\NO{}^k(i^*) \leq \frac{k}{2}) + \Pr(i^* \neq \upiO{}^k)\\
        \leq & \sum_{k = k_{\max} + 1}^K \frac{1}{4Ak^{2\alpha}} + \frac{2A}{k^{2\alpha}} + \frac{2A^2}{k^{2\alpha-1}} + \frac{2A}{k^{2\alpha - 1}}\\
        \leq & \sum_{k = k_{\max} + 1}^\infty \frac{8A^2}{k^{2\alpha-1}} \leq \frac{8A^2}{(2\alpha-2)k_{\max}^{2\alpha - 2}}.
    \end{align*}
    Therefore, all we need to do is to upper bound the regret up to step $k_{\max}$. 
    In the following, we separately upper bound $\EE[\NE{}^{k_{\max}}(i)]$ for $i\neq i^*$ for two cases depending on the comparison between $\DeltaE(i)$ and $\DeltaO(i)$. 
    \paragraph{Case 1-(a) $0<\DeltaE(i) \leq 4\DeltaO(i)$} 
    Recall $\tk_i$ in Lem. \ref{lem:behavior_of_UCB}. In this case, since $\DeltaO^{-1}(i) \leq \DeltaE^{-1}(i)$, we have $\tk_i = O(\frac{1}{\DeltaE^2(i)}\log\frac{A}{\Delta_{\min}})$, and by taking expectation over Eq.~\eqref{eq:decomposition_of_N}:
    \begin{align*}
        \EE[\NE{}^{K}(i)] \leq & \tk_i + \sum_{k=\tk_i + 1}^{K} \Pr(\{\NO{}^k(\underline{\pi}_\tO^k) \geq \frac{k}{2}\}) + \sum_{k=1}^{K} \Pr(0 \geq \hmuE{}^k(i_{\tE}^*) + \sqrt{\frac{2\alphaE \log f(k)}{\NE{}^k(i_{\tE}^*)}} - \muE(i_{\tE}^*))\\
        &+\EE[\sum_{k=1}^{K} \mathbb{I}[\tilde{\mu}_{\tE}^k(i) + \sqrt{\frac{2\alphaE \log f(k)}{k}} - \muE(i) - \DeltaE(i)\geq 0]]\\
        \leq & \tk_i + \sum_{k=\tk_i + 1}^{K} \frac{2}{k^{2\alpha - 1}} + \sum_{k=1}^K \frac{1}{8Ak^{2\alpha}} + 1 + \frac{2}{\DeltaE^2(i)}(\alpha\log f(K) + \sqrt{\pi \alpha\log f(K)} + 1)\tag{Lem. 8.2 in \citep{lattimore2020bandit}}\\
        = & O(\frac{1}{\DeltaE^2(i)}\log \frac{ AK}{\Delta_{\min}}) 
    \end{align*}
    \paragraph{Case 1-(b) $\DeltaE(i) > 4\DeltaO(i) > 0$}
    We introduce $\bar{k}_i := \frac{c_{\tE,i}\alpha}{\DeltaE^2(i)}\log \frac{A}{\Delta_{\min}}$, where $c_{\tE,i}$ is the minimal constant, such that when $k \geq \frac{c_{\tE,i}\alpha}{\DeltaE^2(i)}\log \frac{\alpha A}{\Delta_{\min}}$, we always have $k \geq  \frac{256\alpha\log f(k)}{\DeltaE^2(i)}$. Therefore, for all $k \geq \bar{k}_i$, $\NO{}^k(i) > \frac{k}{2}$ implies $\NO{}^k(i) \geq \frac{128\alphaO\log f(k)}{\DeltaE^2(i)}$ and we have:
    \begin{align*}
        & \mathbb{I}[\{\umuO{}^k(i) \leq \omuE{}^k(i) + \epsilon\}\cap\{\NO{}^k(i) \geq \frac{k}{2}\}\cap\{i = \upiO{}^k\}\cap\{\pi_{\tE}^k = i\}] \\
        =&\mathbb{I}[\{\umuO{}^k(i) - \muO(i) + \frac{\DeltaE(i)}{4} \leq \omuE{}^k(i) \pm \muE(i) \pm \muE(i^*) \pm \muO(i^*) - \muO(i) + \epsilon + \frac{\DeltaE(i)}{4}\}\\
        &\quad \cap\{\NO{}^k(i) \geq \frac{128\alphaO\log f(k)}{\DeltaE^2(i)}\}\cap\{i = \upiO{}^k\}\cap\{\pi_{\tE}^k = i\}]\\
        \leq &\mathbb{I}[\{\umuO{}^k(i) - \muO(i) + \frac{\DeltaE(i)}{4}\leq \omuE{}^k(i) - \muE(i) - \DeltaE(i) + \frac{\DeltaE(i)}{4} + \frac{\DeltaE(i)}{8} + \frac{\DeltaE(i)}{8} + \frac{\DeltaE(i)}{4}\} \tag{$\DeltaO(i)\leq \frac{\DeltaE(i)}{4}$; Optimal value dominance ($\muE(i^*) - \muO(i^*) \leq \frac{\Delta_{\min}}{2}\leq \frac{\DeltaE(i)}{8}$); $\epsilon < \frac{\Delta_{\min}}{4}\leq \frac{\DeltaE(i)}{8}$}\\
        & \quad \cap\{\NO{}^k(i) \geq \frac{128\alphaO\log f(k)}{\DeltaE^2(i)}\}\cap\{i = \upiO{}^k\}\cap\{\pi_{\tE}^k = i\}]\\
        \leq& \mathbb{I}[\{\umuO{}^k(i) - \muO(i) + \frac{\DeltaE(i)}{4}\leq \omuE{}^k(i) - \muE(i) -  \frac{\DeltaE(i)}{4}\}\cap\{\NO{}^k(i) \geq \frac{128\alphaO\log f(k)}{\DeltaE^2(i)}\}\cap\{i = \upiO{}^k\}\cap\{\pi_{\tE}^k = i\}] \tag{$\DeltaO(i) + \frac{\Delta_{\min}}{2} + \epsilon + \frac{\DeltaE(i)}{4} \leq \frac{\DeltaE(i)}{4} + \frac{\DeltaE(i)}{8} + \frac{\DeltaE(i)}{16} + \frac{\DeltaE(i)}{4} < \frac{3\DeltaE(i)}{4}$}\\
        \leq& \mathbb{I}[\{\umuO{}^k(i) - \muO(i) + \frac{\DeltaE(i)}{4}\leq 0\}\cap\{\NO{}^k(i) \geq \frac{128\alphaO\log f(k)}{\DeltaE^2(i)}\}\cap\{i = \upiO{}^k\}\cap\{\pi_{\tE}^k = i\}] \\
        & + \mathbb{I}[\{0\leq \omuE{}^k(i) - \muE(i) -  \frac{\DeltaE(i)}{4}\}\cap\{i = \upiO{}^k\}\cap\{\pi_{\tE}^k = i\}]\\
        =& \mathbb{I}[\{\hmuO{}^k(i) - \muO(i) \leq \sqrt{\frac{2\alphaO \log f(k)}{\NO{}^k(i)}} - \frac{\DeltaE(i)}{4}\}\cap\{\NO{}^k(i) \geq \frac{128\alphaO\log f(k)}{\DeltaE^2(i)}\}] \\
        & + \mathbb{I}[\{\hmuE{}^k(i) - \muE(i) + \sqrt{\frac{2\alphaE \log f(k)}{\NE{}^k(i)}} \geq \frac{\DeltaE(i)}{4}\}\cap\{i = \upiO{}^k\}\cap\{\pi_{\tE}^k = i\}] \\
        \leq& \mathbb{I}[\hmuO{}^k(i) - \muO(i) \leq -\sqrt{\frac{2\alphaO \log f(k)}{\NO{}^k(i)}}] + \mathbb{I}[\{\hmuE{}^k(i) - \muE(i) + \sqrt{\frac{2\alphaE \log f(k)}{\NE{}^k(i)}} \geq \frac{\DeltaE(i)}{4}\}\cap\{i = \upiO{}^k\}\cap\{\pi_{\tE}^k = i\}].\numberthis\label{eq:upper_bound_case1b}
    \end{align*}
    By taking the expectation over both sides of Eq.~\eqref{eq:decomposition_of_N}, we have:
    \begin{align*}
        &\EE[\NE{}^{K}(i)] \\
        \leq & \bar{k}_i + \sum_{k=\bar{k}_i + 1}^{K} \Pr(\hmuO{}^k(i) - \muO(i) \leq -\sqrt{\frac{2\alphaO \log f(k)}{\NO{}^k(i)}}) + \sum_{k=1}^{K} \Pr(0 \geq \hmuE{}^k(i_{\tE}^*) + \sqrt{\frac{2\alphaE \log f(k)}{\NE{}^k(i_{\tE}^*)}} - \muE(i_{\tE}^*))\\
        & + \EE[\sum_{k=\bar{k}_i + 1}^{K}\mathbb{I}[\{\hmuE{}^k(i) - \muE(i) + \sqrt{\frac{2\alphaE \log f(k)}{\NE{}^k(i)}} \geq \frac{\DeltaE(i)}{4}\}\cap\{i = \upiO{}^k\}\cap\{\pi_{\tE}^k = i\}]] \\
        &+\EE[\sum_{k=1}^{K} \mathbb{I}[\tilde{\mu}_{\tE}^k(i) + \sqrt{\frac{2\alphaE \log f(k)}{k}} - \muE(i) - \DeltaE(i)\geq 0]]\\
        \leq & \bar{k}_i + \sum_{k=\bar{k}_i + 1}^{K} \Pr(\hmuO{}^k(i) - \muO(i) \leq -\sqrt{\frac{2\alphaO \log f(k)}{\NO{}^k(i)}})  + \sum_{k=1}^{K} \Pr(0 \geq \hmuE{}^k(i_{\tE}^*) + \sqrt{\frac{2\alphaE \log f(k)}{\NE{}^k(i_{\tE}^*)}} - \muE(i_{\tE}^*))\\
        &+ 2 \EE[\sum_{k=1}^{K} \mathbb{I}[\tilde{\mu}_{\tE}^k(i) + \sqrt{\frac{2\alphaE \log f(k)}{k}} - \muE(i) \geq \frac{\DeltaE(i)}{4}]]\\
        \leq & \bar{k}_i + \sum_{k=\bar{k}_i + 1}^{K} \frac{2}{k^{2\alpha - 1}} + 2 + \frac{64}{\DeltaE^2(i)}(\alpha\log f(K) + \sqrt{\pi \alpha\log f(K)} + 1)\tag{Lem. 8.2 in \citep{lattimore2020bandit}}\\
        = & O(\frac{1}{\DeltaE^2(i)}\log \frac{AK}{\Delta_{\min}}).
    \end{align*}
    Since $k_{\max}=\Poly(A,\frac{1}{\Delta_{\min}})$, combining both cases, we have:
    \begin{align*}
        \Regret_{K}(\ME) =& \sum_{i \neq i^*}\DeltaE(i)\cdot\EE[\NE{}^{K}(i)]\leq \sum_{i \neq i^*} \DeltaE(i)\EE[\NE{}^{k_{\max}}(i)] + \sum_{k = k_{\max} + 1}^K \Pr(\pi_{\tE}^k \neq i^*_{\tE}) \\
        =&O(\sum_{i \neq i^*} \frac{1}{\DeltaE(i)}\log\frac{A}{\Delta_{\min}}).
    \end{align*}

    \paragraph{Case 2: $\ME$ and $\MO$ are not $\epsilon$-close}
    In that case, we will use $i^*_{\tO}$ and $i^*_{\tE}$ to denote optimal arm in $\MO$ and $\ME$, respectively, and either $i^*_{\tO} = i^*_{\tE}$ but $\muO{}(i^*_{\tO}) \geq \muO{}(i^*_{\tE}) + \epsilon$, or $i^*_{\tO}\neq i^*_{\tE}$ and as a result of Assump. \ref{assump:opt_value_dominance}:
    \begin{align}
        \muE{}(i_{\tO}^*) = \muE{}(i_{\tE}^*) - \DeltaE(i_{\tO}^*) \leq \muO{}(i_{\tO}^*) + \frac{\Delta_{\min}}{2} - \DeltaE(i_{\tO}^*) \leq \muO{}(i_{\tO}^*) - \frac{\DeltaE(i_{\tO}^*)}{2}.\label{eq:gap_between_optimal_arms}
    \end{align}
    Next, we first separate arms other than $i^*_{\tO}$ and $i^*_{\tE}$ into three cases:
    \paragraph{Case 2-(a) $i \neq i^*_{\tO}, i\neq i^*_{\tE}$ and $0<\DeltaE(i) \leq 4\DeltaO(i)$}
    The analysis is the same as Case 1-(a), and we have $\EE[\NE{}^{K}(i)]  = O(\frac{1}{\DeltaE^2(i)}\log \frac{AK}{\Delta_{\min}}).$
    \paragraph{Case 2-(b) $i \neq i^*_{\tO}, i\neq i^*_{\tE}$ and $\DeltaE(i) > 4\DeltaO(i) > 0$}
    The analysis is the same as Case 1-(b), and we have $\EE[\NE{}^{K}(i)] = O(\frac{1}{\DeltaE^2(i)}\log \frac{AK}{\Delta_{\min}})$.
    \paragraph{Case 2-(c) Others}
    If $i^*_{\tO} = i^*_{\tE}$, $\ME$ suffers no regret when choosing $i=i^*_{\tO}$, and therefore:
    \begin{align*}
        \Regret_{K}(\ME) =& \sum_{i \neq i_{\tE}^*} \DeltaE(i)\cdot \EE[\NE{}^k(i)] \leq \sum_{i \neq i_{\tE}^*} \DeltaE(i)\cdot O(\frac{1}{\DeltaE^2(i)}\log\frac{AK}{\Delta_{\min}}) =O(\sum_{i \neq i_{\tE}^*} \frac{1}{\DeltaE(i)}\log\frac{AK}{\Delta_{\min}}).
    \end{align*}
    In the following, we study the case when $i^*_{\tO} \neq i^*_{\tE}$.
    For arm $i = i^*_{\tO}$, we define $k_{\max}' = \frac{c_{\max}' \alpha}{\DeltaE(i_{\tO}^*)^2}\log\frac{\alpha A}{\Delta_{\min}}$, where $c_{\max}'$ is the minimal constant, such that for all $k \geq \frac{c_{\max}' \alpha}{\DeltaE(i_{\tO}^*)^2}\log\frac{\alpha A}{\Delta_{\min}}$, we always have $k \geq \frac{1024\alpha\log f(k)}{\DeltaE(i_{\tO}^*)^2}$.
    Similar to Eq.~\eqref{eq:upper_bound_case1b}, we check the following event for $k \geq k_{\max}'$:
    \begin{align*}
        &\mathbb{I}[\{\umuO{}^k(\underline{\pi}_\tO^k) \leq \omuE{}^k(\underline{\pi}_\tO^k) + \epsilon\} \cap \{\NO{}^k(\underline{\pi}_\tO^k) > k/\ratio\}\cap \{i_{\tO}^* = \upiO{}^k\}  \cap\{\pi_{\tE}^k = i_{\tO}^*\}]\\
        =&\mathbb{I}[\{\hmuO{}^k(i_{\tO}^*) - \muO{}(i_{\tO}^*) - \sqrt{\frac{2\alpha\log f(k)}{\NO{}^k(i_{\tO}^*)}} \leq \hmuE{}^k(i_{\tO}^*) - \muE{}(i_{\tO}^*) + (\muE{}(i_{\tO}^*) - \muO{}(i_{\tO}^*)) + \sqrt{\frac{2\alpha\log f(k)}{\NE{}^k(i_{\tO}^*)}} + \epsilon \} \\
        &~~\cap \{\NO{}^k(\underline{\pi}_\tO^k) > k/\ratio\}\cap\{i_{\tO}^* = \upiO{}^k\}\cap\{\pi_{\tE}^k = i_{\tO}^*\}]\\
        \leq&\mathbb{I}[\{\hmuO{}^k(i_{\tO}^*) - \muO{}(i_{\tO}^*) - \sqrt{\frac{2\alpha\log f(k)}{\NO{}^k(i_{\tO}^*)}} \leq \hmuE{}^k(i_{\tO}^*) - \muE{}(i_{\tO}^*) - \frac{\DeltaE(i_{\tO}^*)}{4} + \sqrt{\frac{2\alpha\log f(k)}{\NE{}^k(i_{\tO}^*)}}\} \tag{As a result of Eq.~\eqref{eq:gap_between_optimal_arms}, $\muE{}(i_{\tO}^*) - \muO{}(i_{\tO}^*) + \epsilon \leq -\frac{\DeltaE(i_{\tO}^*)}{2} + \frac{\Delta_{\min}}{4} \leq -\frac{\DeltaE(i_{\tO}^*)}{4}$}\\
        &~~\cap \{\NO{}^k(\underline{\pi}_\tO^k) > k/\ratio\}\cap\{i_{\tO}^* = \upiO{}^k\}\cap\{\pi_{\tE}^k = i_{\tO}^*\}]\\
        \leq&\mathbb{I}[\{\hmuO{}^k(i_{\tO}^*) - \muO{}(i_{\tO}^*) - \sqrt{\frac{2\alpha\log f(k)}{\NO{}^k(i_{\tO}^*)}} \leq - \frac{\DeltaE(i^*_{\tO})}{8}\} \cap \{\NO{}^k(\underline{\pi}_\tO^k) > k/\ratio\}\cap\{i_{\tO}^* = \upiO{}^k\}\cap\{\pi_{\tE}^k = i_{\tO}^*\}]\\
        &+ \mathbb{I}[\{\hmuE{}^k(i_{\tO}^*) - \muE{}(i_{\tO}^*) + \sqrt{\frac{2\alpha\log f(k)}{\NE{}^k(i_{\tO}^*)}}\geq \frac{\DeltaE(i_{\tO}^*)}{8}\}\cap \{\NO{}^k(\underline{\pi}_\tO^k) > k/\ratio\}\cap\{i_{\tO}^* = \upiO{}^k\}\cap\{\pi_{\tE}^k = i_{\tO}^*\}]\\
        \leq&\mathbb{I}[\hmuO{}^k(i_{\tO}^*) - \muO{}(i_{\tO}^*) \leq -\sqrt{\frac{2\alpha\log f(k)}{\NO{}^k(i_{\tO}^*)}}] + \mathbb{I}[\{\hmuE{}^k(i_{\tO}^*) - \muE{}(i_{\tO}^*) + \sqrt{\frac{2\alpha\log f(k)}{\NE{}^k(i_{\tO}^*)}}\geq \frac{\DeltaE(i_{\tO}^*)}{8}\}\cap\{\pi_{\tE}^k = i_{\tO}^*\}]. \numberthis\label{eq:case_2c}  
    \end{align*}
    Therefore, by taking the expectation on both side of Eq.~\eqref{eq:decomposition_of_N} and leveraging the above bound, we have:
    \begin{align*}
        \EE[\NE{}^k(i_{\tO}^*)]
        \leq& k_{\max}' + \EE[\sum_{k=k_{\max}' + 1}^{K} \mathbb{I}[\hmuO{}^k(i_{\tO}^*) - \muO{}(i_{\tO}^*) \leq -\sqrt{\frac{2\alpha\log f(k)}{\NO{}^k(i_{\tO}^*)}}]]\\
        & +\EE[\sum_{k=k_{\max}' + 1}^{K} \mathbb{I}[\{\hmuE{}^k(i_{\tO}^*) - \muE{}(i_{\tO}^*) + \sqrt{\frac{2\alpha\log f(k)}{\NE{}^k(i_{\tO}^*)}}\geq \frac{\DeltaE(i_{\tO}^*)}{8}\}\cap\{\pi_{\tE}^k = i_{\tO}^*\}]]\\
        &+\sum_{k=1}^{K} \Pr(0 \geq \hmuE{}^k(i_{\tE}^*) + \sqrt{\frac{2\alphaE \log f(k)}{\NE{}^k(i_{\tE}^*)}} - \muE(i_{\tE}^*))\\
        &+\EE[\sum_{k=1}^{K} \mathbb{I}[\tilde{\mu}_{\tE}^k(i_{\tO}^*) + \sqrt{\frac{2\alphaE \log f(K)}{k}} - \muE(i_{\tO}^*) - \DeltaE(i_{\tO}^*)\geq 0]]\\
        \leq& k_{\max}' + \sum_{k=k_{\max}' + 1}^{K} \Pr(\hmuO{}^k(i_{\tO}^*) - \muO{}(i_{\tO}^*) \leq -\sqrt{\frac{2\alpha\log f(k)}{\NO{}^k(i_{\tO}^*)}})\\
        &+\sum_{k=1}^{K} \Pr(0 \geq \hmuE{}^k(i_{\tE}^*) + \sqrt{\frac{2\alphaE \log f(k)}{\NE{}^k(i_{\tE}^*)}} - \muE(i_{\tE}^*))\\
        &+2\EE[\sum_{k=1}^{K} \mathbb{I}[\tilde{\mu}_{\tE}^k(i_{\tO}^*) + \sqrt{\frac{2\alphaE \log f(K)}{k}} - \muE(i_{\tO}^*) \geq \frac{\DeltaE(i_{\tO}^*)}{8}]]\\
        \leq & k_{\max}' + 2\sum_{k=k_{\max}' + 1}^{K} \frac{2}{f(k)^\alpha} + 2\cdot (1 + \frac{128}{\DeltaE^2(i_{\tO}^*)}(\alpha \log f(K) + \sqrt{\pi \alpha \log f(K)} + 1))\\
        =& O(\frac{1}{\DeltaE^2(i_{\tO}^*)}\log\frac{AK}{\Delta_{\min}})
    \end{align*}
    As a result, for arbitrary $K$, we also have:
    \begin{align*}
        \Regret_{K}(\ME) =& \sum_{i \neq i_{\tE}^*} \DeltaE(i)\cdot \EE[\NE{}^k(i)] \leq \sum_{i \neq i_{\tE}^*} \DeltaE(i)\cdot O(\frac{1}{\DeltaE^2(i)}\log\frac{AK}{\Delta_{\min}}) =O(\sum_{i \neq i_{\tE}^*} \frac{1}{\DeltaE(i)}\log\frac{AK}{\Delta_{\min}}).
    \end{align*}
\end{proof}

\section{Proofs for RL Setting with Single Source/Low-Tier Task}\label{appx:proof_for_RL}
\subsection{Missing Algorithms, Conditions and Notations}\label{appx:RL_add_Alg_Cond_Nota}
\begin{algorithm}[h!]
    \textbf{Input}: Dataset $D$.\\
    \For{$h=1,2,...,H$}{
        \For{$s_h \in \cS_h ,a_h \in \cA_h$}{
            Use $N_{h}(s_h,a_h)$ and $N_{h}(s_h,a_h,s_{h+1})$ to denote the number of times state, action (next state) occurs in the dataset $D$.\\
            $
            \hmP_{h}(\cdot|s_h,a_h)\gets 
            \begin{cases}
                0,\quad & \text{if } N_{h}(s_h,a_h)=0;\\
                \frac{N_{h}(s_h,a_h,\cdot)}{N_{h}(s_h,a_h)},\quad & \text{otherwise}.
            \end{cases}
            $\\
        }
    }
    \Return $\{\hmP_1,\hmP_2,...,\hmP_H\}$.
    \caption{ModelLearning}\label{alg:Model_Learning}
\end{algorithm}
\begin{condition}[Condition on $\algO$]\label{cond:requirement_on_algO}
    $\algO$ is an algorithm which returns deterministic policies at each iteration, and there exists $C_1,C_2$ only depending on $S,A,H$ and $\Delta_{\min}$ but independent of $k$, such that for arbitrary $k\geq 2$, we have $\Pr(\cEOk) \geq 1-\frac{1}{k^\alpha}$ for $\cEOk$ defined below:
    $$
    \cEOk := \{\sum_{\tk=1}^k \VO{,1}^*(s_1)-\VO{,1}^{\pi_{\tO}^\tk}(s_1) \leq C_1 + \alpha C_2 \log k\}.
    $$
\end{condition}
\begin{remark}
We consider such condition to avoid unnecessary discussion on analyzing $\algO$.
Although most of the existing near-optimal algorithms fixed the confidence level before the running of algorithm, as analyzed in Appx. G in \citep{huang2022tiered}, one may combine those algorithms with doubling trick to realize Cond.~\ref{cond:requirement_on_algO} for $\algE$, only at the cost of increase the the regret of $\algO$ to $O(\log^2 K)$.
\end{remark}

\begin{condition}[Condition on function \textbf{Bonus} in Alg. \ref{alg:RL_Setting}]\label{cond:bonus_term}
    Given a confidence sequence $\{\delta_k\}_{k=1}^K$ with $\delta_1,\delta_2,...,\delta_K\in (0,1/2)$, we define the following event at iteration $k\in[K]$ during the running of Alg.~\ref{alg:RL_Setting}:
    \begin{align*}
        \cEBk := & \bigcap_{\substack{(\cdot)\in\{\tE,\tO\},\\ h\in[H], \\ s_h\in\cS_h,a_h\in\cA_h}}\Big\{\{H\cdot \|\hmP_{(\cdot),h}^k(s_h,a_h) - \mP_{(\cdot),h}(s_h,a_h)\|_1 < \bonus_{(\cdot), h}^k(s_h,a_h) \leq B_1\sqrt{\frac{\log (B_2/\delta_k)}{N_{(\cdot),h}^k(s_h,a_h)}}\}\Big\}
    \end{align*}
    we consider the choice of \textbf{Bonus} such that there exists such a $B_1$ and $B_2$ only depending on $S,A,H$ but independent of $\delta_k$, $k$ or $\Delta$, and $\Pr(\cEBk)\geq 1-\delta_k$ holds for any $k\in[K]$.\footnote{Note that we do not require the knowledge of $\Delta_i$'s to compute $b_{k,h}$.} 
\end{condition}
For simplicity, in Cond.~\ref{cond:bonus_term}, we directly control the $l_1$-norm of the error of model estimation. Our analysis framework is compatible with other bonus term for sharper analysis. We provide a simple example for the choice of $B_1$ and $B_2$ for completeness:
\begin{example}\label{example:choice_of_B}
By Hoeffding's inequality and union bound, w.p. $1-\delta$, for all $s_{h+1},s_h,a_h$, we should have:
\begin{align*}
    |\hmP_{(\cdot),h}^k(s_{h+1}|s_h,a_h) - \mP_{(\cdot),h}(s_{h+1}|s_h,a_h)| \leq \sqrt{\frac{1}{2N^k_{(\cdot),h}(s_h,a_h)}\log \frac{S^2A}{\delta}}.
\end{align*}
which implies:
\begin{align*}
    H\cdot \|\hmP_{(\cdot),h}^k(\cdot|s_h,a_h) - \mP_{(\cdot),h}(\cdot|s_h,a_h)\|_1 = O(SH\sqrt{\frac{\log (SA/\delta)}{N^k_{(\cdot),h}(s_h,a_h)}})
\end{align*}
Therefore, one can choose $B_1=\Theta(SH)$ and $B_2=\Theta(SA)$.
\end{example}

Finally, we introduce the following concentration events about the deviation of the empirical visitation frequency and its expectation:
\begin{align*}
    \cECk :=& \bigcap_{\substack{h\in[H],\\s_h\in\cS_h,\\ a_h\in\cA_h}} \Big\{\{\frac{1}{2}\sum_{\tk=1}^k d^{\pi_{\tO}^\tk}(s_h,a_h) - \alpha \log (2SAHk) \leq \NO{,h}^k(s_h,a_h) \leq e\sum_{\tk=1}^k d^{\pi_{\tO}^\tk}(s_h,a_h) + \alpha \log (2SAHk)\}\\
    &\quad \cap\{\frac{1}{2}\sum_{\tk=1}^k d^{\pi_{\tE}^\tk}(s_h,a_h) - \alpha \log (2SAHk) \leq \NE{,h}^k(s_h,a_h) \leq e\sum_{\tk=1}^k d^{\pi_{\tE}^\tk}(s_h,a_h) + \alpha \log (2SAHk)\}\Big\}\numberthis\label{eq:def_concentration_event}.
\end{align*}

\subsection{Some Basic Lemma}
\begin{lemma}[Underestimation]\label{lem:underestimation}
    Given a \textbf{Bonus} satisfying Cond. \ref{cond:bonus_term}, at each iteration $k$ during the running of Alg. \ref{alg:RL_Setting}, on the events $\cEBk$ defined in Cond. \ref{cond:bonus_term}, $\forall h\in[H],\forall s_h\in\cS_h,a_h\in\cA_h$, we have:
    \begin{align}
        \uQE{,h}^{\pi_{\tE}^k}(s_h,a_h) \leq& \QE{,h}^{\pi_{\tE}^k}(s_h,a_h) \leq \QE{,h}^*(s_h,a_h) \label{eq:under_estimation_QE}\\
        \uQO{,h}^k(s_h,a_h) \leq& \QO{,h}^{\underline{\pi}_{\tO}^k}(s_h,a_h) \leq \QO{,h}^*(s_h,a_h) \label{eq:under_estimation_QO} \\
        \QO{,h}^*(s_h,a_h) - \uQO{,h}^k(s_h,a_h) \leq & 2\EE_{\piO{}^*}[\sum_{\ph=h}^H\min\{H, \bonus_{\tO, \ph}^k(s_\ph,a_\ph)\}|s_\ph,a_\ph].
    \end{align}
\end{lemma}
\begin{proof}
    Note that the relationship between $Q^\pi$ and $Q^*$ will always hold, and therefore, we only compare the underestimation part.

    According to the initialization, we have $\uQE{,h}^k(s_h,a_h)=\QE{,h}^{\underline{\pi}_{\tE}^k}(s_h,a_h)=\QO{,h}^*(s_h,a_h)$ at $h=H+1$. 
    Under the event of $\cEBk$, at iteration $k$, suppose we have the inequality Eq.~\eqref{eq:under_estimation_QE} holds for step $h+1$ for some $h\in[H]$, then at step $h$, we have:
    \begin{align*}
        &\uQE{,h}^{\pi_{\tE}^k}(s_h,a_h)-\QE{,h}^{{\pi}_{\tE}^k}(s_h,a_h) =\hmPE{,h}^k\uVE{,h+1}^{\pi_{\tE}^k}(s_h,a_h) - \bonus_{\tE,h}^k(s_h,a_h) - \mPE{,h}\VE{,h+1}^{{\pi}_{\tE}^k}(s_h,a_h)\\
        = & (\hmPE{,h}^k-\mPE{,h})\uVE{,h+1}^{\pi_{\tE}^k}(s_h,a_h) - \bonus_{\tE,h}^k(s_h,a_h) + \mPE{,h} (\uVE{,h+1}^{\pi_{\tE}^k} - \VE{,h+1}^{{\pi}_{\tE}^k})(s_h,a_h) \\
        \leq & \mPE{,h} (\uQE{,h+1}^k(\cdot, \pi_{\tE}^k) - \QE{,h+1}^{{\pi}_{\tE}^k}(\cdot, \pi_{\tE}^k))(s_h,a_h)\leq 0.
    \end{align*}
    where the first inequality is because $(\hmPE{,h}^k-\mPE{,h})\uVE{,h+1}^{\pi_{\tE}^k}(s_h,a_h) \leq H\cdot \|\hmPE{,h}^k(s_h,a_h)-\mPE{,h}(s_h,a_h)\|_1 \leq \bonus_{\tE,h}^k(s_h,a_h)$.
    The proof for Eq.~\eqref{eq:under_estimation_QO} is similar (except $\uQO{,h}^k$ is the greedy value $\underline{\pi}_{\tO}^k$ instead of $\pi^k_{\tO}$).
    Besides,
    \begin{align*}
        &\QO{,h}^*(s_h,a_h) - \uQO{,h}^k(s_h,a_h) \\
        =& \min\{H, \mPO{,h}\VO{,h+1}^*(s_h,a_h) - \hmPO{,h}^k\uVO{,h+1}^k(s_h,a_h) + \bonus_{\tO, h}^k(s_h,a_h)\} \\
        \leq & \min\{H, (\mPO{,h} - \hmPO{,h}^k)\uVO{,h+1}^k(s_h,a_h) + \bonus_{\tO, h}^k(s_h,a_h)\} + \mPO{,h}(\VO{,h+1}^* - \uVO{,h+1}^k)(s_h,a_h)\\
        \leq & 2\min\{H, \bonus_{\tO, h}^k(s_h,a_h)\} + \mPO{,h}(\VO{,h+1}^* - \uQO{,h+1}^k(\cdot,\piO{}^*))(s_h,a_h) \\
        \leq & ... \leq 2\EE_{\piO{}^*}[\sum_{\ph=h}^H\min\{H, \bonus_{\tO, \ph}^k(s_\ph,a_\ph)\}|s_h,a_h].
    \end{align*}
\end{proof}

\begin{theorem}[Extended from Thm. 4.7 in \citep{huang2022tiered}]\label{thm:algO_regret_vs_algP_density}
    For an arbitrary sequence of deterministic policies $\pi^1,\pi^2,...,\pi^K$, there must exist a sequence of deterministic optimal policies $\pi^{1,*}, \pi^{2,*},...,\pi^{K,*}$, such that $\forall h\in[H], s_h \in\cS_h, a_h\in\cA_h$:
    \begin{align*}
        |\sum_{k=1}^K d^{\pi^k}(s_h,a_h) - \sum_{k=1}^K d^{\pi^{k,*}}(s_h,a_h)| \leq \frac{1}{\Delta_{\min}} \Big(\sum_{k=1}^K V_{1}^*(s_1)-V_{1}^{\pi^k}(s_1)\Big).
    \end{align*}
\end{theorem}
\begin{proof}
    We first define the following events:
    \begin{align*}
        &\cE_{k,h,\pi}:=\{\pi_{k,h}(s_h)\neq \pi_{h}(s_h)\},\quad \tilde{\cE}_{k,h,\pi}:=\cE_{k,h,\pi}\cap \bigcap_{\ph=1}^h \cE_{k,\ph-1,\pi}^\complement,\quad \bar{\cE}_{k,\pi}:=\bigcup_{h=1}^H \cE_{k,h,\pi}.
    \end{align*}
    From Thm. 4.7 in \citep{huang2022tiered}, we already know $\sum_{k=1}^K d^{\pi^k}(s_h,a_h) - \sum_{k=1}^K d^{\pi^{k,*}}(s_h,a_h) \geq -\frac{1}{\Delta_{\min}} \Big(\sum_{k=1}^K V_{1}^*(s_1)-V_{1}^{\pi^k}(s_1)\Big)$.
    Next, we start with the second step in the proof of Lem. E.3 in \citep{huang2022tiered}: by choosing $\delta_{s_h,a_h}:=\mathbb{I}[S_h=s_h,A_h=a_h]$ (which equals one if the state action is $(s_h,a_h)$ at step $h$ and otherwise 0) as reward function, we have:
    \begin{align*}
        d^{\pi}(s_h,a_h) - d^{\pi^k}(s_h,a_h) = & V^{\pi}_1(s_1;\delta_{s_h,a_h}) - V^{\pi^k}_1(s_1;\delta_{s_h,a_h}) \\
        =&\EE_{\pi^k}[\sum_{\ph=1}^h \mathbb{I}[\tilde{\cE}_{k,\ph,\pi}](V^{\pi}_\ph(s_\ph;\delta_{s_h,a_h})-V^{\pi^k}_\ph(s_\ph;\delta_{s_h,a_h}))] \tag{$V_\ph^\pi = V_\ph^{\pi^k} =0$ for all $\ph \geq h+1$}
    \end{align*}
    Starting from here, we do something differently:
    \begin{align*}
        d^{\pi}(s_h,a_h) - d^{\pi^k}(s_h,a_h) \geq& \EE_{\pi^k}[\sum_{\ph=1}^h - \mathbb{I}[\tilde{\cE}_{k,\ph,\pi}] V^{\pi^k}_\ph(s_\ph;\delta_{s_h,a_h})] \tag{$V^{\pi}_\ph \geq 0$}\\
        \geq & -\EE_{\pi^k}[\sum_{\ph=1}^h \mathbb{I}[\tilde{\cE}_{k,\ph,\pi}]] \tag{$V^{\pi^k}_\ph \leq 1$}\\
        \geq& - \EE_{s_1,a_1,s_2,a_2...,s_H,a_H \sim \pi^k}[\mathbb{I}[\bar{\cE}_{k,\pi}]] = -\Pr(\bar{\cE}_{k,\pi}|\pi^k)
    \end{align*}
    Therefore, combining with the results in Lem. E.3, we can conclude that:
    $$
        d^{\pi}(s_h,a_h) - d^{\pi^k}(s_h,a_h) \geq -\Pr(\bar{\cE}_{k,\pi}|\pi^k)
    $$
    We define $\pi^{k,*}$ to be a policy that equals $\pi^k$ on those states where $\pi^k$ is optimal, and takes the optimal action when $\pi^k$ is non-optimal, then we have:
    \begin{align*}
        V^{\pi^{k,*}}_1(s_1) - V^{\pi^k}_1(s_1) = &\EE_{\pi^k}[\sum_{h=1}^H \mathbb{I}[\tilde{\cE}_{k,h,\pi^{k,*}}](V^{\pi^{k,*}}_h(s_h)-V^{\pi^k}_h(s_h))] \\
        \geq& \EE_{\pi^k}[\sum_{h=1}^H \mathbb{I}[\tilde{\cE}_{k,h,\pi^{k,*}}](V^{\pi^{k,*}}_h(s_h)-Q^{\pi^{k,*}}_h(s_h, \pi^k(s_h)))]\\
        \geq& \EE_{\pi^k}[\sum_{h=1}^H \mathbb{I}[\tilde{\cE}_{k,h,\pi^{k,*}}]\Delta_{\min}]=\Delta_{\min} \Pr(\bar{\cE}_{k,\pi^{k,*}}|\pi^k)\\
        \geq& \Delta_{\min} (d^{\pi^k}(s_h,a_h) - d^{\pi^{k,*}}(s_h,a_h)).
    \end{align*}
    Sum over all $k\in[K]$, we have
    \begin{align*}
        \sum_{k=1}^K d^{\pi^k}(s_h,a_h) - \sum_{k=1}^K d^{\pi^{k,*}}(s_h,a_h) \leq \frac{1}{\Delta_{\min}} \Big(\sum_{k=1}^K V_{1}^*(s_1)-V_{1}^{\pi^k}(s_1)\Big).
    \end{align*}
    which completes the proof.
\end{proof}
\begin{corollary}[Unique Optimal Policy]\label{corl:unique_optimal_policy}
    Under Assump.~\ref{assump:unique_optimal_policy}, Thm. \ref{thm:algO_regret_vs_algP_density} implies that:
    \begin{align*}
        |\sum_{k=1}^K d^{\pi^k_\tO}_\tO(s_h,a_h) - K d^{\pi^*_\tO}(s_h,a_h)| \leq \frac{1}{\Delta_{\min}} (\sum_{k=1}^K V^*_{\tO,1}(s_1)-V^{\pi^k_\tO}_{\tO,1}(s_1)\Big)
    \end{align*}
\end{corollary}

\begin{lemma}\label{lem:concentration}
    Let $\mathcal{F}_i$ for $i,1...$ be a filtration and $X_1,...X_n$ be a sequence of Bernoulli random variables with $\Pr(X_i=1|\mathcal{F}_{i-1})=P_i$ with $P_i$ being $\mathcal{F}_{i-1}$-measurable and $X_i$ being $\mathcal{F}_i$ measurable. It holds that
    \begin{align*}
        \Pr(\exists n:~\sum_{t=1}^n X_t < \frac{1}{2}\sum_{t=1}^n P_t - W) \leq e^{-W};~\Pr(\exists n:~\sum_{t=1}^n X_t > e\sum_{t=1}^n P_t - W) \leq e^{-W}.
    \end{align*}
\end{lemma}
\begin{proof}
    The first inequality has been proven in Lemma F.4 of \citep{dann2017unifying}. Here we adopt similar techniques to prove the second one.

    We first define $m_t := e^{X_t - eP_t}$, since $X_t$ is a Bernoulli random variable with $\Pr(X_t = 1) = P_t$, we should have:
    \begin{align*}
        \EE_{X_t}[e^{X_t - eP_t}|\cF_{t-1}] = \frac{eP_t + (1-P_t)}{e^{eP_t}} \leq \frac{eP_t + (1-P_t)}{eP_t + 1}\leq 1 
    \end{align*}
    where in the last but two step, we use $e^x \geq x + 1$. Therefore, $M_n := \prod_{t=1}^n m_t = e^{\sum_{t=1}^n X_t - eP_t}$ is a supermartingale. By Markov inequality, we have:
    \begin{align*}
        \Pr(\sum_{t=1}^n X_t - eP_t \geq W) = \Pr(M_n \geq e^W) \leq \frac{\EE[M_n]}{e^W} \leq e^{-W}.
    \end{align*}
    As a result, for a fixed $n$, we have $\Pr(\sum_{t=1}^n X_t \geq eP_t + W) \leq e^{-W}$.
    After a similar discussion about stopping time as \citep{dann2017unifying}, we have $\Pr(\exists n:~\sum_{t=1}^n X_t \geq eP_t + W) \leq e^{-W}$.
\end{proof}
As a direct result of Lem. \ref{lem:concentration}, we have the following result:
\begin{lemma}\label{lem:occupancy_of_Nh}
    For arbitrary $k \geq 1$, and arbitrary $\alpha > 2$, $\Pr(\cECk) \geq 1 - \frac{1}{k^\alpha}$.
\end{lemma}

\subsection{Analysis of $\algO$}

\begin{lemma}[The relationship between $d^*_{\tO}$ and $\NO{}^k$]\label{lem:lower_bound_of_dOstar}
    There exists a constant $c_{occup}$ which is independent of $\lambda,S,A,H$ and gap $\Delta$, s.t., for all $k \geq k_{occup} := c_{occup}\frac{C_1+\alpha C_2}{\lambda \Delta_{\min}}\log(\frac{\alpha C_1C_2 SAH}{\lambda \Delta_{\min}})$, on the events of $\cEOk$ and $\cECk$, $\NO{,h}^k(s_h,a_h) \geq \frac{\lambda}{3} k$ implies that $d^*_{\tO}(s_h,a_h) \geq \frac{\lambda}{9} > 0$, and conversely, if $d^*_{\tO}(s_h,a_h) \geq \lambda$, we must have $\NO{,h}^k(s_h) \geq \NO{,h}^k(s_h,\pi^*_{\tO}) \geq \frac{\lambda}{3} k$.
\end{lemma}
\begin{proof}
    On the event of $\cEOk$ and $\cECk$, as a result of Cor. \ref{corl:unique_optimal_policy}, $\NO{,h}^k(s_h,a_h) \geq \frac{\lambda}{3} k$ implies:
    \begin{align*}
        \frac{\lambda}{3} k\leq \NO{,h}^k(s_h,a_h) \leq ekd^*_{\tO}(s_h,a_h) + \alpha \log(2SAHk) + \frac{1}{\Delta_{\min}}(C_1 + \alpha C_2 \log  k)
    \end{align*}
    There should exists a constant $c_{occup}$, such that, $\frac{3-e}{9}\lambda k \geq \alpha \log(2SAHk)+\frac{1}{\Delta_{\min}}(C_1 + \alpha C_2 \log  k)$ can be satisfied for all $k \geq c_{occup}\frac{C_1+\alpha C_2}{\lambda \Delta_{\min}}\log(\frac{\alpha C_1C_2 SAH}{\lambda \Delta_{\min}})$, which implies that:
    \begin{align*}
        d^*_{\tO}(s_h,a_h) \geq \frac{\frac{\lambda}{3} k - \frac{3-e}{9}\lambda k}{ek} \geq \frac{\lambda}{9}.
    \end{align*}
    On the other hand, if $d^*_{\tO}(s_h,a_h) \geq \lambda$, on the event $\cECk$ and Cor. \ref{corl:unique_optimal_policy}, we have:
    \begin{align*}
        \NO{,h}^k(s_h,a_h) \geq& \frac{1}{2}kd^*_{\tO}(s_h,a_h) - \alpha \log(2SAHk) -  \frac{1}{\Delta_{\min}}(C_1 + \alpha C_2 \log  k)\\
        \geq& \frac{\lambda}{2}k - \alpha \log(2SAHk) -  \frac{1}{\Delta_{\min}}(C_1 + \alpha C_2 \log  k).
    \end{align*}
    with the same $c_{occup}$, we have:
    \begin{align*}
        \NO{,h}^k(s_h) \geq  \NO{,h}^k(s_h,a_h) \geq \frac{\lambda}{2}k - \frac{3-e}{9}\lambda k = \frac{3 + 2e}{18}\lambda k\geq \frac{\lambda}{3} k.
    \end{align*}
    which finishes the proof.
\end{proof}

\begin{lemma}[Convergence Speed of PVI]\label{lem:convergence_speed_of_PVI}
    There exists an absolute constant $c_\Xi$, such that for arbitrary fixed $\threshold > 0$ and $\lambda > 0$, and for arbitrary 
    \begin{equation}
        k \geq c_\Xi \max\{\frac{\alpha B_1^2H^2S}{\lambda^2\threshold^2}\log(\frac{\alpha HSAB_1B_2}{\lambda \threshold}), \frac{(C_1+\alpha C_2)SH}{\Delta_{\min}\lambda\threshold}\log\frac{C_1C_2SAH}{\Delta_{\min}\lambda\threshold}\}.\label{eq:converge_speed_PVI}
    \end{equation}
    on the event $\cECk,\cEBk$ and $\cEOk$, for arbitrary $h \in [H], s_h \in \cS_h$ with $N^k_{\tO,h}(s_h) > \frac{\lambda}{3}$, we have
    \begin{align*}
        V^*_{\tO,h}(s_h) - \uVO{,h}^k(s_h) \leq \threshold.
    \end{align*}
\end{lemma}
\begin{proof}
    As a result of Lem. \ref{lem:lower_bound_of_dOstar}, considering $c_\Xi \geq c_{occup}$, on the events of $\cEOk$ and $\cECk$, $N^k_{\tO,h}(s_h) > \frac{\lambda}{3}$ implies $d^*_{\tO}(s_h) \geq \frac{\lambda}{9}$. According to the Lem. \ref{lem:underestimation}, for arbitrary $s_h$ with $d^{\pi^*}_{\tO}(s_h) > 0$
    \begin{align*}
        V^*_{\tO,h}(s_h) - \uVO{,h}^k(s_h) \leq& 2\EE_{\pi_{\tO}^*, \MO}[\sum_{\ph=h}^H\min\{\bonus_{\tO,\ph}^k(s_\ph,\pi^*_{\tO}), H\}|s_h] \\
        = & 2\sum_{\ph=h}^H\sum_{s_\ph,\pi^*_{\tO}}d^{\pi^*_{\tO}}_{\tO}(s_\ph,\pi^*_{\tO}|s_h)\min\{\bonus_{\tO,\ph}^k(s_\ph,\pi^*_{\tO}), H\}\\
        \leq & \frac{2}{d^{\pi^*_{\tO}}_{\tO}(s_h)}\sum_{\ph=h}^H\sum_{s_\ph,\pi^*_{\tO}}d^{\pi^*_{\tO}}_{\tO}(s_\ph,\pi^*_{\tO})\min\{\bonus_{\tO,\ph}^k(s_\ph,\pi^*_{\tO}), H\}\\
        \leq & \frac{18}{\lambda}\EE_{\pi_{\tO}^*, \MO}[\sum_{\ph=h}^H\min\{\bonus_{\tO,\ph}^k(s_\ph,\pi^*_{\tO}), H\}|s_1]
    \end{align*}
    Given threshold $\threshold$, we define the following set:
    $
    \cY_{\geq h}^\threshold := \bigcup_{h'\geq h}\{s_h | d^*_{\tO}(s_h) > \frac{\lambda\threshold}{36SH}\}.
    $
Note that for those $s_h \in \cS_h \setminus \cY_{\geq h}^\threshold$, we have:
\begin{align*}
    \sum_{\ph=h}^H \sum_{s_\ph \in \cS_\ph \setminus \cY_{\geq h}^\threshold}\frac{2}{\lambda}\EE_{\pi_{\tO}^*, \MO}[\min\{\bonus_{\tO,\ph}^k(s_\ph,\pi^*_{\tO}), H\}|s_1] \leq SH\cdot\frac{18}{\lambda} \cdot \frac{\lambda\threshold}{36SH} = \frac{\threshold}{2}.
\end{align*}
In the following, we study the bonus term for $s_h \in \cY_{\geq h}^\xi$. According to Lem. \ref{lem:occupancy_of_Nh}, on the event of $\cEOk$ and $\cECk$, we have:
\begin{align*}
    \NO{,h}^k(s_h,a_h) \geq \frac{k}{2}d^*_{\tO}(s_h, a_h) - \frac{1}{\Delta_{\min}}(C_1 + \alpha C_2 \log k) - \alpha \log (2SAHk).
\end{align*}
We define: 
$$
k_{s_h} := \arg\min_k  \quad s.t.\quad kd^*_{\tO}(s_h, a_h) / 4 \geq \frac{1}{\Delta_{\min}}(C_1 + \alpha C_2 \log  k) + \alpha \log (2SAHk),\quad \forall k' \geq k
$$
which implies that $k_{s_h} = c_0\cdot \frac{C_1+\alpha C_2}{\Delta_{\min}d^*_{\tO}(s_h, a_h)}\log\frac{C_1C_2SAH}{\Delta_{\min}d^*_{\tO}(s_h, a_h)} \leq c_0'\cdot \frac{(C_1+\alpha C_2)SH}{\Delta_{\min}\lambda\threshold}\log\frac{C_1C_2SAH}{\Delta_{\min}\lambda\threshold}$ for some absolute constants $c_0$ and $c_0'$, where the second step we use $d^*_{\tO}(s_h,a_h) > \lambda\threshold/36SH$. Therefore, for $k \geq k_{s_h}$, we have:
\begin{align*}
    \bonus^k_{\tO,h}(s_h,a_h) \leq B_1\sqrt{\frac{\log(B_2/\delta_k)}{\NO{,h}^k(s_h,a_h)}} \leq 2B_1\sqrt{\frac{\log(B_2/\delta_k)}{kd^*_{\tO}(s_h, a_h)}}.
\end{align*}
which implies that:
\begin{align*}
    \EE_{\pi_{\tO}^*, \MO}[\min\{\bonus_{\tO,h}^k(s_h,\pi^*_{\tO}), H\}|s_1]=&  2d^*_{\tO}(s_h, \pi_{\tO}^*) \cdot B_1\sqrt{\frac{\log(B_2/\delta_k)}{kd^*_{\tO}(s_h, \pi_{\tO}^*)}} = 2 B_1\sqrt{\frac{d^*_{\tO}(s_h, \pi_{\tO}^*)\log(B_2/\delta_k)}{k}}.
\end{align*}
Therefore,
\begin{align*}
    &\sum_{\ph=h}^H \sum_{s_\ph\in\cY_{\geq h}^\threshold}\frac{18}{\lambda}\EE_{\pi_{\tO}^*, \MO}[\min\{\bonus_{\tO,\ph}^k(s_\ph,\pi^*_{\tO}), H\}|s_1]\\
    \leq & \frac{36}{\lambda}B_1 \sum_{\ph=h}^H \sum_{s_\ph \in \cS_\ph \setminus \cY_{\geq h}^\xi} \sqrt{\frac{d^*_{\tO}(s_\ph, \pi^*_{\tO})\log(B_2/\delta_k)}{k}} \leq \frac{36}{\lambda}B_1\sqrt{\frac{\log(B_2/\delta_k)}{k}}\sum_{\ph=h}^H \sum_{s_\ph \in \cS_\ph \setminus \cY_{\geq h}^\xi} \sqrt{d^*_{\tO}(s_\ph, \pi^*_{\tO})} \\
    \leq &\frac{36}{\lambda}B_1\sqrt{\frac{\log(B_2/\delta_k)}{k}}\sqrt{SH\sum_{\ph=h}^H \sum_{s_\ph \in \cS_\ph \setminus \cY_{\geq h}^\xi}d^*_{\tO}(s_\ph, \pi^*_{\tO})}\leq \frac{36}{\lambda}B_1H\sqrt{\frac{S\log(B_2/\delta_k)}{k}}.
\end{align*}
Recall that $\delta_k = O(1/SAHk^\alpha)$, the RHS is less than $\frac{\threshold}{2}$ when:
\begin{align*}
    k \geq c_0''\frac{\alpha B_1^2H^2S}{\lambda^2\threshold^2}\log(\frac{\alpha HSA B_1B_2}{\lambda \threshold}).
\end{align*}
for some constant $c_0''$. Therefore, by choosing $c_\Xi = \max\{c_0',c_0''\}$ we can conclude that, as long as:
\begin{align*}
    k \geq c_\Xi \max\{\frac{\alpha B_1^2H^2S}{\lambda^2\threshold^2}\log(\frac{\alpha HSAB_1B_2}{\lambda \threshold}), \frac{(C_1+\alpha C_2)SH}{\Delta_{\min}\lambda\threshold}\log\frac{C_1C_2SAH}{\Delta_{\min}\lambda\threshold}\}.
\end{align*}
we have $V^*_{\tO,h}(s_h) - \uVO{,h}^k(s_h) \leq \xi$.
\end{proof}

\subsection{Analysis of Regret on $\ME$}\label{appx:analysis_of_algE_RL}
For the simplification of the notation, in the following, we will denote: 
$$
\zeta^k(s_h) := \mathbb{I}[\{\uQO{,h}^k(s_h,\upiO{,h}^k) \leq \tQ_{\tE,h}^{k}(s_h,\upiO{,h}^k) + \epsilon\}\cap\{\NO{,h}^k(s_h) > \frac{\lambda}{3} k\}].
$$
In another word, $\zeta^k(s_h) = 1$ if and only if we will trust $\MO$ at state $s_h$ in Alg. \ref{alg:RL_Setting} and therefore set $\piE{}^k$ by exploiting information from $\MO$. 
Next, we define the surplus.

\begin{definition}[Definition of Surplus in Pessimistic Algorithm setting]\label{def:deficit}
    We define the surplus for pessimistic estimation in $\MO$ and optimistic estimation in $\ME$:
    \begin{align*}
        \surplusE{,h}^k(s_h,a_h) =& \tQE{,h}^k(s_h,a_h) - \mP_{\tE, h}\tVE{,h+1}^k(s_h,a_h) + \mP_{\tE, h}\uVE{,h+1}^{\pi_{\tE}^k}(s_h,a_h) - \uQE{,h}^{\pi_{\tE}^k}(s_h,a_h).
    \end{align*}
    We also define:
    \begin{align*}
        \dsurplusE{,h}^k(s_h,a_h) := \Clip{\surplusE{,h}^k(s_h,a_h)}{\Big|\max\{\frac{\Delta_{\min}}{4eH}, \frac{\DeltaE(s_h,a_h)}{4e}\}}.
    \end{align*}
    and
    \begin{align*}
        \ddot{Q}_{\tE,h}^\pi(s_h,a_h) := \rE{,h}(s_h,a_h) + \mPE{,h}\ddot{V}_{\tE,h+1}^\pi(s_h,a_h) + e\dsurplusE{,h}^k(s_h,a_h),\quad \ddot{V}_{\tE,h}^\pi(s_h) := \ddot{Q}_{\tE,h}^\pi(s_h,\pi).
    \end{align*}
\end{definition}
We first show that $\tilde{V}^k_{\tE,h}$ will be an overestimation eventually.

\ThmOverEst*
\begin{proof}
    In this theorem, $\cE_k$ denotes the event $\cEOk \cap \cEBk \cap \cECk$.
    \paragraph{Part 1: Proof of Overestimation}
    We do the proof by induction. First of all, note that the overestimation is true for horizon $h=H+1$, since all the value is zero. Next, we assume the overestimation is true for step $h+1$, we will show it holds for step $h$. 
    We first show $\tQE{,h}^k$ is an overestimation:
    \begin{align*}
        &\tQ_{\tE,h}^{k}(s_h,a_h) - Q^*_{\tE,h}(s_h,a_h) \\
        =& \min\{H-Q^*_{\tE,h}(s_h,a_h), \hmPE{,h}^k\tV_{\tE,h+1}^k(s_h,a_h) + \bonus_{\tE,h}^k(s_h,a_h)- \mPE{,h} V^*_{\tE,h+1}(s_h,a_h)\}\\
        =& \min\{H-Q^*_{\tE,h}(s_h,a_h), (\hmPE{,h}^k-\mPE{,h})\tV_{\tE,h+1}^k(s_h,a_h) + \bonus_{\tE,h}^k(s_h,a_h) + \mPE{,h} (\tV_{\tE,h+1}^k-V^*_{\tE,h+1})(s_h,a_h)\}\\
        \geq & \min\{H-Q^*_{\tE,h}(s_h,a_h), \mPE{,h} (\tV_{\tE,h+1}^k-V^*_{\tE,h+1})(s_h,a_h)\} \geq 0
    \end{align*}
    where the first inequality is because of event $\cEBk$. In the following, we separate to three cases:
    \paragraph{Case 1: $\zeta^k(s_h,a_h) = 0$}
    In this case, 
    \begin{align*}
        \tVO{,h}^k(s_h) = \max_a \tQ_{\tE,h}^{k}(s_h,a) \geq \tQ_{\tE,h}^{k}(s_h,\pi_{\tE}^*) \geq Q^*_{\tE,h}(s_h,\pi_{\tE}^*) = V^*_{\tE,h}(s_h).
    \end{align*}
    \paragraph{Case 2: $\zeta^k(s_h,a_h) = 1$ and $\upiO{,h}^k(s_h) = \pi_{\tE,h}^*(s_h)$}
    In this case, as a result of Lem. \ref{lem:underestimation}, we have:
    \begin{align*}
        \tVO{,h}^k(s_h) = \tQ_{\tE,h}^{k}(s_h,\pi_{\tE}^*) + \frac{1}{H}(\tQ_{\tE,h}^{k}(s_h,\pi_{\tE}^*) - \uQE{,h}^{\pi^k_{\tE}}(s_h,\pi_{\tE}^*)) \geq \tQ_{\tE,h}^{k}(s_h,\pi_{\tE}^*) \geq Q^*_{\tE,h}(s_h,\pi_{\tE}^*) = V^*_{\tE,h}(s_h).
    \end{align*}
    \paragraph{Case 3: $\zeta^k(s_h,a_h) = 1$ and $\upiO{,h}^k(s_h) \neq \pi_{\tE,h}^*(s_h)$}   
    This case is more complicated. Intuitively, we want to show that after $k_{ost}$, in case 3, the ``uncerntainty'' must be high, and therefore, adding $1/H$ of the uncerntainty interval will ensure the overestimation.

    As a result of Lem. \ref{lem:convergence_speed_of_PVI}, we choose $k_{ost}$ by plugging $\threshold = \frac{\Delta_{\min}}{4(H+1)}$ into Eq.~\eqref{eq:converge_speed_PVI}, which yields:
    \begin{align}
        k_{ost} := c_{ost}\cdot \max\{\alpha\frac{B_1^2H^2S}{\lambda^2\Delta_{\min}^2}\log(\alpha HSAB_1B_2), \frac{(C_1+\alpha C_2)SH}{\lambda\Delta_{\min}^2}\log\frac{C_1C_2SAH}{\Delta_{\min}\lambda}\},
    \end{align}
    for some constant $c_{ost}$. Then, for arbitrary $k \geq k_{ost}$, on the event of $\cEBk, \cECk, \cEOk$, case 3 implies that:
    \begin{align*}
        V^*_{\tO,h}(s_h) - \uVO{,h}^k(s_h) \leq \frac{\Delta_{\min}}{4(H+1)}.
    \end{align*}
    Combining with Lem. \ref{lem:underestimation}, it directly implies that $\upiO{,h}^k(s_h) = \pi_{\tO,h}^*(s_h)$. Therefore, in the following, we directly use $\pi_{\tO,h}^*$ to refer $\upiO{,h}^k(s_h)$. Then first observation is that, under Cond.~\ref{assump:opt_value_dominance}, in this case, we have:
    \begin{align*}
        \VE{,h}^*(s_h) - \tQ_{\tE,h}^k(s_h,\pi_{\tO}^*) \leq & \VE{,h}^*(s_h) - \uQO{,h}^k(s_h,\piO{,h}^*) + \epsilon \\
        = & \VE{,h}^*(s_h) - \uVO{,h}^k(s_h) + \epsilon \\
        \leq & \VO{,h}^*(s_h) + \frac{\Delta_{\min}}{2(H+1)} - \uVO{,h}^k(s_h) + \epsilon \\
        = & V^*_{\tO,h}(s_h) - \uVO{,h}^k(s_h) + \frac{\Delta_{\min}}{2(H+1)} + \epsilon \\
        \leq & \frac{\Delta_{\min}}{H+1} \leq \frac{1}{H+1} (\VE{,h}^*(s_h) - \QE{,h}^*(s_h,\pi_{\tO}^*))
    \end{align*}
    which implies that:
    \begin{align*}
        \VE{,h}^*(s_h) \leq&  \tQ_{\tE,h}^k(s_h,\pi_{\tO}^*) + \frac{1}{H} (\tQ_{\tE,h}^k(s_h,\pi_{\tO}^*) - \QE{,h}^*(s_h,\pi_{\tO}^*)\\
        \leq & \tQ_{\tE,h}^k(s_h,\pi_{\tO}^*) + \frac{1}{H} (\tQ_{\tE,h}^k(s_h,\pi_{\tO}^*) - \uQE{,h}^{\pi_{\tE}^k}(s_h,\pi_{\tO}^*) \tag{Lem. \ref{lem:underestimation}}\\
        = & \tVO{,h}^k(s_h)
    \end{align*}
    which finishes the proof for overestimation.
    \paragraph{Part 2: Proof for Eq.~\eqref{eq:regret_bound}}

    The following proof relies on Lem. \ref{lem:bounds_of_surplus} and Lem. \ref{lem:surplus_vs_overestimation_gap}, whose proofs we just provide after finishing the proof for this theorem.
    The first observation is, for arbitrary policy $\pi_{\tE}$,
    \begin{align*}
        \VE{,h}^*(s_h) - \VE{,h}^{\pi_{\tE}}(s_h) =& \QE{,h}^*(s_h,\pi_{\tE}^*) - \QE{,h}^*(s_h,\pi_{\tE})\\
        =& \DeltaE(s_h,\pi_{\tE}) + \QE{,h}^*(s_h,\pi_{\tE}) - \QE{,h}^*(s_h,\pi_{\tE}) \\
        = & \DeltaE(s_h,\pi_{\tE}) + \mPE{,h}(\VE{,h+1}^* - \VE{,h+1}^{\pi_{\tE}})(s_h,\pi_{\tE})\\
        = & ...\\
        = & \EE_{\pi_{\tE},\ME}[\sum_{\ph=h}^H \DeltaE(s_\ph,a_\ph)|s_h].\numberthis\label{eq:gap_and_regret}
    \end{align*}
    Besides, according to the definition of $\dsurplusE{,h}$, we have:
    \begin{align*}
        \ddot{V}_{\tE,h}^{\pi_{\tE}^k}(s_h) - \VE{,h}^{\pi_{\tE}^k}(s_h) =& e\EE_{\pi_{\tE}^k}[\sum_{\ph=h}^H \dsurplusE{\ph}^k(s_\ph,a_\ph)|s_h] \\
        \geq & e\EE_{\pi_{\tE}^k}[\sum_{\ph=h}^H \surplusE{\ph}^k(s_\ph,a_\ph) - \epsClip - \frac{\DeltaE(s_\ph,a_\ph)}{4e}|s_h] \\
        \geq & e\EE_{\pi_{\tE}^k}[\sum_{\ph=h}^H \surplusE{\ph}^k(s_\ph,a_\ph) - \frac{\DeltaE(s_\ph,a_\ph)}{4e}|s_h]- eH\cdot \frac{\Delta_{\min}}{4eH}.\\
        \geq & \tVE{,h}^k(s_h)-\VE{,h}^{\pi_{\tE}^k}(s_h) - \frac{1}{4}(\VE{,h}^*(s_h) - \VE{,h}^{\pi_{\tE}^k}(s_h)) - \frac{\Delta_{\min}}{4}.\tag{Eq.~\eqref{eq:gap_and_regret} and Lem. \ref{lem:surplus_vs_overestimation_gap}}\\
        \geq & \frac{3}{4}(\VE{,h}^*(s_h) - \VE{,h}^{\pi_{\tE}^k}(s_h)) - \frac{\Delta_{\min}}{4}. \tag{Overestimation}
    \end{align*}
    If $\pi_{\tE}^k(s_h) \neq \pi_{\tE}^*(s_h)$, since $\VE{,h}^*(s_h) - \VE{,h}^{\pi_{\tE}^k}(s_h) \geq \DeltaE(s_h,\pi_{\tE}^k)$, we further have:
    \begin{align*}
        \ddot{V}_{\tE,h}^{\pi_{\tE}^k}(s_h) - \VE{,h}^{\pi_{\tE}^k}(s_h) \geq \frac{1}{2}(\VE{,h}^*(s_h) - \VE{,h}^{\pi_{\tE}^k}(s_h)) + \frac{\DeltaE(s_h,\pi_{\tE}^k)}{4} - \frac{\Delta_{\min}}{4} \geq \frac{1}{2}(\VE{,h}^*(s_h) - \VE{,h}^{\pi_{\tE}^k}(s_h)).
    \end{align*}
    otherwise,
    \begin{align*}
        \ddot{V}_{\tE,h}^{\pi_{\tE}^k}(s_h) - \VE{,h}^{\pi_{\tE}^k}(s_h) =& \dsurplusE{,h}(s_h,\pi_{\tE}^k) + \EE_{\pi_{\tE}^k,\ME}[\ddot{V}_{\tE,h+1}^{\pi_{\tE}^k}(s_{h+1}) - \VE{,h+1}^{\pi_{\tE}^k}(s_{h+1})|s_h]\\
        \geq & \EE_{\pi_{\tE}^k,\ME}[\ddot{V}_{\tE,h+1}^{\pi_{\tE}^k}(s_{h+1}) - \VE{,h+1}^{\pi_{\tE}^k}(s_{h+1})|s_h].
    \end{align*}
    Therefore, we have:
    \begin{align*}
        \ddot{V}_{\tE,1}^{\pi_{\tE}^k}(s_1) - \VE{,1}^{\pi_{\tE}^k}(s_1) \geq & \EE_{\pi_{\tE}^k, \ME}[\sum_{h=1}^H \mathbb{I}[\tilde\cE_{\tE,h}^k](\ddot{V}_{\tE,h}^{\pi_{\tE}^k}(s_h) - \VE{,h}^{\pi_{\tE}^k}(s_h))]\\
        \geq &\frac{1}{2}\EE_{\pi_{\tE}^k, \ME}[\sum_{h=1}^H \mathbb{I}[\tilde\cE_{\tE,h}^k](\VE{,h}^{*}(s_h) - \VE{,h}^{\pi_{\tE}^k}(s_h))]\\
        = & \frac{1}{2}(\VE{,h}^{*}(s_1) - \VE{,h}^{\pi_{\tE}^k}(s_1)).
    \end{align*}
    Combining Lem. \ref{lem:bounds_of_surplus} and the definition of $\ddot{V}_{\tE,1}^{\pi_{\tE}^k}(s_1) $, we can finish the proof for Eq.~\eqref{eq:regret_bound}.
\end{proof}

\begin{lemma}[Upper and lower bounds of the surplus]\label{lem:bounds_of_surplus}
    For arbitrary $k$, on the event of $\cEBk$, we have:
    \begin{align*}
        \forall h \in [H],~\forall s_h\in\cS_h,a_h\in\cA_h,\quad 0\leq \surplusE{,h}^k(s_h,a_h) \leq \min\{H, 4\bonus_{\tE,h}^k(s_h,a_h)\}
    \end{align*}
\end{lemma}
\begin{proof}
    According to Alg. \ref{alg:RL_Setting}, we should have $\tQE{,h}^k,~\tVE{,h+1}^k,~\uVE{,h+1}^{\pi_{\tE}^k},~\uQE{,h}^{\pi_{\tE}^k} \in[0,H]$. 
    By Lem. \ref{lem:underestimation} and Thm. \ref{thm:overestimation}, we also have $\tVE{,h+1}^k(\cdot) \geq \uVE{,h+1}^{\pi_{\tE}^k}(\cdot)$, which implies that $\surplusE{,h}^k(\cdot,\cdot) \leq H$.
    Besides,
    \begin{align*}
        \surplusE{,h}^k(s_h,a_h) = (\hmPE{,h}^k - \mPE{,h})\tVE{,h+1}^k(s_h,a_h) + (\mPE{,h} - \hmPE{,h}^k)\uVE{,h+1}^{\pi_{\tE}^k}(s_h,a_h) + 2\bonus_{\tE,h}^k(s_h,a_h).
    \end{align*}
    On the event of $\cEBk$, we have $0 \leq \surplusE{,h}^k(s_h,a_h) \leq 4\bonus_{\tE,h}^k(s_h,a_h)$, which finishes the proof.
\end{proof}
\begin{lemma}[Relationship between surplus and overestimation gap]\label{lem:surplus_vs_overestimation_gap}
    Under the same condition of Thm. \ref{thm:overestimation},
    \begin{align*}
        \tVE{,h}^k(s_h)-\VE{,h}^{\pi_{\tE}^k}(s_h)\leq e\EE_{\pi_{\tE}^k, \ME}[\sum_{\ph=h}^H \surplusE{,h}^k(s_h,\pi_{\tE}^k)|s_h].
    \end{align*}
\end{lemma}
\begin{proof}
\begin{align*}
    &\tVE{,h}^k(s_h)-\VE{,h}^{\pi_{\tE}^k}(s_h) \leq \tVE{,h}^k(s_h) - \uVE{,h}^{\pi_{\tE,k}}(s_h) \tag{Lem. \ref{lem:underestimation}}\\
    \leq& (1+\frac{1}{H})(\tQE{,h}^k(s_h,\pi_{\tE}^k) - \uQE{,h}^{\pi_{\tE,k}}(s_h,\pi_{\tE}^k)) \tag{Update rule in Alg. \ref{alg:RL_Setting} and $\tQE{,h}^k \geq \uQE{,h}^{\pi_{\tE,k}}$}\\
    = & (1+\frac{1}{H})(\tQE{,h}^k(s_h,\pi_{\tE}^k) - \mP_{\tE, h}\tVE{,h+1}^k(s_h,\pi_{\tE}^k) + \mP_{\tE, h}\uVE{,h+1}^k(s_h,\pi_{\tE}^k) - \uQE{,h}^{\pi_{\tE,k}}(s_h,\pi_{\tE}^k))\\
    & + (1+\frac{1}{H}) \mP_{\tE, h}(\tVE{,h+1}^k-\uVE{,h+1}^k)(s_h,\pi_{\tE}^k)\\
    = &  (1+\frac{1}{H})\surplusE{,h}^k(s_h,\pi_{\tE}^k) +  (1+\frac{1}{H}) \mP_{\tE, h}(\tVE{,h+1}^k-\uVE{,h+1}^k)(s_h,\pi_{\tE}^k)\\
    \leq & ... \\
    \leq & \EE_{\pi_{\tE}^k, \ME}[\sum_{\ph=h}^H (1+\frac{1}{H})^{\ph-h+1}\surplusE{,h}^k(s_h,\pi_{\tE}^k) ]\\
    \leq & e\EE_{\pi_{\tE}^k, \ME}[\sum_{\ph=h}^H \surplusE{,h}^k(s_h,\pi_{\tE}^k)].
\end{align*}
\end{proof}
In the following lemma, we show the benefits of transfer action between similar states.
\begin{lemma}[Benefits on Similar States]\label{lem:effects_on_similar_states}
    If $s_h$ in $\ME$ is $\epsilon$-close to $s_h$ in $\MO$, i.e. satisfying the property in Def. \ref{def:close_state}, and $d^*_{\tO}(s_h) \geq \lambda$, where $\lambda$ is the hyper-parameter in Alg. \ref{alg:RL_Setting}, then, for arbitrary $k \geq k_{occup} := c_{occup}\frac{C_1+\alpha C_2}{\lambda \Delta_{\min}}\log(\frac{\alpha C_1C_2 SAH}{\lambda \Delta_{\min}})$, on the events of $\cEOk,\cECk, \cEBk$, we have $\pi_{\tE}^k(s_h) = \pi^*_{\tE}(s_h)$. 
\end{lemma}
\begin{proof}
    As a result of Lem. \ref{lem:lower_bound_of_dOstar}, we have, for arbitrary $s_h$ with $d^*_{\tO}(s_h)=d^*_{\tO}(s_h,\pi^*_{\tO}) \geq \lambda$, after $k \geq c_{occup}\frac{C_1+\alpha C_2}{\lambda \Delta_{\min}}\log(\frac{\alpha C_1C_2 SAH}{\lambda \Delta_{\min}})$, we should have:
    \begin{align*}
        \NO{,h}^k(s_h) \geq \NO{,h}^k(s_h,\pi^*_{\tO}) \geq \frac{\lambda}{3} k.
    \end{align*}
    On the events of $\cEBk$, and the value dominance condition Def. \ref{assump:opt_value_dominance}, we also have:
    \begin{align*}
        \uQO{,h}^k(s_h,\upiO{,h}^k) \leq \tQE{,h}^k(s_h,\upiO{,h}^k) + \epsilon.
    \end{align*}
    which implies that the algorithm will choose the ``trust'' branch and choose $\pi_{\tE,h}^k(s_h) = \upiO{,h}^k(s_h)$. 
    
    On the other hand, if there is another $a_h \neq \pi^*_{\tO}(s_h)$ satisfying $\NO{,h}^k(s_h,a_h) \geq \frac{\lambda}{3} k$, by applying Lem. \ref{lem:lower_bound_of_dOstar} again, we can make a contradication and therefore, we must have 
    $$
    \NO{,h}^k(s_h,\pi^*_{\tO}) \geq \frac{\lambda}{3} k > \NO{,h}^k(s_h,a_h),\quad \forall a_h \neq \pi^*_{\tO}(s_h).
    $$
    which implies that $\pi_{\tE}^k(s_h) = \pi^*_{\tO}(s_h)$ in Alg. \ref{alg:RL_Setting}. Given that $s_h$ is $\epsilon$-close between $\ME$ and $\MO$, we directly have $\pi_{\tE}^k(s_h) = \pi_{\tE}^*(s_h)$.
\end{proof}

\begin{theorem}[Detailed Version of Thm.\ref{thm:regret_upper_bound}]\label{thm:regret_upper_bound_detailed}
    Under Assump. \ref{assump:unique_optimal_policy}, \ref{assump:opt_value_dominance} and \ref{assump:lower_bound_Delta_min}, Cond.~\ref{cond:requirement_on_algO} for $\algO$ and Cond.~\ref{cond:bonus_term} for \textbf{Bonus} function,
    by running Alg.~\ref{alg:RL_Setting} with $\epsilon=\frac{\tilde{\Delta}_{\min}}{4(H+1)}$, $\alpha > 2$, an any $\lambda > 0$, we have
    \begin{align*}
        \Regret_{K}(\ME) = & O\Big(H\cdot \max\{\alpha\frac{S^3H^4}{\lambda^2\Delta_{\min}^2}\log(\alpha SAH), \frac{(C_1+\alpha C_2)SH}{\lambda\Delta_{\min}^2}\log\frac{C_1C_2SAH}{\Delta_{\min}\lambda}\} \\
        &\quad \quad + \sum_{h=1}^H \sum_{(s_h,a_h) \in \cC^*_h} \frac{SH^2}{\Delta_{\min}} \log (SAH (K \wedge \frac{1}{\lambda \Delta_{\min} d^*_{\tE}(s_h)})) \\
        & \quad \quad + SH \sum_{h=1}^H \sum_{(s_h,a_h)\in \cS_h\times\cA_h \setminus \cC^{\lambda}_h} (\frac{H}{\Delta_{\min}}\wedge \frac{1}{\DeltaE(s_h,a_h)})\log (SAHK) \Big)\\
        =& O\Big(
        SH \sum_{h=1}^H \sum_{(s_h,a_h)\in \cS_h\times\cA_h \setminus \cC^{\lambda}_h} (\frac{H}{\Delta_{\min}}\wedge \frac{1}{\DeltaE(s_h,a_h)})\log (SAHK)\Big)\numberthis\label{eq:final_bound_RL}.
    \end{align*}
    \end{theorem}
\begin{proof}
    We consider $k_{start} : = \max\{k_{ost}, k_{occup}\}$.
    We first study the regret part after $k \geq k_{start}$:
    \begin{align*}
        &\EE[\sum_{k=k_{start}+1}^K \VE{,1}^*(s_1) - \VE{,1}^{\piE{}^k}(s_1)] \\
        \leq & \sum_{k=k_{start}+1}^K 2e\EE_{\pi_{\tE}^k}[\sum_{h=1}^H \Clip{\min\{H, 4\bonus^k_{\tE,h}(s_h,a_h)\}}{\Big|\frac{\Delta_{\min}}{4eH}\vee \frac{\DeltaE(s_h,a_h)}{4e}}\mathbb{I}[\cEBk\cap\cEOk\cap\cECk]] \\
        & + \sum_{k=k_{start}+1}^K H\cdot \Pr(\cEBk^\complement\cup\cEOk^\complement\cup\cECk^\complement).
    \end{align*}
    Since the failure rate for events $\cEBk, \cEOk, \cECk$ is only at the level of $k^{-\Theta(\alpha)}$, the second part is constant, and we mainly focus on the first term. For all state action $(s_h,a_h)$, and for all $k \geq k_{start}$, we have:
    \begin{align*}
        &\EE_{\pi_{\tE}^k}[\Clip{\min\{H, 4\bonus^k_{\tE,h}(s_h,a_h)\}}{\Big|\frac{\Delta_{\min}}{4eH}\vee \frac{\DeltaE(s_h,a_h)}{4e}}\mathbb{I}[\cEBk\cap\cEOk\cap\cECk]]\\
        =&d^{\pi_{\tE}^k}_{\tE}(s_h,a_h)\Clip{\min\{H, 4B_1\sqrt{\frac{\log(B_2 k^\alpha)}{\NE{,h}^k(s_h,a_h)}}\}}{\Big|\frac{\Delta_{\min}}{4eH}\vee \frac{\DeltaE(s_h,a_h)}{4e}}\\
        \leq&d^{\pi_{\tE}^k}_{\tE}(s_h,a_h)\Clip{\min\{H, 4B_1\sqrt{\frac{\alpha\log(B_2K)}{\NE{,h}^k(s_h,a_h)}}}{\Big|\frac{\Delta_{\min}}{4eH}\vee \frac{\DeltaE(s_h,a_h)}{4e}} \numberthis\label{eq:relaxed_clipping_bound}
    \end{align*}
    Under the event of $\cECk$, as a result of Lem. \ref{thm:algO_regret_vs_algP_density}, we have:
    \begin{align*}
        \NE{,h}^k(s_h,a_h) \geq \frac{1}{2} \sum_{k' = 1}^{k-1} d^{\piE{}^{k'}}_{\tE}(s_h,a_h) - \alpha \log(2SAHk) \geq \frac{1}{2} \sum_{k' = 1}^{k-1} d^{\piE{}^{k'}}_{\tE}(s_h,a_h) - \alpha \log(2SAHK)
    \end{align*}
    We denote $\tau_{s_h,a_h}^K := \min_k~s.t.~\forall k' \geq k,~\frac{1}{4}\sum_{k'=1}^{k-1}d^{\pi_{\tE}^k}_{\tE}(s_h,a_h) \geq \alpha \log (2SAHK)$. Then we have:
    \begin{align*}
        \eqref{eq:relaxed_clipping_bound} \leq & d^{\pi_{\tE}^k}_{\tE}(s_h,a_h)\Clip{\min\{H, 8B_1\sqrt{\frac{\alpha \log(B_2 K)}{\sum_{k'=1}^{k-1}d^{\pi_{\tE}^k}_{\tE}(s_h,a_h)}}\}}{\Big|\frac{\Delta_{\min}}{4eH}\vee \frac{\DeltaE(s_h,a_h)}{4e}},\quad \forall k \geq \tau_{s_h,a_h}^K.
    \end{align*}
    Therefore, for arbitrary $s_h,a_h$, there exists an absolute constant $c_{s_h,a_H}$, such that:
    \begin{align*}
        &\sum_{k=k_{start}+1}^K \EE_{\pi_{\tE}^k}[\Clip{\min\{H, 4\bonus^k_{\tE,h}(s_h,a_h)\}}{\Big|\frac{\Delta_{\min}}{4eH}\vee \frac{\DeltaE(s_h,a_h)}{4e}}]\\
        \leq & H\cdot \sum_{k=1}^{\tau_{s_h,a_h}^K} d_{\tE}^{\piE{}^k}(s_h,a_h) + \sum_{k=\tau_{s_h,a_h}^K+1}^K \EE_{\pi_{\tE}^k}[\Clip{\min\{H, 4B_1\sqrt{\frac{\alpha\log(B_2 K)}{\NE{,h}^k(s_h,a_h)}}}{\Big|\frac{\Delta_{\min}}{4eH}\vee \frac{\DeltaE(s_h,a_h)}{4e}}]\\
        \leq & c_{s_h,a_h} H \log (2SAHK)) + \sum_{k=\tau_{s_h,a_h}^K + 1}^K d^{\pi_{\tE}^k}_{\tE}(s_h,a_h)\Clip{\min\{H, 8B_1\sqrt{\frac{\alpha \log(B_2 K)}{\sum_{k'=1}^{k-1}d^{\pi_{\tE}^k}_{\tE}(s_h,a_h)}}\}}{\Big|\frac{\Delta_{\min}}{4eH}\vee \frac{\DeltaE(s_h,a_h)}{4e}} \\
        \leq & c_{s_h,a_h} H \log (2SAHK) +  c_{s_h,a_h} \cdot  \int_{\alpha \log (2SAHK)}^{K/4}\Clip{B_1\sqrt{\frac{\alpha \log(B_2 K)}{x}}}{\Big|\frac{\Delta_{\min}}{4eH}\vee \frac{\DeltaE(s_h,a_h)}{4e}}dx \\
        = & O\left(H \log (2SAHK) +  B_1(\frac{H}{\Delta_{\min}}\wedge \frac{1}{\DeltaE(s_h,a_h)})\log (B_2K)\right).
    \end{align*}
    As a result, we can establish the following regret upper bound (note that $k_{start} = O(k_{ost})$):
    \begin{align*}
        &\EE[\sum_{k=1}^K \VE{,1}^*(s_1) - \VE{,1}^{\piE{}^k}(s_1)]\\
        \leq & \sum_{k=k_{start}+1}^K 2e\EE_{\pi_{\tE}^k}[\sum_{h=1}^H \Clip{\min\{H, 4\bonus^k_{\tE,h}(s_h,a_h)\}}{\Big|\frac{\Delta_{\min}}{4eH}\vee \frac{\DeltaE(s_h,a_h)}{4e}}\mathbb{I}[\cEBk\cap\cEOk\cap\cECk]] \\
        & + H k_{start} + \sum_{k=k_{start}+1}^K H\cdot \Pr(\cEBk^\complement\cup\cEOk^\complement\cup\cECk^\complement)\\
        = & O\left(k_{start} \cdot H +  SAH^2\log(2SAHK) + B_1\log(B_2 K)\sum_{h=1}^H\sum_{s_h,a_h} (\frac{H}{\Delta_{\min}}\wedge \frac{1}{\DeltaE(s_h,a_h)}) \right).\numberthis\label{eq:RL_regret_loose}
    \end{align*}
    As introduced in maintext, because of the knowledge transfer, we may expect to achieve constant regret on some special state action pairs, and we analyze them in the following.
    \paragraph{Type 1: $(s_h,a_h) \in \cC^{\lambda, 1}_h \cup \cC^{\lambda, 2}_h$: Constant Regret because of Low Visitation Probability}
    As discussed in Lem. \ref{lem:effects_on_similar_states}, for $s_h \in \cZ_h^{\epsilon,\lambda}$, since $k_{start} \geq k_{occup} := c_{occup}\frac{C_1+\alpha C_2}{\lambda \Delta_{\min}}\log(\frac{\alpha C_1C_2 SAH}{\lambda \Delta_{\min}})$, on the event of $\cEOk,\cEBk$ and $\cECk$, we have $\forall a_h \neq \pi_{\tE}^*(s_h),~d^{\piE{}^k}_{\tE}(s_h,a_h) = 0$. Moreover, according to the definition of $\cC^{\lambda,2}_h$, we also have $\forall (s_h,a_h) \in \cC^{\lambda,2}_h,~d^{\piE{}^k}_{\tE}(s_h,a_h) = 0$. Therefore, for all $h\in[H]$ and $(s_h,a_h) \in \cC^{\lambda, 1}_h \cup \cC^{\lambda, 2}_h$:
    \begin{align*}
        \sum_{k=k_{start}+1}^K \EE_{\pi_{\tE}^k}[\Clip{\min\{H, 4\bonus^k_{\tE,h}(s_h,a_h)\}}{\Big|\frac{\Delta_{\min}}{4eH}\vee \frac{\DeltaE(s_h,a_h)}{4e}}\mathbb{I}[\cEBk\cap\cEOk\cap\cECk]] = 0.
    \end{align*}
    \paragraph{Type 2: $(s_h,a_h) \in \cC^*_h$, i.e. $d^*_{\tE}(s_h, a_h)=d^*_{\tE}(s_h) > 0$}
    Because of the sub-linear regret in Eq.~\eqref{eq:RL_regret_loose}, we may expect that $\NE{,h}^k(s_h,a_h)\approx \sum_{k=1}^K d^{\piE{}^k}_{\tE}(s_h,a_h) \sim O(K d^*_{\tE}(s_h,a_h) )$ when $K$ is large enough. To see this, note that $\sum_{\tk=1}^k (\VE{,1}^*(s_1) - \VE{,1}^{\piE{}^\tk}(s_1)) - \EE[\sum_{\tk=1}^k \VE{,1}^*(s_1) - \VE{,1}^{\piE{}^\tk}(s_1)]$ is a martingale difference sequence with bounded difference. We define $\cEEk := \{\sum_{\tk=1}^k \VE{,1}^*(s_1) - \VE{,1}^{\piE{}^\tk}(s_1) \geq H\sqrt{2 \alpha k \log k} + \EE[\sum_{\tk=1}^k \VE{,1}^*(s_1) - \VE{,1}^{\piE{}^\tk}(s_1)]\}$, according to the Azuma-Hoeffding inequality, we have:
    \begin{align*}
        \Pr(\cEEk) \leq \exp(\frac{-2\alpha H^2 k\log k}{2 k H^2}) \leq \frac{1}{k^\alpha}.
    \end{align*}
    On the event of $\cEBk, \cEOk, \cECk$ and $\cEEk$, as a result of Thm. \ref{thm:algO_regret_vs_algP_density}, we have:
    \begin{align*}
        |\sum_{\tk=1}^k d^{\piE{}^\tk}(s_h,a_h) - kd^*_{\tE}(s_h,a_h)| \leq H\sqrt{2 \alpha k \log k} + \eqref{eq:RL_regret_loose}
    \end{align*}
    To make sure $\sum_{\tk=1}^k d^{\piE{}^\tk}(s_h,a_h) \geq \frac{k}{2}d^*_{\tE}(s_h,a_h)$, we expect:
    \begin{align*}
        \frac{k}{2}d^*_{\tE}(s_h,a_h) \geq H\sqrt{2 \alpha k \log k} + \eqref{eq:RL_regret_loose}
    \end{align*}
    which can be satisfied by:
    \begin{align*}
        k \geq \bar{\tau}_{s_h} := \bar{c}^*_{s_h} \frac{1}{(d^*_{\tE}(s_h))^2}\Poly(S,A,H,\lambda^{-1}, \Delta_{\min}^{-1})
    \end{align*}
    for some constant $\bar{c}^*_{s_h}$. As a result, for $k \geq \max\{\tau_{s_h,a_h}^K, \bar{\tau}_{s_h}\} + 1$, we should have:
    \begin{align*}
        & \Clip{\min\{H, 4\bonus^k_{\tE,h}(s_h,a_h)\}}{\Big|\frac{\Delta_{\min}}{4eH}\vee \frac{\DeltaE(s_h,a_h)}{4e}}\\
        \leq & \Clip{\min\{H, 8B_1\sqrt{\frac{\alpha \log(k B_2)}{\sum_{k'=1}^{k-1}d^{\pi_{\tE}^k}_{\tE}(s_h,a_h)}}\}}{\Big|\frac{\Delta_{\min}}{4eH}} \leq \Clip{8B_1\sqrt{\frac{2\alpha \log(k B_2)}{k d^*_{\tE}(s_h,a_h)}}}{\Big|\frac{\Delta_{\min}}{4eH}}. \tag{$\DeltaE(s_h,a_h) = 0$}
    \end{align*}
    Note that $8B_1\sqrt{\frac{2\alpha \log(k B_2)}{kd^*_{\tE}(s_h,a_h)}} = 8B_1\sqrt{\frac{2\alpha \log(k B_2)}{kd^*_{\tE}(s_h)}} \leq \frac{\Delta_{\min}}{4eH}$ can be satisfied when $k \geq \tau_{s_h}' := c'_{s_h}\frac{\alpha H^2 B_1^2}{\Delta_{\min}^2 d^*_{\tE}(s_h)} \log (\frac{\alpha B_1 B_2 H}{\Delta_{\min} d^*_{\tE}(s_h)})$ for some absolute constant $c'_{s_h}$. 
    Therefore, $\min\{H, 4\bonus^k_{\tE,h}(s_h,a_h)\} \leq \frac{\Delta_{\min}}{4eH}$ and the regret will not increase after $k \geq \tau^*_{s_h,a_h} := c_{s_h,a_h}^* \max\{k_{start}, \tau_{s_h,a_h}, \bar{\tau}_{s_h}, \tau'_{s_h}\} = \Poly(S,A,H,\lambda^{-1}, \Delta_{\min}^{-1},(d^*_{\tE}(s_h))^{-1})$, for some absolute constant $c^*_{s_h,a_h}$, which implies that,
    \begin{align*}
        & \sum_{k=k_{start}+1}^K \EE_{\pi_{\tE}^k}[\Clip{\min\{H, 4\bonus^k_{\tE,h}(s_h,a_h)\}}{\Big|\frac{\Delta_{\min}}{4eH}\vee \frac{\DeltaE(s_h,a_h)}{4e}}\mathbb{I}[\cEBk\cap\cEOk\cap\cECk]]\\
        =
        & \sum_{k=k_{start}}^{\tau^*_{s_h,a_h}} \EE_{\pi_{\tE}^k}[\Clip{\min\{H, 4\bonus^k_{\tE,h}(s_h,a_h)\}}{\Big|\frac{\Delta_{\min}}{4eH}\vee \frac{\DeltaE(s_h,a_h)}{4e}}\mathbb{I}[\cEBk\cap\cEOk\cap\cECk]]\\
        = & O\left(\frac{HB_1}{\Delta_{\min}} \log (SAH B_2 \min\{K, \frac{1}{\lambda \Delta_{\min} d^*_{\tE}(s_h)}\})\right).
    \end{align*}
    where in last step, we use the fact that for $\Delta_{\tE}(s_h,a_h) = 0$:
    \begin{align*}
        O\Big(H \log (2SAHK) +  B_1 \log (B_2K) \cdot (\frac{H}{\Delta_{\min}}\wedge \frac{1}{\DeltaE(s_h,a_h)})\Big) = O\Big(\frac{HB_1}{\Delta_{\min}} \log (SAHB_2 K)\Big).
    \end{align*}
    As a summary, recall $k_{start}=\max\{k_{ost}, k_{occup}\}$, we have:
    \begin{align*}
        \Regret_{K}(\ME) = & O\Big(H\cdot \max\{\alpha\frac{B_1^2H^2S}{\lambda^2\Delta_{\min}^2}\log(\alpha HSAB_1B_2), \frac{(C_1+\alpha C_2)SH}{\lambda\Delta_{\min}^2}\log\frac{C_1C_2SAH}{\Delta_{\min}\lambda}\} \\
        &\quad \quad + \sum_{h=1}^H \sum_{(s_h,a_h) \in \cC^*_h} \frac{HB_1}{\Delta_{\min}} \log (SAH B_2 (K \wedge \frac{1}{\lambda \Delta_{\min} d^*_{\tE}(s_h)})) \\
        & \quad \quad + \sum_{h=1}^H \sum_{(s_h,a_h)\in \cS_h\times\cA_h \setminus \cC^{\lambda}_h} H \log (2SAHK) +  B_1 (\frac{H}{\Delta_{\min}}\wedge \frac{1}{\DeltaE(s_h,a_h)})\log (B_2K) \Big).
    \end{align*}
    By considering $B_1=O(SH)$ and $B_2=O(SA)$ in Example~\ref{example:choice_of_B}, and omitting all the constant terms independent w.r.t. $K$, we can rewrite the above upper bound to:
    \begin{align*}
        \Regret_{K}(\ME) = & O\Big(H\cdot \max\{\alpha\frac{S^3H^4}{\lambda^2\Delta_{\min}^2}\log(\alpha SAH), \frac{(C_1+\alpha C_2)SH}{\lambda\Delta_{\min}^2}\log\frac{C_1C_2SAH}{\Delta_{\min}\lambda}\} \\
        &\quad \quad + \sum_{h=1}^H \sum_{(s_h,a_h) \in \cC^*_h} \frac{SH^2}{\Delta_{\min}} \log (SAH (K \wedge \frac{1}{\lambda \Delta_{\min} d^*_{\tE}(s_h)})) \\
        & \quad \quad + SH \sum_{h=1}^H \sum_{(s_h,a_h)\in \cS_h\times\cA_h \setminus \cC^{\lambda}_h} (\frac{H}{\Delta_{\min}}\wedge \frac{1}{\DeltaE(s_h,a_h)})\log (SAHK) \Big)\\
        =& O\Big(
        SH \sum_{h=1}^H \sum_{(s_h,a_h)\in \cS_h\times\cA_h \setminus \cC^{\lambda}_h} (\frac{H}{\Delta_{\min}}\wedge \frac{1}{\DeltaE(s_h,a_h)})\log (SAHK)\Big).
    \end{align*}
\end{proof}
\section{Proofs for Tiered MAB with Multiple Source/Low-Tier Tasks}\label{appx:bandit_multi_task_version}
\LemAbsSimTask*
\begin{proof}
    In the following, we denote $\cW^* := \{\task \in [\Task]| i^*_{\tO,\task} = i^*_{\tE},~\muO{}(i^*_{\tE,\task^*}) \leq \muE{}(i^*_{\tE}) + \frac{\tilde{\Delta}_{\min}}{4}\}$.  $\tcW^* := \{\task \in [\Task]| i^*_{\tO,\task} = i^*_{\tE}\}$. 
    In another word, $\cW^*$ includes all transferable tasks, while $\tcW^*$ includes all the tasks which share the optimal action $i^*_{\tE}$ regardless of whether the value function are close enough or not.

    Consider the event $\cE := \{\exists k' \in [\frac{k}{2}, k],~s.t.~\cW^*\cap\cI^{k'}=\emptyset\} \cup \{\exists k' \in [\frac{k}{2}, k],~\exists \task \in [\Task],~s.t.~\upiO{,\task}^{k'} \neq i^*_{\tE}\}$, note that on its complement: $\cE^\complement := \{\forall k' \in [\frac{k}{2}, k],~\cW^*\cap\cI^{k'}\neq\emptyset\} \cap \{\forall k' \in [\frac{k}{2}, k],~\forall \task \in [\Task],~\upiO{,\task}^{k'} = i^*_{\tO,\task}\}$, if $\piE{}^k \neq i^*_{\tE}$ still happens, we must have:
        $
        \{\forall k' \in [\frac{k}{2}, k],~\exists \task \not\in \tcW^* ,~\task \in \cI^{k'},~\task^{k'}=\task\}.
        $
    That's because, if $\cE^\complement$ holds, and $\cI^{k'} \subset \tcW^* $ for some $k'$, no matter which task in $\tcW^*$ is chosen as $\task^{k'}$, for all $\tilde{k} \in [k',k]$, no matter whether $w^{\tilde{k}}$ changes or not, the action we transfer is always $i^*_{\tE}$ (i.e. $\piE{}^{\tk} = i^*_{\tE}$), because of the action inheritance startegy. Therefore,
    \begin{align*}
        \Pr(\piE{}^k \neq i^*_{\tE}) \leq& \Pr(\cE) + \Pr(\cE^\complement\cap \{\piE{}^k \neq i^*_{\tE}\})\\
        \leq& \Pr(\cE) + \Pr(\cE^\complement\cap \{\forall k' \in [\frac{k}{2}, k],~\exists \task \not\in \tcW^* ,~\task \in \cI^{k'},~\task^{k'}=\task\})\\
        \leq & \sum_{k'=\frac{k}{2}}^k \Pr(\task^* \not\in \cI^{k'}) + \sum_{k'=\frac{k}{2}}^k \sum_{\task=1}^\Task  \Pr(\upiO{,\task}^k \neq i^*_{\tO,\task}) \tag{$\Pr(\cW^*\cap\cI^{k'} = \emptyset) \leq \Pr(\task^* \not\in \cI^{k'})$}\\
        & +\sum_{k'=\frac{k}{2}}^k  \sum_{\task=1}^\Task  \Pr(\cE^\complement\cap \{\forall k' \in [\frac{k}{2}, k],~\exists \task \not\in \tcW^* ,~\task \in \cI^{k'},~\task^{k'}=\task\}).\numberthis\label{eq:t_neq_ttrust}
    \end{align*}
    For the first and second term, by considering $f(k)=1+16A^2\Task(k+1)^2$, with a similar discussion as Lem. \ref{lem:valid_under_est} and Lem. \ref{lem:behavior_of_UCB}, we have, for arbitrary $k \geq k_{\max}^{[\Task]}:=\max_{\task\in[\Task]}\{k^\task_{\max}\}$, where $k^\task_{\max} := c(A + \frac{\alpha A}{\DeltaO^2(i)}\log(1+\frac{\alpha A \Task}{\Delta_{\min}}))$ for some constant $c$ is an analogue of $k_{\max}$ defined in Lem. \ref{lem:behavior_of_UCB} specified on task $t$:
\begin{align*}
    \sum_{k'=\frac{k}{2}}^k \Pr(\task^* \not\in \cI^{k'}) + \sum_{\task=1}^\Task \sum_{k'=\frac{k}{2}}^k \Pr(\upiO{,\task}^k \neq i^*_{\tO,\task}) \leq \frac{2^{2\alpha}}{T\cdot k^{2\alpha - 1}} + \frac{2\cdot 2^{2\alpha}A\Task}{T\cdot k^{2\alpha-1}}+\frac{2\cdot 2^{2\alpha-1} A^2T}{T\cdot k^{2\alpha-2}} \leq \frac{3\cdot 2^{2\alpha-1} A^2}{k^{2\alpha-2}}.\numberthis\label{eq:sub_step_1}
\end{align*}
Therefore, we mainly focus on the second term. 
We denote $k'' := c''\frac{\alpha A }{\Delta_{\min}^2}\log\frac{\alpha A \Task}{\Delta_{\min}}$ for some constant $c''$, such that for all $\tk \geq k''$, we always have $\tk \geq 2A\cdot \frac{512\alpha \log f(\tk)}{\Delta_{\min}^2}$.
Therefore, for arbitrary $\task \not\in \tcW^*$, and arbitrary $\tk \geq \max\{k_{\max}^{[\Task]}, k''\}$ we have:
\begin{align*}
    &\Pr(\{\task \in \cI^\tk\} \cap \{\NE{}^\tk(i^*_{\tO,\task}) \geq \frac{512\alpha\log f(\tk)}{\Delta_{\min}^2}\} ) \\
    \leq & \Pr(\{\task \in \cI^\tk\} \cap \{\NE{}^\tk(i^*_{\tO,\task}) \geq \frac{512\alpha\log f(\tk)}{\Delta_{\min}^2}\} \cap \{\NO{,\task}^\tk(i^*_{\tO,\task}) > \frac{\tk}{2}\} ) + \Pr(\{\NO{,\task}^\tk(i^*_{\tO,\task}) \leq \frac{\tk}{2}\}) \\
    \leq & \frac{2A}{T\cdot\tk^{2\alpha-1}} + \Pr(\{\umuO{}^{t,\tk}(i^*_{\tO,\task}) \leq \omuE{}^{t,\tk}(i^*_{\tO,\task}) + \epsilon\}\cap\{\NO{,\task}^\tk(i^*_{\tO,\task}) > \frac{\tk}{2}\}\cap\{\NE{}^\tk(i^*_{\tO,\task}) \geq \frac{512\alpha\log f(\tk)}{\Delta_{\min}^2}\})
\end{align*}
For the second part, it equals:
\begin{align*}
    &\Pr(\{\hmuO{}^k(i_{\tO}^*) - \muO{}(i_{\tO}^*) - \sqrt{\frac{2\alpha\log f(k)}{\NO{,\task}^k(i_{\tO}^*)}} \leq \hmuE{}^k(i_{\tO}^*) - \muE{}(i_{\tO}^*) + (\muE{}(i_{\tO}^*) - \muO{}(i_{\tO}^*)) + \sqrt{\frac{2\alpha\log f(k)}{\NE{}^k(i_{\tO}^*)}} + \epsilon\}\\
    &\quad \cap\{\NO{,\task}^k(i^*_{\tO,\task}) > \frac{k}{2}\}\cap\{\NE{}^k(i^*_{\tO,\task}) \geq \frac{512\alpha\log f(k)}{\Delta_{\min}^2}\})\\
    \leq & \Pr(\{\hmuO{}^k(i_{\tO}^*) - \muO{}(i_{\tO}^*) - \sqrt{\frac{2\alpha\log f(k)}{\NO{,\task}^k(i_{\tO}^*)}} \leq \hmuE{}^k(i_{\tO}^*) - \muE{}(i_{\tO}^*) - \frac{\Delta_{\min}}{4} + \sqrt{\frac{2\alpha\log f(k)}{\NE{}^k(i_{\tO}^*)}} \}\\
    &\quad \cap\{\NO{,\task}^k(i^*_{\tO,\task}) > \frac{k}{2}\}\cap\{\NE{}^k(i^*_{\tO,\task}) \geq \frac{512\alpha\log f(k)}{\Delta_{\min}^2}\})\\
    \leq & \Pr(\hmuO{}^k(i_{\tO}^*) - \muO{}(i_{\tO}^*) \leq - \sqrt{\frac{2\alpha\log f(k)}{\NO{,\task}^k(i_{\tO}^*)}}) + \Pr( \sqrt{\frac{2\alpha\log f(k)}{\NE{}^k(i_{\tO}^*)}} \leq \hmuE{}^k(i_{\tO}^*) - \muE{}(i_{\tO}^*)) \tag{$\sqrt{\frac{2\alpha\log f(k)}{\NO{,\task}^k(i_{\tO}^*)}}, \sqrt{\frac{2\alpha\log f(k)}{\NE{}^k(i_{\tO}^*)}} \leq \frac{\Delta_{\min}}{16}$}\\
    \leq & \frac{1}{T\cdot k^{2\alpha}}.
\end{align*}
Therefore, we can conclude that 
\begin{align*}
    \forall k \geq \max\{k_{\max}^{[\Task]}, k''\},\quad \Pr(\{\task \in \cI^k\} \cap \{\NE{}^k(i^*_{\tO,\task}) \geq \frac{512\alpha\log f(k)}{\Delta_{\min}^2}\}) \leq \frac{4}{T\cdot k^{2\alpha}} + \frac{2A}{T\cdot k^{2\alpha-1}} \leq \frac{3A}{T\cdot k^{2\alpha-1}}.
\end{align*}
Next, we are ready to upper bound the second term in Eq.~\eqref{eq:t_neq_ttrust}.
The key observation is that, as long as the event $\{\forall k' \in [\frac{k}{2}, k],~\exists \task \not\in \tcW^* ,~\task \in \cI^{k'},~\task^{k'}=\task\}$ happens, no matter what the sequence $\{\task^{k'}\}_{k'=k/2}^k$ is, since we only have $A-1$ sub-optimal arms, and $i^*_{\tO,\task^{k'}} \neq i^*_{\tE}$ for all $k'\in[k/2,k]$, there must be an arm which has been taken for at least $\frac{k}{2(A-1)}$ times from step $k/2$ to $k$, therefore,
\begin{align*}
    &\Pr(\forall k' \in [\frac{k}{2}, k],~\exists \task \not\in \tcW^* ,~\task \in \cI^{k'},~\task^{k'}=\task) \\
    \leq & \Pr(\exists k'\in[\frac{k}{2},k],~\exists \task\not\in \tcW^*,~s.t.~\NE{}^{k'}(i^*_{\tO,\task}) = \NE{}^{k/2}(i^*_{\tO,\task}) + \frac{k}{2(A-1)} - 1,~\task\in\cI^{k'})\\
    \leq & \sum_{k'=\frac{k}{2}}^k\sum_{\task \not\in \tcW^*}\Pr(\{\task\in\cI^{k'}\}\cap\{\NE{}^{k'}(i^*_{\tO,\task}) \geq \frac{k}{2(A-1)} - 1,~\task\in\cI^{k'}\})\\
    \leq & \sum_{k'=\frac{k}{2}}^k\sum_{\task \not\in \tcW^*} \Pr(\{\task \in \cI^{k'}\} \cap \{\NE{}^{k'}(i^*_{\tO,\task}) \geq \frac{512\alpha\log f(k)}{\Delta_{\min}^2}\})\\
    \leq & \frac{3\cdot 2^{2\alpha-1}A}{k^{2\alpha-2}}.\numberthis\label{eq:sub_step_2}
\end{align*}
According to the definition of $k_{\max}^{[\Task]}$ and $k''$, there must exists a constant $c^*$ such that $\max\{k'', k_{\max}^{[\Task]}\} \leq c^*\frac{\alpha A }{\Delta_{\min}^2}\log\frac{\alpha A \Task}{\Delta_{\min}}$. By choosing such $c^*$, and combining Eq.~\eqref{eq:sub_step_1} and Eq.~\eqref{eq:sub_step_2}, we have:
\begin{align*}
    \Pr(\piE{}^k \neq i^*_{\tE})  \leq \frac{3\cdot 2^{2\alpha-1} A^2}{k^{2\alpha-2}} + \frac{3\cdot 2^{2\alpha-1}A}{k^{2\alpha-2}} \leq \frac{16A}{(k/2)^{2\alpha-2}} = O(\frac{A}{k^{2\alpha-2}}).
\end{align*}
\end{proof}

\begin{lemma}[Extension of Lem. \ref{lem:upper_bound_NEK}]\label{lem:upper_bound_NEK_MT}
    For arbitrary $K \geq A + 1$ and arbitrary $1\leq k_0 \leq K$, we have:
    \begin{align*}
        \NE{}^{K}(i) \leq & k_0 + \sum_{k=k_0 + 1}^{K} \mathbb{I}[\{\task^k \neq \Null\}\cap \{\umuOtk{}^k(\underline{\pi}_{\tO,\task^k}^k) \leq \omuE{}^k(\underline{\pi}_{\tO,\task^k}^k) + \epsilon\} \\
        &\quad\quad\quad\quad\quad \cap \{\NO{,\task^k}^k(\underline{\pi}_{\tO,\task^k}^k) > k/\ratio\}\cap\{i = \upiO{,\task^k}^k\}\cap\{\pi_{\tE}^k = i\}]\\
        & +\sum_{k=k_0 + 1}^{K} \mathbb{I}[0 \geq \hmuE{}^k(i^*) + \sqrt{\frac{2\alphaE \log f(k)}{\NE{}^k(i^*)}} - \muE(i^*)]\\
        & +\sum_{k=k_0 + 1}^{K} \mathbb{I}[\tilde{\mu}_{\tE}^k(i) + \sqrt{\frac{2\alphaE \log f(K)}{k}} - \muE(i) - \DeltaE(i)\geq 0]\numberthis\label{eq:decomposition_of_N_MT}
    \end{align*}
\end{lemma}
We omit the proof here since it is almost the same as Lem. \ref{lem:upper_bound_NEK}, except that we need to specify $\task^k$.

\ThmMTMAB*
\begin{proof}
    One key observation is that when analyzing $\NE{}^K$, we only need to analyze the second term in Eq.~\eqref{eq:decomposition_of_N_MT}, since the others can be directly bounded.

    We first study the case when there exists $\task^*$ such that $\ME$ and $M_{\tO,\task^*}$ are $\epsilon$-close. As a result of Lem. \ref{lem:absorbing_to_similar_task}, we only need to upper bound the regret before step $k^*$.
    Similar to the proof of Thm. \ref{thm:regret_bound_in_bandit}, we separate two cases for each $k$ and $i$. 
    \paragraph{Case 1-(a) $\task^k \neq \Null$, $0<\DeltaE(i) \leq 4\Delta_{\tO,\task^k}(i)$} 
    In the following, we will define $\tk_{t,i} := 3 + c\cdot \frac{3 \alpha }{\DeltaOt^2(i)}\log(T+\frac{\alpha A T}{\Delta_{\min}})$ (i.e. similar to the role of $\tk_i$ in Lem. \ref{lem:behavior_of_UCB} with specified task index $t$). Since $\DeltaE(i) \leq 4\Delta_{\tO,\task^k}(i)$, we define $\tk_{[\Task],i} := 3 + c\cdot \frac{48 \alpha }{\DeltaE^2(i)}\log(T+\frac{\alpha A T}{\Delta_{\min}})$. In this case, obviously $\tk_{[\Task],i} \geq \tk_{t,i}$. As a result of Lem. \ref{lem:behavior_of_UCB},
    \begin{align*}
        &\mathbb{I}[ \{\task^k \neq \Null\}\cap\{\umuOtk{}^k(\underline{\pi}_{\tO,\task^k}^k) \leq \omuE{}^k(\underline{\pi}_{\tO,\task^k}^k) + \epsilon\} \cap \{\NO{,\task^k}^k(\underline{\pi}_{\tO,\task^k}^k) > k/\ratio\}\cap\{i = \upiO{,\task^k}^k\}\cap\{\pi_{\tE}^k = i\}]\\
        \leq& \mathbb{I}[k \leq \tk_{[\Task],i}] + \mathbb{I}[\{\task^k \neq \Null\}\cap\{k > \tk_{[\Task],i}\}\cap\{\NO{,\task^k}^k(\underline{\pi}_{\tO,\task^k}^k) > k/\ratio\}].\numberthis\label{eq:0_DE_4DO_MT}
    \end{align*}
    by taking expectation, we have:
    \begin{align*}
        &\Pr(\mathbb{I}[ \{\task^k \neq \Null\}\cap\{\umuOtk{}^k(\underline{\pi}_{\tO,\task^k}^k) \leq \omuE{}^k(\underline{\pi}_{\tO,\task^k}^k) + \epsilon\} \\
        &\quad\quad\quad\quad\quad \cap \{\NO{,\task^k}^k(\underline{\pi}_{\tO,\task^k}^k) > k/\ratio\}\cap\{i = \upiO{,\task^k}^k\}\cap\{\pi_{\tE}^k = i\}])\\
        \leq & \mathbb{I}[k \leq \tk_{[\Task],i}] + \frac{2}{k^{2\alpha-1}}.
    \end{align*}
    \paragraph{Case 1-(b) $\task^k \neq \Null$, $\DeltaE(i) > 4\Delta_{\tO,\task^k}(i) \geq 0$}
    We consider $\bar{k}_{i} := \frac{c_{\tE,i}\alpha}{\DeltaE^2(i)}\log \frac{A\Task}{\Delta_{\min}}$, where $c_{\tE,i}$ is the minimal constant, such that when $k \geq \frac{c_{\tE,i}\alpha}{\DeltaE^2(i)}\log \frac{\alpha A\Task}{\Delta_{\min}}$, we always have $k \geq  \frac{256\alpha\log f(k)}{\DeltaE^2(i)}$.
    With a similar discussion as Eq.~\eqref{eq:upper_bound_case1b}, we have:
    \begin{align*}
        &\mathbb{I}[\{\task^k \neq \Null\}\cap\{\umuOtk{}^k(\underline{\pi}_{\tO,\task^k}^k) \leq \omuE{}^k(\underline{\pi}_{\tO,\task^k}^k) + \epsilon\} \cap \{\NO{,\task^k}^k(\underline{\pi}_{\tO,\task^k}^k) > k/\ratio\}\cap\{i = \upiO{,\task^k}^k\}\cap\{\pi_{\tE}^k = i\}]\\
        \leq&  \mathbb{I}[k < \bar{k}_{i}] + \mathbb{I}[\{\task^k \neq \Null\}\cap\{\{\umuO{,\task^k}^k(\underline{\pi}_{\tO,\task^k}^k) \leq \omuE{}^k(\underline{\pi}_{\tO,\task^k}^k) + \epsilon\}\}\\
        &\quad\quad\quad\quad\quad\quad\quad \cap\{\NO{,\task^k}^k(i) \geq \frac{128\alphaO\log f(k)}{\DeltaE^2(i)}\}\cap\{i = \upiO{}^k\}\cap\{\pi_{\tE}^k = i\}] \\
        \leq & \mathbb{I}[k < \bar{k}_{i}] +  \mathbb{I}[\{\task^k \neq \Null\}\cap\{\hmuO{,\task^k}^k(i) - \mu_{\tO,\task^k}(i) \leq -\sqrt{\frac{2\alphaO \log f(k)}{\NO{,\task^k}^k(i)}}\}] \\
        & + \mathbb{I}[\{\hmuE{}^k(i) - \muE(i) + \sqrt{\frac{2\alphaE \log f(k)}{\NE{}^k(i)}} \geq \frac{\DeltaE(i)}{4}\}\cap\{\pi_{\tE}^k = i\}].\numberthis\label{eq:4DO_DE_MT}
    \end{align*}
    We denote $k_i^* := \max\{\tk_{[\Task],i}, \bar{k}_i\}$. Combining Eq.~\eqref{eq:0_DE_4DO_MT} and \eqref{eq:4DO_DE_MT} with Eq.~\eqref{eq:decomposition_of_N_MT}, for arbitrary $K$, we have (recall that $\tilde{\mu}_{\tE}^k(i)$ is defined to be the average of $k$ random samples from reward distribution of arm $i$ in $\ME$):
    \begin{align*}
        &\EE[\NE{}^K(i)] \\
        \leq& \sum_{k=1}^{K} \mathbb{I}[k \leq \tk_{[\Task],i}] +  \sum_{k=1}^{K} \frac{2}{k^{2\alpha-1}} + \sum_{k=1}^{K} \mathbb{I}[k < \bar{k}_i] \\
        & +  \sum_{k=1}^{K} \Pr(\{\task^k \neq \Null\}\cap\{\hmuO{,\task^k}^k(i) - \mu_{\tO,\task^k}(i) \leq -\sqrt{\frac{2\alphaO \log f(k)}{\NO{,\task^k}^k(i)}}\}) \\
        & + \EE[\sum_{k=1}^{K} \mathbb{I}[\{\hmuE{}^k(i) - \muE(i) + \sqrt{\frac{2\alphaE \log f(k)}{\NE{}^k(i)}} \geq \frac{\DeltaE(i)}{4}\}\cap\{\pi_{\tE}^k = i\}]]\\
        &+\sum_{k=1}^{K} \Pr(0 \geq \hmuE{}^k(i^*) + \sqrt{\frac{2\alphaE \log f(k)}{\NE{}^k(i^*)}} - \muE(i^*))\\
        &+\sum_{k=1}^{K} \EE[\mathbb{I}[\tilde{\mu}_{\tE}^k(i) + \sqrt{\frac{2\alphaE \log f(K)}{k}} - \muE(i) - \DeltaE(i)\geq 0]]\\
        \leq & \tk_{[\Task],i} + \sum_{k=1}^K \frac{2}{k^{2\alpha-1}} + \bar{k}_i + 2\cdot \sum_{k=1}^{K} \frac{1}{8Ak^{2\alpha}} + 2\sum_{k=1}^{K} \EE[\mathbb{I}[\tilde{\mu}_{\tE}^k(i) + \sqrt{\frac{2\alphaE \log f(K)}{k}} - \muE(i) - \frac{\DeltaE(i)}{4}\geq 0]]\\
        = & O(\frac{1}{\DeltaE^2(i)}\log \Task K) \numberthis\label{eq:MT_NEK_upper_bound}.
    \end{align*}
    As a result, combining with Lem. \ref{lem:absorbing_to_similar_task}, we can conclude that:
    \begin{align*}
        \Regret_{K}(\ME) =&  \sum_{i \neq i^*} \DeltaE(i)\EE[\NE{}^{k^*}(i)] + \sum_{k = k^* + 1}^K \Pr(\pi_{\tE}^k \neq i^*_{\tE}) = O(\sum_{i \neq i^*} \frac{1}{\DeltaE(i)}\log\frac{A\Task}{\Delta_{\min}})
    \end{align*}
    Next, we study the case when there is no task among $\{\MOt\}_{\task=1}^\Task$ close enough to $\ME$. Similarly, we also decompose into three cases.
    \paragraph{Case 2-(a) $\task^k \neq \Null$, $i \neq i^*_{\tO,\task^k}, i\neq i^*_{\tE}$ and $0<\DeltaE(i) \leq 4\Delta_{\tO,\task^k}(i)$}
    The result is the same as Eq.~\eqref{eq:0_DE_4DO_MT}.

    \paragraph{Case 2-(b) $\task^k \neq \Null$, $i \neq i^*_{\tO,\task^k}, i\neq i^*_{\tE}$ and $\DeltaE(i) > 4\Delta_{\tO,\task^k}(i) > 0$}
    The result is the same as Eq.~\eqref{eq:4DO_DE_MT}.

    \paragraph{Case 2-(c) $\task^k \neq \Null$, $i = i^*_{\tO,\task^k}$} If $i^*_{\tO,\task^k} = i^*_{\tE}$, $\ME$ suffers no regret when choosing $i^*_{\tO,\task^k}$.
    Therefore, in the following, we only study the case when $i^*_{\tO,\task^k} \neq i^*_{\tE}$.
    For arm $i$ (note that $i = i^*_{\tO,\task^k}$ in this case), we define $k_{\max}' = \frac{c_{\max}' \alpha}{\DeltaE(i)^2}\log\frac{\alpha A T}{\Delta_{\min}}$, where $c_{\max}'$ is the minimal constant, such that for all $k \geq \frac{c_{\max}' \alpha}{\DeltaE(i)^2}\log\frac{\alpha A T}{\Delta_{\min}}$, we always have $k \geq \frac{1024\alpha\log f(k)}{\DeltaE(i)^2}$.
    With a similar discussion as Eq.~\eqref{eq:case_2c}, for the following event, we have:
    \begin{align*}
        &\mathbb{I}[\{\task^k\neq\Null\}\cap\{\umuO{,\task^k}^k(\upiO{,\task^k}^k) \leq \omuE{}^k(\upiO{,\task^k}^k) + \epsilon\} \cap \{\NO{,\task^k}^k(\upiO{,\task^k}^k) > k/\ratio\}\cap \{i = \upiO{,\task^k}^k\}  \cap\{\pi_{\tE}^k = i\}]\\
        \leq & \mathbb{I}[k \leq k_{\max}'] + \mathbb{I}[\{\task^k \neq \Null\} \cap \hmuO{,\task^k}^k(i) - \mu_{\tO,\task^k}(i) \leq -\sqrt{\frac{2\alpha\log f(k)}{\NO{,\task^k}^k(i_{\tO}^*)}}] \\
        & + \mathbb{I}[\{\hmuE{}^k(i) - \muE{}(i) + \sqrt{\frac{2\alpha\log f(k)}{\NE{}^k(i)}}\geq \frac{\DeltaE(i)}{8}\}\cap\{\pi_{\tE}^k = i\}].
    \end{align*}
    Combining all the cases above, similar to Case 1, for arbitrary $i\neq i^*_{\tE}$, we can conclude:
    \begin{align*}
        &\EE[\NE{}^K(i)] \\
        \leq& \sum_{k=1}^{K} \mathbb{I}[k \leq \tk_{[\Task],i}] +  \sum_{k=1}^{K} \frac{2}{k^{2\alpha-1}} + \sum_{k=1}^{K} \mathbb{I}[k < \bar{k}_i] +  \sum_{k=1}^{K} \Pr(\{\task^k \neq \Null\}\cap\{\hmuO{}^k(i) - \muO(i) \leq -\sqrt{\frac{2\alphaO \log f(k)}{\NO{,\task^k}^k(i)}}\}) \\
        & + \EE[\sum_{k=1}^{K} \mathbb{I}[\{\hmuE{}^k(i) - \muE(i) + \sqrt{\frac{2\alphaE \log f(k)}{\NE{}^k(i)}} \geq \frac{\DeltaE(i)}{4}\}\cap\{\pi_{\tE}^k = i\}]]\\
        &+\sum_{k=1}^{K} \Pr(0 \geq \hmuE{}^k(i^*) + \sqrt{\frac{2\alphaE \log f(k)}{\NE{}^k(i^*)}} - \muE(i^*))\\
        &+\sum_{k=1}^{K} \EE[\mathbb{I}[\tilde{\mu}_{\tE}^k(i) + \sqrt{\frac{2\alphaE \log f(K)}{k}} - \muE(i) - \DeltaE(i)\geq 0]]\\
        & + \sum_{k=1}^K \mathbb{I}[k \leq k_{\max}'] +  \sum_{k=1}^{K} \Pr(\{\task^k \neq \Null\}\cap\{\hmuO{}^k(i) - \muO(i) \leq -\sqrt{\frac{2\alphaO \log f(k)}{\NO{,\task^k}^k(i)}}\}) \\
        & + \EE[\sum_{k=1}^{K} \mathbb{I}[\{\hmuE{}^k(i) - \muE(i) + \sqrt{\frac{2\alphaE \log f(k)}{\NE{}^k(i)}} \geq \frac{\DeltaE(i)}{4}\}\cap\{\pi_{\tE}^k = i\}]]\\
        \leq & \tk_{[\Task],i} + \sum_{k=1}^K \frac{2}{k^{2\alpha-1}} + \bar{k}_i + 2\cdot \sum_{k=1}^{K} \frac{1}{8Ak^{2\alpha}} + 3\sum_{k=1}^{K} \EE[\mathbb{I}[\tilde{\mu}_{\tE}^k(i) + \sqrt{\frac{2\alphaE \log f(K)}{k}} - \muE(i) - \frac{\DeltaE(i)}{8}\geq 0]]\\
        = & O(\frac{1}{\DeltaE^2(i)}\log \Task K).
    \end{align*}
    which implies:
    \begin{align*}
        \Regret_{K}(\ME) =&  \sum_{i \neq i^*} \DeltaE(i)\EE[\NE{}^{K}(i)] = O(\sum_{i \neq i^*} \frac{1}{\DeltaE(i)}\log \Task K).
    \end{align*}
\end{proof}

\section{Proofs for Tiered RL with Multiple Source/Low-Tier Tasks}\label{appx:RL_multi_task_version}
We first introduce the notion of transferable states in this multi-source tasks setting. 
Comparing with Def.~\ref{def:transferable_states}, we have an additional constraint on $d^*_\tE(s_h) > 0$. This is because, to distinguish which source task to transfer, we require $s_h$ in $\ME$ to be visited frequently enough for accurate estimation, and it is only possible for those $s_h$ on optimal trajectories given that we expect $\Regret_K(\ME)$ is at least near-optimal.
\begin{definition}[$\lambda$-Transferable States in TRL-MST]\label{def:transferable_states_MT}
    Given any $\lambda > 0$, we say $s_h$ is $\lambda$-transferable if $d^*_{\tE}(s_h) > 0$, and $\exists w\in[W]$, such that $d^*_{\tO,w}(s_h) \geq \lambda$ and $\ME$ is $\frac{\tilde{\Delta}_{\min}}{4(H+1)}$-close to $M_{\tO,w}$ on state $s_h$.
    We use $\cZ_{h}^{\lambda,[W]}$ to denote the set of $\lambda$-transferable state at step $h \in [H]$.
\end{definition}

\begin{definition}[Benefitable States in TRL-MST]\label{def:benefitable_states_MT}
    Similar to the single task case, we define $\cC_h^{\lambda,[W],1} := \{(s_h,a_h) | s_h \in \cZ^{\lambda,[W]}, a_h \neq \pi^*_{\tE,h}(s_h)\}$, $\cC_h^{\lambda,[W],2} := \{(s_h,a_h) | \text{Block}(\{\cC^{\lambda,[W],1}_\ph\}_{\ph=1}^{h-1},s_h)=\text{True},s_h\not\in \cC_h^{\lambda,[W],1},a_h\in\cA_h\}$ and $\cC^*_h := \{(s_h,a_h)|d^*_\tE(s_h,a_h) > 0\}$, 
which represents the three categories of state-action pairs with constant regret.  We define $\cC_h^{\lambda,[W]} := \cC_h^{\lambda,[W],1} \cup \cC_h^{\lambda,[W],2} \cup \cC^*_h$, which captures the benefitable state-action pairs.  
\end{definition}
\begin{remark}[Constant Regret in Entire $M_\tE$]
For each individual task $M_{\tO,w}$, the additional constraint $d^*_{\tE}(s_h) > 0$ reduces the size of transferable states comparing with single task learning setting. However, if the tasks are diverse enough, we expect $\cC^{\lambda,[W]}_h$ to be much larger than $\cC^{\lambda}_h$ in Sec.~\ref{sec:single_task_RL} and we can achieve more benefits with only an additional cost of order $\log W$.
Besides, if $\forall h, s_h$ with $d^*_\tE(s_h) > 0$ we have $s_h \in \cZ^{\lambda,[W]}_h$, then, $\cC^{\lambda,[W]}_h = \cS_h\times\cA_h$ and we can achieve constant regret for the entire $\ME$.
\end{remark}

\subsection{Additional Algorithms, Conditions and Notations}\label{appx:RL_add_Alg_Cond_Nota_MT}
Our algorithm is provided in Alg.~\ref{alg:RL_Setting_MT}, which is extended from Alg.~\ref{alg:RL_Setting} and integrated with the ``Trust till Failure'' strategy introduced in bandit setting in Alg.~\ref{alg:Bandit_Setting_MT}.
\begin{algorithm}
    \textbf{Input}: 
    Ratio $\lambda \in (0,1)$; 
    Bonus term computation function \textbf{Bonus}; 
    Sequence of confidence level $(\delta_k)_{k\geq 1}$ with $\delta_k = 1/SAH\Task k^\alpha$;
    Model learning function \textbf{ModelLearning}. $\epsilon < \Delta_{\min}/4(H+1)$\\
    \textbf{Initialize}: $D^0_{\tOt} \gets \{\}$ for $\task\in[\Task]$, $D^0_{\tE} \gets \{\}$, set $\uVE{,h+1}^k, \uQE{,h+1}^k, \uVOt{,h+1}^k, \uQOt{,h+1}^k,\tVOt{,h+1}^k,\tQOt{,h+1}^k$ to be 0 for all $k=1,2,...$.\\
    \For{$k=1,2,...$}{
        \For{$t = 1,2,...T$}{
            $\pi_{\tOt}^k \gets \algO(D^{k-1}_{\tOt})$; collect data from $\MOt$ with $\pi_{\tOt}^k$; update $D_{\tOt}^{k} \gets D_{\tOt}^{k-1} \cup \{\tau^k_t\}.$ \\
            $\{\hmPOt{,h}^k\}_{h=1}^H \gets \textbf{ModelLearning}(D_{\tOt}^{k-1}),\quad \{\bonus_{\tO,h}^k\}_{h=1}^H \gets \textbf{Bonus}(D_{\tOt}^{k-1},~\delta_k)$.\\
            \For{$h=H,H-1...,1$}{
                $\uQOt{,h}^k(\cdot,\cdot) \gets \max\{0, \rO{,h}(\cdot,\cdot) + \hmPOt{,h}^k\uVOt{,h+1}^k(\cdot,\cdot) - \bonus_{\tO,h}^k(\cdot,\cdot)\}.$ \\
                $\uVOt{,h}^k(\cdot) = \max_a \uQOt{,h}^k(\cdot,a),\quad \upiO{,\task,h}^k(\cdot) \gets \arg\max_a \uQOt{,h}^k(\cdot,a).$\\
            }
        }
        $\{\hmPE{,h}^k\}_{h=1}^H \gets \textbf{ModelLearning}(D_{\tE}^{k-1}),\quad \{\bonus_{\tE,h}^k\}_{h=1}^H \gets \textbf{Bonus}(D_{\tE}^{k-1},~\delta_k)$.\\
        \For{$h=H,H-1...,1$}{
            $\uQE{,h}^{\pi_{\tE}^k}(\cdot,\cdot) \gets \max\{0, \rE{,h}^k(\cdot,\cdot) + \hmPE{,h}^{\pi_{\tE}^k}\uVE{,h+1}^k(\cdot,\cdot) - \bonus_{\tE,h}^k(\cdot,\cdot)\},~\quad \uVE{,h}^{\pi_{\tE}^k}(\cdot) = \uQE{,h}^{\pi_{\tE}^k}(\cdot,\pi_{\tE}^k)$\\
            $\tQ_{\tE,h}^{k}(\cdot,\cdot) \gets \min\{H, r_{\tE}(\cdot,\cdot) + \hmPE{,h}^k \tV^k_{\tE, h+1}(\cdot,\cdot) + \bonus_{\tE,h}^k(\cdot,\cdot)\}.$\\
            \For{$s_h\in\cS_h$}{
                $\cI^k_{s_h} \gets \{\task\in[\Task]|\{\uVOt{,h}^k(s_h) \leq \tQ_{\tE,h}^{k}(s_h,\upiO{,\task,h}^k) + \epsilon\} \cap \{\max_a \NO{,\task,h}^k(s_h,a) \geq \frac{\lambda}{3} k\}\}.$ \\
                \If{$\cI^k(s_h) \neq \emptyset$}{
                    \lIf{$\task^{k-1}_{s_h} \neq \Null$ and $\task^{k-1}\in \cI^k_{s_h}$}{
                        $\task^k_{s_h} \gets \task^{k-1}_{s_h}$
                    }
                    \lElseIf{$\task^{k-1}_{s_h} \neq \Null$ and $\exists w\in \cI^k(s_h)$ s.t. $\piE{}^{k-1}(s_h) = \arg\max_a \NO{,w,h}(s_h,a)$}{
                        $\task^k_{s_h} \gets w$
                    }
                    \lElse{
                        $\task^k_{s_h} \gets \text{Unif}(\cI^k_{s_h})$.
                    }
                    $\piE{}^k(s_h) \gets \arg\max_a \NO{,\task^k_{s_h},h}^k(s_h,a)$.
                }
                \lElse{
                    $\task^k_{s_h} \gets \Null$, $\pi_{\tE}^k(s_h) \gets \arg\max_a \tQ_{\tE,h}^{k}(s_h,a)$
                }
                $\tV^k_{\tE,h}(s_h) \gets \min\{H, \tQ_{\tE,h}^{k}(s_h,\pi_{\tE}^k)+\frac{1}{H}(\tQ_{\tE,h}^{k}(s_h,\pi_{\tE}^k) - \uQE{,h}^{\pi_{\tE}^k}(s_h,\pi_{\tE}^k))\}$
            }
        }
        Deploy $\pi_{\tE}$ to interact with $\ME$ and receive $\tauE^k$; update $D_{\tE}^{k} \gets D_{\tE}^{k-1} \cup \{\tauE^k\}$\\
    }
    \caption{Robust Tiered RL with Multiple Low-Tier Tasks}\label{alg:RL_Setting_MT}
\end{algorithm}
We consider the same \textbf{ModelLearning} and \textbf{Bonus} algorithm in Sec.~\ref{appx:RL_add_Alg_Cond_Nota}, but a different condition for $\algO$ listed below:
\begin{condition}[Condition on $\algO$ in MT-TRL]\label{cond:requirement_on_algO_MT}
    $\algO$ is an algorithm which returns deterministic policies at each iteration for each task $M_{tO,w} \in \cM_\tO$, and there exists $C_1,C_2$ only depending on $S,A,H$ and $\Delta_{\min}$ but independent of $k$, such that for arbitrary $k\geq 2$, we have $\Pr(\cEOTk) \geq 1-\frac{1}{k^\alpha}$ for $\cEOTk$ defined below:
    $$
    \cEOTk := \bigcap_{\task\in\Task}\{\sum_{\tk=1}^k \VO{,w,1}^*(s_1)-\VO{,w,1}^{\pi_{\tO}^\tk}(s_1) \leq C_1 + \alpha C_2 \log \Task k\}.
    $$
\end{condition}
Next, we introduce some notations. As analogues of $\cEBk$ and $\cECk$ in single low-tier task setting, we consider the following events:
\begin{align*}
    \cEBTk := \bigcap_{\substack{(\cdot)\in\{\tE,\tO_1,...,\tO_T\},\\ h\in[H], \\ s_h\in\cS_h,a_h\in\cA_h}}&\Big\{\{H\cdot \|\hmP_{(\cdot),h}^k(s_h,a_h) - \mP_{(\cdot),h}(s_h,a_h)\|_1 < \bonus_{(\cdot), h}^k(s_h,a_h) \leq B_1\sqrt{\frac{\log (B_2/\delta_k)}{N_{(\cdot),h}^k(s_h,a_h)}}\}\Big\};\\
    \cECTk :=   \bigcap_{\substack{h\in[H],\\s_h\in\cS_h,\\ a_h\in\cA_h}}\Big\{ & \{\frac{1}{2}\sum_{\tk=1}^k d^{\pi_{\tE}^\tk}(s_h,a_h) - \alpha \log (2SAH\Task k) \leq \NE{,h}^k(s_h,a_h) \\
    & \qquad\qquad \qquad\qquad \qquad\qquad \leq e\sum_{\tk=1}^k d^{\pi_{\tE}^\tk}(s_h,a_h) + \alpha \log (2SAH\Task k)\}\\
    & \cap \Big(\bigcap_{\task\in[\Task]}\{\frac{1}{2}\sum_{\tk=1}^k d^{\pi_{\tO,\task}^\tk}(s_h,a_h) - \alpha \log (2SAH\Task k) \leq \NO{,\task,h}^k(s_h,a_h) \\
    & \qquad\qquad \qquad\qquad \qquad\qquad \leq e\sum_{\tk=1}^k d^{\pi_{\tO,\task}^\tk}(s_h,a_h) + \alpha \log (2SAH\Task k)\}\Big)\Big\}.
\end{align*}
Under the choice of $\delta_k = 1/SAHWk^\alpha$, we have $\Pr(\cEBTk) \geq 1 - \frac{1}{k^\alpha}$. Besides, as a result of Lem. \ref{lem:concentration}, we have $\Pr(\cECTk) \geq 1 - \frac{1}{k^\alpha}$.

\subsection{Analysis}\label{appx:analysis_RL_MT}
In this sub-section, we introduce the analysis for Alg.~\ref{alg:RL_Setting_MT}.
We first provide several lemma and theorem for preparation, which are extended from single task setting. We will omit the detailed proofs if they are almost the same expect there are multiple source tasks and the additional $\Task$ in the log factors.
\begin{lemma}[Lem.~\ref{lem:convergence_speed_of_PVI} in TRL-MST Setting]\label{lem:convergence_speed_of_PVI_MT}
    There exists an absolute constant $c_\Xi^{[W]}$, such that for arbitrary fixed $\threshold > 0$, and for arbitrary 
    $$k \geq c_\Xi^{[W]} \max\{\frac{\alpha B_1^2H^2S}{\lambda^2\threshold^2}\log(\frac{\alpha HSAB_1B_2\Task}{\lambda \threshold}), \frac{(C_1+\alpha C_2)SH}{\Delta_{\min}\lambda\threshold}\log\frac{C_1C_2SAH\Task}{\Delta_{\min}\lambda\threshold}\}$$ 
    on the event $\cECTk,\cEBTk$ and $\cEOTk$, $\forall h \in [H], s_h \in \cS_h$ with $N^k_{\tO,h}(s_h) > \frac{\lambda}{3}$, we have
    \begin{align*}
        V^*_{\tO,h}(s_h) - \uVO{,h}^k(s_h) \leq \threshold.
    \end{align*}
\end{lemma}
\begin{lemma}[Lem.~\ref{lem:lower_bound_of_dOstar} in TRL-MST Setting]\label{lem:lower_bound_of_dOstar_MT}
    There exists a constant $c_{occup}^{[W]}$ which is independent of $\lambda,S,A,H$ and gap $\Delta$, s.t., for all $k \geq k_{occup}^{[W]} := c_{occup}^{[W]}\frac{C_1+\alpha C_2}{\lambda \Delta_{\min}}\log(\frac{\alpha C_1C_2 SAH\Task}{\lambda \Delta_{\min}})$, on the events of $\cEOTk$ and $\cECTk$, forall $w\in[W]$, $\NO{,w,h}^k(s_h,a_h) \geq \frac{\lambda}{3} k$ implies that $d^*_{\tO,w}(s_h,a_h) \geq \frac{\lambda}{9}$, and conversely, if $d^*_{\tO,w}(s_h,a_h) \geq \lambda$, we must have $\NO{,w,h}^k(s_h) \geq \NO{,w,h}^k(s_h,\pi^*_{\tO,w}) \geq \frac{\lambda}{3} k$.
\end{lemma}
\begin{restatable}{theorem}{ThmOverEstMT}[Thm.~\ref{thm:overestimation} in TRL-MST Setting]\label{thm:overestimation_MT}
    There exists a constant $c_{overest}^{[W]}$, such that, for arbitrary $k \geq k_{ost}^{[W]}$ with
    $$
    k_{ost}^{[W]} := c_{overest}^{[W]}\cdot \max\{\alpha\frac{B_1^2H^2S}{\lambda^2\Delta_{\min}^2}\log(HSA\Task B_1B_2), \frac{(C_1+\alpha C_2)SH}{\lambda\Delta_{\min}^2}\log\frac{C_1C_2SAH\Task }{\Delta_{\min}\lambda}\},
    $$
    on the event of $\cECTk, \cEBTk,\cEOTk$, we have:
    \begin{align*}
        \QE{,h}^*(s_h,a_h) \leq& \tQE{,h}^k(s_h,a_h),\quad \VE{,h}^*(s_h) \leq \tVE{,h}^k(s_h),\quad \forall h \in [H],~s_h\in\cS_h,~a_h\in\cA_h.
    \end{align*}
    and
    \begin{align}
        \VE{,1}^*(s_1) - \VE{,1}^{\pi_{\tE}^k}(s_1) \leq 2e\EE_{\pi_{\tE}^k}[\sum_{h=1}^H \Clip{\min\{H, 4\bonus^k_{\tE,h}(s_h,a_h)\}}{\Big|\frac{\Delta_{\min}}{4eH}\vee \frac{\DeltaE(s_h,a_h)}{4e}}].\label{eq:regret_bound_MT}
    \end{align}
\end{restatable}
\noindent Next, we first establish a $O(\log K)$-regret bound regardless of similarity between $\cM_\tO$ and $\ME$, and use it to prove Lem.~\ref{lem:absorb_to_sim_task_RL}, which will be further used to establish the tighter bound in Thm.~\ref{thm:regret_bound_RL_MT}.
\begin{theorem}[Regret bound for general cases]\label{thm:regret_bound_RL_MT_general}
    \begin{align}
        \Regret_{K}(\ME) = O\left(k_{start}^{[W]} \cdot H +  SAH^2\log(2SAHK) + B_1\log(B_2 K)\sum_{h=1}^H\sum_{s_h,a_h} (\frac{H}{\Delta_{\min}}\wedge \frac{1}{\DeltaE(s_h,a_h)}) \right).\label{eq:regret_bound_RL_MT_general}
    \end{align}
\end{theorem}
\begin{proof}
    Consider $k_{start}^{[W]} := \max\{k_{ost}^{[W]}, k_{occup}^{[W]}\}$.
    Similar to the proof of Thm. \ref{thm:regret_upper_bound}, we can conduct identical techniques and provide the following bounds for arbitrary $s_h,a_h$:
    \begin{align*}
        &\sum_{k=k_{start}^{[W]}+1}^K \EE_{\pi_{\tE}^k}[\Clip{\min\{H, 4\bonus^k_{\tE,h}(s_h,a_h)\}}{\Big|\frac{\Delta_{\min}}{4eH}\vee \frac{\DeltaE(s_h,a_h)}{4e}}]\\
        \leq & H\cdot \sum_{k=1}^{\tau_{s_h,a_h}^K} d_{\tE}^{\piE{}^k}(s_h,a_h) + \sum_{k=\tau_{s_h,a_h}^K+1}^K \EE_{\pi_{\tE}^k}[\Clip{\min\{H, 4B_1\sqrt{\frac{\alpha\log(K B_2)}{\NE{,h}^k(s_h,a_h)}}\}}{\Big|\frac{\Delta_{\min}}{4eH}\vee \frac{\DeltaE(s_h,a_h)}{4e}}]\\
        = & O\left(H \log (2SAH\Task K) +  B_1 (\frac{H}{\Delta_{\min}}\wedge \frac{1}{\DeltaE(s_h,a_h)})\log (B_2\Task K)\right).
    \end{align*}
    where $\tau_{s_h,a_h}^K := \min_k~s.t.~\forall k' \geq k,~\frac{1}{4}\sum_{k'=1}^{k-1}d^{\pi_{\tE}^k}_{\tE}(s_h,a_h) \geq \alpha \log (2SAH\Task K)$.
    Then, we can conclude:
    \begin{align*}
        &\EE[\sum_{k=1}^K \VE{,1}^*(s_1) - \VE{,1}^{\piE{}^k}(s_1)]\\
        \leq & \sum_{k=k_{start}^{[W]}+1}^K 2e\EE_{\pi_{\tE}^k}[\sum_{h=1}^H \Clip{\min\{H, 4\bonus^k_{\tE,h}(s_h,a_h)\}}{\Big|\frac{\Delta_{\min}}{4eH}\vee \frac{\DeltaE(s_h,a_h)}{4e}}\mathbb{I}[\cEBTk\cap\cEOTk\cap\cECTk]] \\
        & + \sum_{k=k_{start}^{[W]}+1}^K H\cdot \Pr(\cEBTk^\complement\cup\cEOTk^\complement\cup\cECTk^\complement) + H\cdot k_{start}^{[W]} \\
        = & O\left(k_{ost}^{[W]} \cdot H +  SAH^2\log(2SAH\Task K) + B_1\sum_{h=1}^H\sum_{s_h,a_h} (\frac{H}{\Delta_{\min}}\wedge \frac{1}{\DeltaE(s_h,a_h)}) \log(B_2 \Task K) \right).
    \end{align*}
\end{proof}
In the following, we try to show an extension of Lem.~\ref{lem:absorbing_to_similar_task} in RL setting.
To establish result, we require some techniques to overcome the difficulty raised by state transition. 
We first establish sub-linear regret bound for Alg.~\ref{alg:RL_Setting_MT} based on the multi-task version of Eq.~\eqref{eq:regret_bound}, and use it we can show that $\sum_{k'=k/2}^k d^*_{\tE}(s_h,a_h) > 0$ when $k$ is large enough, which imply that $\pi^*_{\tO,w^{\tilde{k}}}(s_h) = \pi^*_{\tE}(s_h)$ for some $\tilde{k} \in [\frac{k}{2},k]$ with high probability.
Under good events, by our task selection strategy, we can expect if $s_h \in \cZ^{\lambda,[W]}_h$ and $\pi^*_{\tO,w^{\tilde{k}}}(s_h) = \pi^*_{\tE}(s_h)$, from $\tilde{k}$ to $k$, either the trusted task does not change, or it will hand over to another one recommending the same action, which is exactly $\pi^*_{\tE}(s_h)$.
\begin{restatable}{lemma}{LemAbsSimTaskRL}[Absorb to Similar Task]\label{lem:absorb_to_sim_task_RL}
    Under Assump.~\ref{assump:unique_optimal_policy}, \ref{assump:opt_value_dominance}, \ref{assump:lower_bound_Delta_min}
    for all $s_h \in \cZ^{\lambda, [W]}_h$, 
    by running Alg.~\ref{alg:RL_Setting_MT} in Appx.~\ref{appx:RL_add_Alg_Cond_Nota_MT} with $\epsilon = \frac{\tilde{\Delta}_{\min}}{4(H+1)}$, $\alpha > 2$ and arbitrary $\lambda > 0$, there exists $\iota^*_{s_h} = \Poly(SAH,\lambda^{-1},\Delta_{\min}^{-1},1/d^*_{\tE}(s_h),\log W)$, such that, $\forall k \geq \iota^*_{s_h}$, $\Pr(\piE{}^k(s_h) \neq \piE{}^*(s_h)) = O(\frac{1}{k^{\alpha-1}})$.
\end{restatable}
\begin{proof}
    Similar to MAB setting, we define $\cW^*_{s_h} := \{t \in [\Task]| \pi^*_{\tO,\task}(s_h) = \pi^*_{\tE}(s_h), \VO{,w,h}(s_h) \leq \VE{,w,h}(s_h) + \frac{\tilde{\Delta}_{\min}}{4(H+1)}, d^*_{\tO,\task}(s_h) > \lambda\}$ and $\tcW^*_{s_h} := \{t \in [\Task]| \pi^*_{\tO,\task}(s_h) = \pi^*_{\tE}(s_h), d^*_{\tO,\task}(s_h) > \lambda\}$.
    The key observation is that, when $k \geq k_{start}$:
    \begin{align*}
        \Pr(\piE{}^k(s_h) = \piE{}^*(s_h))  \geq &\Pr(\underbrace{\{\forall k' \in [\frac{k}{2}, k],~\cW^*_{s_h} \subset \cI^{k'}_{s_h}\}}_{\cE_{k,1}} \cap \underbrace{\{\exists k' \in [\frac{k}{2},k],~s.t.~d^{\piE{}^{k'}}_{\tE}(s_h,\piE{}^k)  > 0\}}_{\cE_{k,2}}\\
        &\quad\quad \cap \underbrace{\{\forall k' \in [\frac{k}{2}, k],~\forall t \in \cI^{k'}_{s_h},~\upiO{,w}^{k'}(s_h) = \piO{,\task^{k'}}^{*}(s_h)\}}_{\cE_{k,3}}).
    \end{align*}
    That's because if $\cE_{k,2}$ holds at some $k'$, then, we can only have the following two cases: (1) $\task^{k'}_{s_h} = \task^*_{s_h}$: because of $\cE_{k,1}$, we must have $\task^k_{s_h} = \task^*_{s_h}$, which implies $\piE{}^k(s_h) = \piE{}^*(s_h)$ as a result of $\cE_{k,3}$; (2) $\task^{k'}_{s_h} \neq \task^*_{s_h}$: because of $\cE_{k,1}$ and $\cE_{k,3}$, $\task^\tk$ can only transfer inside $\tcW^*_{s_h}$ for $\tk\in[k',k]$, and therefore we still have $\piE{}^k(s_h) = \piE{}^*(s_h)$. In the following, we will provide a upper bound for $\Pr(\cE_{k,1}^\complement), \Pr(\cE_{k,2}^\complement), \Pr(\cE_{k,3}^\complement)$.

    On the event of $\cEBTk$, with a similar analysis as Lem. \ref{lem:underestimation} on $\uQO{,\task,h}^k$, we have:
    \begin{align*}
        \Pr(\cE_{k,1}^\complement) \leq \sum_{k'=\frac{k}{2}}^k \Pr(\task^*_{s_h} \not\in \cI^{k'}_{s_h}) \leq \sum_{k'=\frac{k}{2}}^k \Pr(\cEBTk^\complement) \leq \frac{2}{(k/2)^{\alpha - 1}}.
    \end{align*}
    Besides, for arbitrary $s_h$ with $d^*_{\tE}(s_h) > 0$, and arbitrary $k$, by Azuma-Hoeffding inequality and Thm. \ref{thm:algO_regret_vs_algP_density}, with probability at least $1-\frac{1}{k^\alpha}$, we have:
    \begin{align*}
        \sum_{k' = k/2}^k d^{\pi_{\tE}^k}_{\tE}(s_h,\pi^*_{\tE}) \geq& \sum_{k' = k/2}^k d^{*}_{\tE}(s_h,\pi^*_{\tE}) - H\sqrt{2\alpha k \log k} - \EE[\sum_{k=1}^K \VE{,1}^*(s_1) - \VE{,1}^{\piE{}^k}(s_1)] \\
        = & \frac{k}{2} d^{*}_{\tE}(s_h) - H\sqrt{2\alpha k \log k} - \EE[\sum_{k=1}^K \VE{,1}^*(s_1) - \VE{,1}^{\piE{}^k}(s_1)].
    \end{align*}
    As a result of Thm. \ref{thm:regret_bound_RL_MT_general}, $H\sqrt{2\alpha k \log k} + \EE[\sum_{k=1}^K \VE{,1}^*(s_1) - \VE{,1}^{\piE{}^k}(s_1)]$ will be sub-linear w.r.t. $k$, so there exists $\iota_{s_h} := c_{s_h} \Poly(S,A,H,1/\lambda, 1/\Delta_{\min}, 1/d^*_\tE(s_h))$ for some absolute constant $c_{s_h}$, such that for arbitrary $k \geq \iota_{s_h}$,  $\sum_{k' = k/2}^k d^{\pi_{\tE}^k}_{\tE}(s_h,\pi^*_{\tE}) > 0$, which implies:
    \begin{align*}
        \forall k \geq \iota_{s_h},~\Pr(\cE_{k,2}^\complement) = \Pr(\sum_{k' = k/2}^k d^{\pi_{\tE}^k}_{\tE}(s_h,\pi^*_{\tE}) = 0) \leq \frac{1}{k^\alpha}.
    \end{align*}
    Moreover, as a result of Lem.~\ref{lem:lower_bound_of_dOstar_MT}, when $k \geq k_{occup}^{[W]} := c_{occup}^{[W]}\frac{C_1+\alpha C_2}{\lambda \Delta_{\min}}\log(\frac{\alpha C_1 C_2SAH\Task }{\lambda \Delta_{\min}})$, on the event of $\cEOTk$ and $\cECTk$, if $t \in \cI^k_{s_h}$, we must have $\max_a \NO{,\task,h}(s_h,a) > \lambda k$, which implies $\pi^*_{\tO,\task}(s_h) = \arg\max \NO{,\task,h}^k(s_h,a)$. Therefore, $\forall k \geq 2 \cdot k_{occup}^{[W]}$:
    \begin{align*}
        \Pr(\cE_{k,3}^\complement) \leq& \sum_{k'=k/2}^k \sum_{t\in [\Task]}\Pr(\{t \in \cI^{k'}_{s_h}\} \cap \{\pi^*_{\tO,\task}(s_h) \neq \arg\max_a \NO{,\task,h}^{k'}(s_h,a)\})\\
        \leq & \sum_{k'=k/2}^k  \Pr(\cEOTk^\complement) + \Pr(\cECTk^\complement)\leq \frac{4}{(k/2)^{\alpha-1}}.
    \end{align*}
    Therefore, for arbitrary $k \geq \iota^*_{s_h} := \max\{k_{start}^{[W]}, \iota_{s_h}, k_{occup}^{[W]}\}=\Poly(S,A,H,1/\lambda, 1/\Delta_{\min}, 1/d^*_\tE(s_h),\log \Task)$, we have:
    \begin{align*}
        \Pr(\piE{}^k(s_h) = \piE{}^*(s_h)) \geq 1 - \frac{2}{(k/2)^{\alpha-1}} -\frac{2}{(k/2)^{\alpha-1}} - \frac{4}{(k/2)^{\alpha-1}} \geq 1 - \frac{8}{(k/2)^{\alpha - 1}}.
    \end{align*}
    which finishes the proof.
\end{proof}

\begin{theorem}[Detailed Version of Thm.~\ref{thm:regret_bound_RL_MT}]\label{thm:regret_bound_RL_MT_formal}
    Under Assump. \ref{assump:unique_optimal_policy}, \ref{assump:opt_value_dominance} and \ref{assump:lower_bound_Delta_min}, 
    Cond.~\ref{cond:bonus_term} and~\ref{cond:requirement_on_algO_MT},
    by running Alg.~\ref{alg:RL_Setting_MT} in Appx.~\ref{appx:RL_add_Alg_Cond_Nota_MT}, with $\epsilon=\frac{\tilde{\Delta}_{\min}}{4(H+1)}$, $\alpha > 2$ and arbitrary $\lambda > 0$, we have\footnote{We only keep non-constant terms and defer the detailed result to Eq.~\eqref{eq:regret_bound_RL_detail_MT}.}:
    \begin{align*}
        &\Regret_{K}(\ME) \\
        =& O\Big(H\cdot \max\{\alpha\frac{B_1^2H^2S}{\lambda^2\Delta_{\min}^2}\log(HSA\Task B_1B_2), \frac{(C_1+\alpha C_2)SH}{\lambda\Delta_{\min}^2}\log\frac{C_1C_2SAH\Task }{\Delta_{\min}\lambda}\}\\
        &+ \sum_{h=1}^H \sum_{(s_h,a_h) \in \cC^{\lambda,[W],1}_h} H \log (2SAH\Task (K\wedge \frac{1}{\Delta_{\min}\lambda d^*_{\tE}(s_h)}))  +  \frac{B_1}{\DeltaE(s_h,a_h)}\log (SAHB_2\Task (K\wedge \frac{1}{\Delta_{\min}\lambda d^*_{\tE}(s_h)}))\\
        & + \sum_{h=1}^H \sum_{(s_h,a_h) \in \cC^{\lambda,[W],2}_h} H \log (2SAH\Task (K\wedge \frac{1}{\Delta_{\min}\lambda d^*_{\tE,\min}})) \\
        & \qquad\qquad\qquad\qquad + B_1(\frac{H}{\Delta_{\min}}\wedge \frac{1}{\DeltaE(s_h,a_h)})\log (SAHB_2\Task (K\wedge \frac{1}{\Delta_{\min}\lambda d^*_{\tE,\min}}))\\
        & + \sum_{h=1}^H \sum_{(s_h,a_h) \in \cC^{*}_h} \frac{HB_1}{\Delta_{\min}} \log (SAH B_2\Task \min\{K, \frac{1}{\lambda \Delta_{\min} d^*_{\tE}(s_h)}\})\\
        & + \sum_{h=1}^H \sum_{(s_h,a_h) \in \cS_h\times\cA_h \setminus (\cC^{\lambda,[W],1}_h\cup \cC^{\lambda,[W],2}_h)} H \log (2SAH\Task K) +  B_1 (\frac{H}{\Delta_{\min}}\wedge \frac{1}{\DeltaE(s_h,a_h)})\log (B_2\Task K)\Big)\numberthis\label{eq:regret_bound_RL_detail_MT}\\
        =& O(SH\sum_{h=1}^H \sum_{(s_h,a_h) \in \cS_h\times\cA_h \setminus \cC^{\lambda,[W]}_h} \frac{H}{\Delta_{\min}}\wedge \frac{1}{\DeltaE(s_h,a_h)}\log (SAH\Task K)).
    \end{align*}
\end{theorem}
\begin{proof}
    Consider the same $k_{start}^{[W]} = \max\{k_{ost}^{[W]}, k_{occup}^{[W]}\}$. Similar to single task setting, we study the following two types of states.
    \paragraph{Type 1: $(s_h,a_h) \in \cC^{\lambda, [W], 1}_h \cup \cC^{\lambda, [W], 2}_h$: Constant Regret because of Low Visitation Probability}
    For each $(s_h,a_h) \in \cC^{\lambda, [W], 1}_h$, by leveraging Lem.~\ref{lem:absorb_to_sim_task_RL}, we have:
    \begin{align*}
        &\sum_{k=k_{start}^{[W]}+1}^K \EE_{\pi_{\tE}^k}[\Clip{\min\{H, 4\bonus^k_{\tE,h}(s_h,a_h)\}}{\Big|\frac{\Delta_{\min}}{4eH}\vee \frac{\DeltaE(s_h,a_h)}{4e}}\mathbb{I}[\cEBTk\cap\cEOTk\cap\cECTk]]\\
        =&\sum_{k=k_{start}^{[W]}+1}^{\iota^*_{s_h}} \EE_{\pi_{\tE}^k}[\Clip{\min\{H, 4\bonus^k_{\tE,h}(s_h,a_h)\}}{\Big|\frac{\Delta_{\min}}{4eH}\vee \frac{\DeltaE(s_h,a_h)}{4e}}\mathbb{I}[\cEBTk\cap\cEOTk\cap\cECTk]]\\
        =& O\Big(H \log (2SAH\Task (K\wedge \frac{1}{\Delta_{\min}\lambda d^*_{\tE}(s_h)})) +  \frac{B_1}{\DeltaE(s_h,a_h)}\log (SAHB_2\Task (K\wedge \frac{1}{\Delta_{\min}\lambda d^*_{\tE}(s_h)}))\Big).
    \end{align*}
    Note that different from single task setting, the convergence speed of $\pi^k_\tE(s_h)$ to $\pi^*_\tE(s_h)$ for $s_h \in \cZ^{\lambda,[W]}$ depends on $d^*_\tE(s_h)$. Therefore, for $(s_h,a_h) \in \cC^{\lambda, [W], 2}_h$, we can guarantee $d^{\pi^k_{\tE}(s_h,a_h)}$ decays to zero only after the convergence of all its ancester states in $\cZ^{\lambda,[W]}$, and:
    \begin{align*}
        &\sum_{k=k_{start}^{[W]}+1}^K \EE_{\pi_{\tE}^k}[\Clip{\min\{H, 4\bonus^k_{\tE,h}(s_h,a_h)\}}{\Big|\frac{\Delta_{\min}}{4eH}\vee \frac{\DeltaE(s_h,a_h)}{4e}}\mathbb{I}[\cEBTk\cap\cEOTk\cap\cECTk]]\\
        \leq &\sum_{k=k_{start}^{[W]}+1}^{k_{s_h,a_h}} \EE_{\pi_{\tE}^k}[\Clip{\min\{H, 4\bonus^k_{\tE,h}(s_h,a_h)\}}{\Big|\frac{\Delta_{\min}}{4eH}\vee \frac{\DeltaE(s_h,a_h)}{4e}}\mathbb{I}[\cEBTk\cap\cEOTk\cap\cECTk]]\\
        =& O\Big(H \log (2SAH\Task (K\wedge \frac{1}{\Delta_{\min}\lambda d^*_{\tE,\min}})) +  B_1(\frac{H}{\Delta_{\min}}\wedge \frac{1}{\DeltaE(s_h,a_h)})\log (SAHB_2\Task (K\wedge \frac{1}{\Delta_{\min}\lambda d^*_{\tE,\min}}))\Big).
    \end{align*}
    where $k_{s_h,a_h} := \max_{h'\in[h-1],s_{h'}\in \cC^{\lambda,[W],1}_{h'}}\iota^*_{s_{h'}}$ and $d^*_{\tE,\min} := \min_{s_h:d^*_\tE(s_h) > 0} d^*_\tE(s_h)$ is the minimal probability to reach state by optimal policy in $\tE$.

    \paragraph{Type 2: $(s_h,a_h) \in \cC^*_h$, i.e. $d^*_{\tE}(s_h, a_h)=d^*_{\tE}(s_h) > 0$}
    Similar to the discussion of Type 2 in the proof of Thm. \ref{thm:regret_upper_bound}, when $k \geq \tau^*_{s_h} $ for some $\tau^*_{s_h} = \Poly(S,A,H,\lambda^{-1}, \Delta_{\min}^{-1},(d^*_{\tE}(s_h))^{-1},\log W)$, we have $\sum_{k'=1}^{k-1} d^{\pi^k_\tE}_\tE(s_h,a_h) = O(k d^*_{\tE}(s_h))$, and therefore,
    \begin{align*}
        & \sum_{k=k_{start}^{[W]}+1}^K \EE_{\pi_{\tE}^k}[\Clip{\min\{H, 4\bonus^k_{\tE,h}(s_h,a_h)\}}{\Big|\frac{\Delta_{\min}}{4eH}\vee \frac{\DeltaE(s_h,a_h)}{4e}}|\cEBTk,\cEOTk,\cECTk]\\
        =&O\left(\frac{HB_1}{\Delta_{\min}} \log (SAH B_2\Task \min\{K, \frac{1}{\lambda \Delta_{\min} d^*_{\tE}(s_h)}\})\right).
    \end{align*}
    Therefore, we can conclude that:
    \begin{align*}
        &\Regret_{K}(\ME) \\
        =& O\Big(H\cdot \max\{\alpha\frac{B_1^2H^2S}{\lambda^2\Delta_{\min}^2}\log(HSA\Task B_1B_2), \frac{(C_1+\alpha C_2)SH}{\lambda\Delta_{\min}^2}\log\frac{C_1C_2SAH\Task }{\Delta_{\min}\lambda}\}\\
        &+ \sum_{h=1}^H \sum_{(s_h,a_h) \in \cC^{\lambda,[W],1}_h} H \log (2SAH\Task (K\wedge \frac{1}{\Delta_{\min}\lambda d^*_{\tE}(s_h)}))  +  \frac{B_1}{\DeltaE(s_h,a_h)}\log (SAHB_2\Task (K\wedge \frac{1}{\Delta_{\min}\lambda d^*_{\tE}(s_h)}))\\
        & + \sum_{h=1}^H \sum_{(s_h,a_h) \in \cC^{\lambda,[W],2}_h} H \log (2SAH\Task (K\wedge \frac{1}{\Delta_{\min}\lambda d^*_{\tE,\min}})) \\
        & \qquad\qquad\qquad\qquad + B_1(\frac{H}{\Delta_{\min}}\wedge \frac{1}{\DeltaE(s_h,a_h)})\log (SAHB_2\Task (K\wedge \frac{1}{\Delta_{\min}\lambda d^*_{\tE,\min}}))\\
        & + \sum_{h=1}^H \sum_{(s_h,a_h) \in \cC^{*}_h} \frac{HB_1}{\Delta_{\min}} \log (SAH B_2\Task \min\{K, \frac{1}{\lambda \Delta_{\min} d^*_{\tE}(s_h)}\})\\
        & + \sum_{h=1}^H \sum_{(s_h,a_h) \in \cS_h\times\cA_h \setminus (\cC^{\lambda,[W],1}_h\cup \cC^{\lambda,[W],2}_h)} H \log (2SAH\Task K) +  B_1 (\frac{H}{\Delta_{\min}}\wedge \frac{1}{\DeltaE(s_h,a_h)})\log (B_2\Task K)\Big).
    \end{align*}
    We consider the similar Hoeffding bound for the choice of $B_1$ and $B_2$ as Example~\ref{example:choice_of_B}, and by plugging $B_1=\Theta(SH)$ and $B_2=\Theta(SA)$ into the above equation and omitting all the terms independent with $K$, we have:
    \begin{align*}
        \Regret_{K}(\ME)=& O(SH\sum_{h=1}^H \sum_{(s_h,a_h) \in \cS_h\times\cA_h \setminus \cC^{\lambda,[W]}_h} \frac{H}{\Delta_{\min}}\wedge \frac{1}{\DeltaE(s_h,a_h)}\log (SAH\Task K)).
    \end{align*}
\end{proof}

\section{Missing Details for Experiments}\label{appx:experiments}
\paragraph{Construction of Source and Target Tasks} We first randomly construct the transition function of the high-tier task $M_{\text{Hi}}$ (i.e. $\mathbb{P}_{\text{Hi}}$ are randomly sampled and normalized to make sure their validity). Then, similarly, we randomly construct the reward function of $M_{\text{Hi}}$ and shift the reward function to ensure $M_{\text{Hi}}$ has unique optimal policy and $\Delta_{\min, \text{Hi}} = 0.1$.

Next, we construct the source tasks by randomly permute the transition matrix. In another word, for any $s_h$, we randomly permute $a_1,a_2,a_3$ to $a_1',a_2',a_3'$ and assign $\mP_{h,\text{Lo}}(\cdot|s_h,a_i') \gets \mP_{h,\text{Hi}}(\cdot|s_h,a_i)$ for $i\in[3]$. In this way, the Optimal Value Dominance (OVD) condition is ensured, and we can expect some of $s_h$ are transferable when $\pi^*_{\text{Lo}}(s_h) = \pi^*_{\text{Hi}}(s_h)$.
When the number of source tasks $W > 1$, we repeat the above process and construct $W$ different source tasks.
\endgroup
\end{document}